\documentclass{article}

\usepackage[main, final]{neurips_2026}

\usepackage{amsmath}  
\usepackage{amssymb}  
\usepackage{amsthm}
\usepackage{mathtools}
\usepackage{comment}
\usepackage{graphicx}
\usepackage[utf8]{inputenc} 
\usepackage[T1]{fontenc}    
\usepackage{url}            
\usepackage{booktabs}       
\usepackage{amsfonts}       
\usepackage{nicefrac}       
\usepackage{microtype}      
\usepackage[dvipsnames,svgnames,x11names]{xcolor}         
\usepackage{subcaption}
\usepackage{algorithm}
\usepackage{algpseudocode}
\allowdisplaybreaks
\usepackage{enumitem}

\newtheorem{assumption}{Assumption}
\newtheorem{theorem}{Theorem}
\newtheorem{proposition}[theorem]{Proposition}

\newtheorem{lemma}[theorem]{Lemma}

\newcommand{\erf}{\text{erf}}

\newcommand{\half}{\nicefrac{1}{2}}

\newcommand{\X}{\mathbb{X}}

\newcommand{\Xhalf}{\X^{(\half)}}

\usepackage{xcolor}

\usepackage{xcolor}
\usepackage[
    colorlinks=true,
    linkcolor=blue!55!black,
    citecolor=green!45!black,
    urlcolor=purple!60!black,
    filecolor=magenta!60!black,
    bookmarks=true,
    bookmarksopen=true,
    bookmarksnumbered=true,
    unicode=true
]{hyperref}

\usepackage[nameinlink,noabbrev]{cleveref}

\crefname{equation}{Eq.}{Eqs.}
\Crefname{equation}{Equation}{Equations}

\crefname{figure}{Fig.}{Figs.}
\Crefname{figure}{Figure}{Figures}

\crefname{subfigure}{Fig.}{Figs.}
\Crefname{subfigure}{Figure}{Figures}

\crefname{table}{Table}{Tables}
\Crefname{table}{Table}{Tables}

\crefname{section}{Sec.}{Secs.}
\Crefname{section}{Section}{Sections}

\crefname{appendix}{App.}{Apps.}
\Crefname{appendix}{Appendix}{Appendices}

\crefname{algorithm}{Alg.}{Algs.}
\Crefname{algorithm}{Algorithm}{Algorithms}

\crefname{assumption}{Assumption}{Assumptions}
\Crefname{assumption}{Assumption}{Assumptions}

\crefname{theorem}{Theorem}{Theorems}
\Crefname{theorem}{Theorem}{Theorems}

\crefname{proposition}{Proposition}{Propositions}
\Crefname{proposition}{Proposition}{Propositions}

\crefname{corollary}{Corollary}{Corollaries}
\Crefname{corollary}{Corollary}{Corollaries}

\crefname{lemma}{Lemma}{Lemmas}
\Crefname{lemma}{Lemma}{Lemmas}

\crefname{definition}{Definition}{Definitions}
\Crefname{definition}{Definition}{Definitions}

\crefname{remark}{Remark}{Remarks}
\Crefname{remark}{Remark}{Remarks}

\title{Understanding axial attention in TSFMs through single location regression}

\author{%
  Andrei Tiberiu Pantea\thanks{Corresponding author: \texttt{andrei-tiberiu.pantea@polytechnique.edu}}\\
  CMAP, École polytechnique, IP Paris\\
  Electricité de France R\&D\\
  Palaiseau, France\\
  \And
  Aymeric Dieuleveut\\
  CMAP, École polytechnique, IP Paris\\
  Palaiseau, France\\
}

\begin{document}

\maketitle

\begin{abstract}
  Recent advances in Time Series Foundation Models (TSFMs) increasingly rely on axial attention mechanisms, but despite their empirical success, the theoretical properties of axial attention remain largely unexplored.
  To bridge this gap, we extend the Single Location Regression framework introduced in \citet{marion} to 2D structured data, in order to study sparse information settings where a target depends on a small number of tokens.
  We formulate an analytically tractable axial attention-type predictor and rigorously analyze its statistical properties and training dynamics. We prove that under a specific asymptotic regime and with knowledge of the underlying oracle parameters, our predictor achieves Bayes optimality, whereas a class of semi-linear simplifications strictly fails to do so, underlining the role of the inner non-linearity. 
  We prove that for a specific temperature scaling, Projected Gradient Descent converges globally to the optimal oracle parameters. Extensive numerical experiments validate our theoretical findings and highlight the critical role of inverse temperature scheduling in empirical convergence.
\end{abstract}

\section{Introduction}

Recent years have witnessed a surge in the development of Time Series Foundation Models (TSFMs) leveraging transformer architectures \citep{vaswani2017attention} (e.g., Chronos-2 \citep{chronos-2}, Moirai 2.0 \citep{moirai}, TimesFM \citep{timesfm}, TabPFN-TS \citep{tabpfnts}, Toto \citep{toto}, PatchTST-FM \citep{patchtst-fm}). 
However, applying standard global attention to high-dimensional multivariate time series presents a significant computational bottleneck. 
Standard attention treats inputs as a flattened 1D sequence, which scales quadratically with the total number of tokens and fails to explicitly leverage the natural 2D grid structure (time steps $\times$ channels) of multivariate data. 
To address this, axial attention introduced in \citet{ho2019axial}, and specifically for multivariate time series in \citet{crossformer}, has emerged as a powerful alternative already implemented in some of the state-of-the-art TSFMs, e.g., Chronos-2, TabPFN-TS and Toto, but also in task specific models such as CEReBrO \citep{cerebro}, an EEG foundation model that utilizes alternating intra-channel temporal and inter-channel spatial attention and STAIT (Spatio-Temporal Alternating Iterative Transformer) \citep{stait} for remote sensing. 
 
The axial-attention layer combines two token-mixing half-layers, as described in Figure~\ref{fig:axial_attention}. Formally, we consider our input time series  as $M \times N$ matrices of tokens $\X = (X_{ij})_{1 \leq i \leq M, 1 \leq j \leq N}$,  taking values in the embedding (or latent) space $\mathbb{R}^d$, where $M$ can be seen as the number of channels and $N$ as the number of time series patches.  Generically, given as hyperparameters two real numbers $\lambda_r,\lambda_c$, referred to as \textit{inverse temperatures}, and an activation function $\sigma: \mathbb R^d\to \mathbb R^d$, the standard axial self-attention layer is parametrized by weight matrices $Q_r, K_r \in \mathbb{R}^{d\times p_r}$, $V_r \in \mathbb{R}^{d \times d}$, $Q_c, K_c, V_c \in \mathbb{R}^{d\times p_c}$, and $O_c \in \mathbb{R}^{p_c \times o_c}$.  The subscripts $r$ and $c$ represent row-wise and, respectively, column-wise operations. The mapping is defined as, for any $j\in [N]$:
\begin{equation} \label{eq:transformer}
    T_{\lambda_r, \lambda_c, \sigma}^{(Q_{r,c}, K_{r,c}, V_{r,c}, O_{c})}(\X)_{\cdot j} = \sigma \Big(\lambda_c \underbrace{\Xhalf_{\cdot j} Q_c}_{M \times p_c} \underbrace{K_c^T \left(\Xhalf_{\cdot j}\right)^T}_{p_c \times M}\Big) \underbrace{\Xhalf_{\cdot j} V_c}_{M \times p_c} \underbrace{O_c}_{p_c \times o_c},
\end{equation}
where the initial row-wise computation is, for any $i\in [M]$:
\begin{equation}\label{eq:halfstep}
    \Xhalf_{i \cdot} = \sigma (\lambda_r \underbrace{\X_{i \cdot} Q_r}_{N \times p_r} \underbrace{K_r^T \X^T_{i \cdot}}_{p_r \times N}) \underbrace{\X_{i \cdot} V_r}_{N \times d},
\end{equation}
where we use the notations $[n] = \{1, 2, \dots, n\}$ for any positive integer $n$, and $W_{i \cdot}$ (resp.~$W_{\cdot j}$) to denote the $i$-th line (resp.~$j$-th column) of a matrix $W$. Note that in \eqref{eq:halfstep} (and throughout the paper when applied to a matrix) the activation function is used row-wise. When $\sigma$ is the softmax function,  this structure is one of the building blocks used in Chronos-2, TabPFN-TS and Toto, stripped of residual connections and normalization layers. Note that the order in which the row attention and column attention are performed has no impact on the results in our paper; we choose row attention followed by column attention just for presentation purposes.

\begin{figure}
    \centering
    \includegraphics[width=1\linewidth]{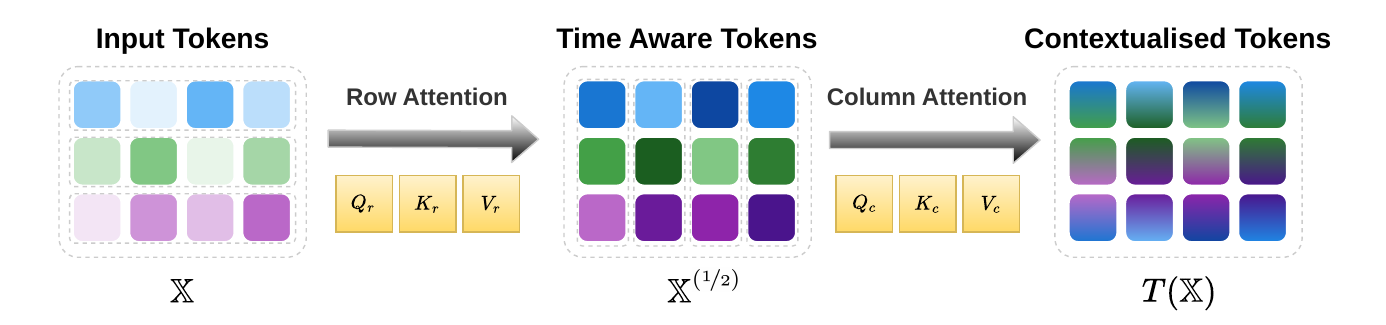}
    \caption{Schematic view of an axial attention layer. Raw input tokens are first processed by a time-wise self-attention layer that isolates features along the rows. A subsequent channel-wise attention layer vertically blends these extracted features, effectively capturing global cross-channel  and temporal context.}
    \label{fig:axial_attention}
\end{figure}

Although the empirical success of these axial architectures is evident across various domains, their theoretical properties remain largely unexplored. 
For example, in \citet{berthelier2026}, the authors have shown that pretrained Chronos-2 is not always capable of identifying simple relationships between targets and covariates in the token-wise sparse information setting. 
Our goal is to improve the theoretical understanding of axial attention by analyzing their ability to approximate and learn the optimal predictor in this setting. 
To this end, we build on the work of \citet{marion}, conducted in the context of classical transformers and motivated by natural language processing (NLP) tasks. 
In that paper, the authors introduced a new class of statistical models, namely Single Location Regression (SLR) models, through which they aimed to investigate the ability of standard attention layers to recover relevant information in a token-wise sparse information setting, where the target depends only on a relatively small number of tokens. 
They have provided theoretical guarantees for a modification of the standard transformer architecture, proving that one can recover optimal predictors for SLRs.
This naturally raises a fundamental, yet unanswered, question: \textit{Can we achieve the same statistical and optimization performance with axial attention as we do with normal global attention?} 
The answer is not obvious, since the information flow is much more restricted in an axial transformer, where each token only attends to tokens with which it shares an axis, as opposed to the free flow of information between any two tokens in a standard transformer.

Building upon the foundation established by \citet{marion}, we introduce a SLR model and extend the theoretical analysis of transformer architectures to axial attention mechanisms, as a method of assessing the theoretical capabilities of current state-of-the-art TSFMs. 
Our main contributions are structured as follows:
\begin{itemize}[leftmargin=*]
    \item In Section~\ref{sec:sec02}, we introduce an SLR statistical model for the token-wise sparse information setting adapted to 2D structured data (e.g., multivariate time series, tables). We then explain the connection between such a model and multivariate time series anomaly analysis, and finally, we introduce a trainable predictor and show that it corresponds to a slightly simplified axial attention layer.
    \item In Section~\ref{sec:sec04}, we analyze the performance of the predictor introduced and prove that there exists an asymptotic regime in which it is statistically optimal for the SLR model investigated, when using the oracle parameters. Furthermore, we show that removing the inner non-linearity in our predictor renders it suboptimal.
    \item In Section~\ref{sec:sec05}, we propose an optimization algorithm for our predictor and provide theoretical guarantees for its convergence, despite the non-convexity of the problem. We further conduct extensive numerical studies which reveal a complex optimization landscape that confirm our theoretical findings.
\end{itemize}

Our structure closely follows that of \citet{marion}, but our motivation (the study of TSFMs) and the technical challenges are, nevertheless, very different. 
Standard global attention benefits from symmetric token coupling, which facilitates the derivation of closed-form analytical solutions for the expected risk. 
By contrast, axial attention operates in a sequential manner. 
This decoupled architecture breaks the global symmetry and induces a cascading non-linear dependence between axes. 
As a result, the optimization landscape exhibits significant complexity, rendering standard closed-form approaches to statistical optimality and convergence of optimization methods intractable and necessitating novel analytical techniques.

\section{SLR model and simplified axial attention predictor} \label{sec:sec02}

In this section, we introduce an SLR model and underline its connection to multivariate time series anomaly analysis. Afterwards, we consider a simplification of axial attention that renders the statistical analysis in Section~\ref{sec:sec04} tractable.

\paragraph{Single location regression model.}
We focus on the following statistical model, where the output $Y \in \mathbb{R}$ is assumed to be given by
\begin{equation} \label{eq:task} \tag{SLR}
    Y = X_{I_0 J_0} ^T v^* + \xi,
\end{equation}
where $\xi \sim \mathcal{N}(0, \epsilon^2)$ is independent noise and  $I_0$, $J_0$ are discrete random variables over $[M]$ and $[N]$, respectively. No restriction on the distribution of $(I_0, J_0)$ is imposed unless otherwise specified. Conditionally on $(I_0, J_0)$, for $(\ell, \gamma^2, \epsilon^2, k^*, v^*) \in \mathbb{R}_+ \times \mathbb{R}^*_+ \times \mathbb{R}^*_+ \times \mathbb{S}^{d-1} \times \mathbb{S}^{d-1}$,
\begin{equation*}
    \begin{cases}
        X_{I_0 J_0} & \sim \quad \mathcal{N}\left(\ell k^*, \gamma^2 I_d\right), \\
        X_{ij} & \sim \quad \mathcal{N}\left(0, I_d\right), \text{ for } (i,j) \neq (I_0, J_0),
    \end{cases}
\end{equation*}
where $X_{ij}$ are independent for all $(i, j) \in [M] \times [N]$.

The parameters of \eqref{eq:task} are the shift magnitude $\ell$, the localization and noise variances $\gamma^2$ and $\epsilon^2$, the unobserved projection vector $v^*$ and the localization vector $k^*$, with both vectors taking values on the $d$-dimensional unit sphere $\mathbb{S}^{d-1}$. The target variable $Y$ depends exclusively on a single embedded token within the $M \times N$ input matrix, positioned at random discrete coordinates $(I_0, J_0)$. This active token is statistically identifiable because its expected value is shifted along the direction $k^*$. Consequently, the primary mathematical challenge of this task stems from its dynamic sparsity.

{\textbf{Remark 1.} Unlike in \citet{marion}, we consider $\ell$ to be decoupled from $d$ throughout our analysis, but note that for $\ell = \sqrt{\nicefrac{d}{2}}$ and $\gamma^2 = \nicefrac{1}{2}$, all tokens have the same expected norm, i.e. $\mathbb{E}[\|X_{I_0J_0}\|_2^2] = \mathbb{E}[\|X_{ij}\|_2^2] = d$, making them indistinguishable just based on their magnitude.}

\begin{figure}[htbp]
    \centering
    \includegraphics[width=1\linewidth]{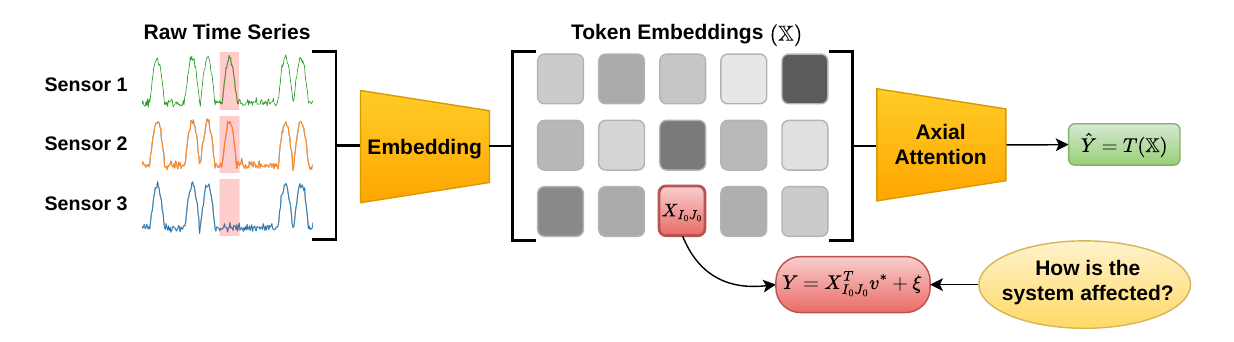}
    \caption{Under normal operating conditions, the background embeddings behave as uninformative noise, or as a standard token, modeled by the standard normal distribution $\mathcal{N}(0, I_d)$. When an anomaly occurs in one sensor, its specific signature is encoded into the token $X_{I_0 J_0}$, characterized by a shifted mean $\ell k^*$ and a localization variance $\gamma^2 I_d$. The regression target $Y$ represents a downstream consequence or system-level anomaly score that depends strictly on identifying this specific embedded state via the projection vector $v^*$. {Our goal is to investigate whether we can recover this target through an axial attention layer.}}
    \label{fig:SLR_Anomaly}
\end{figure}

\paragraph{Connection to multivariate time series anomaly analysis.} 
The Single Location Regression (SLR) model detailed above provides a natural theoretical framework for multivariate time series anomaly analysis, particularly when operating on learned representation embeddings $\X$. We consider system-wide anomalous events identifiable through correlated behaviors across multiple sensors, as pictured in Figure~\ref{fig:SLR_Anomaly}. We operate under the assumption that the embedding layer has already captured these correlated dynamics, localizing the signature of the system-wide anomaly into the embedding of a single primary driving sensor $I_0$ at a specific time step $J_0$. By investigating whether axial attention, which restricts information flow, can successfully recover the true parameters $k^*$ and $v^*$, we can theoretically assess its ability to isolate and leverage localized anomalies.

Furthermore, using axial attention layers instead of standard global attention is attractive also from a computational point of view. The global attention mechanism treats the input as a flattened 1D sequence of length $M \times N$ and this coupling requires $\mathcal{O}(M^2N^2)$ operations per layer, while axial attention only requires $\mathcal{O}(MN^2 + M^2N)$ (see \citet{ho2019axial} for a derivation).

\paragraph{An analytically tractable simplification of axial attention for SLR.} 

In what follows, we introduce a class of predictors that are direct simplifications of the axial attention layer. For an activation $\sigma \colon \mathbb{R}^d \to \mathbb{R}^d$ and inverse temperatures $\lambda_r, \lambda_c \in (0, \infty)$, let
\begin{equation} \label{eq:T_sigma}
    T_{\lambda_r, \lambda_c, \sigma}^{(k, v)}(\X) \coloneqq \sigma(\lambda_c \underbrace{\Xhalf k}_{M \times 1}) ^T \; \underbrace{\Xhalf v}_{M \times 1}, \;  \text{ where } \Xhalf_i \coloneqq \sigma(\lambda_r \underbrace{\X_{i \cdot}k}_{N \times 1})^T \; \underbrace{\X_{i \cdot}}_{N \times d},
\end{equation}
for $i \in [M]$ and learnable parameters $k, v \in \mathbb{S}^{d-1}$. Note that we slightly overload the notation and use $\Xhalf$ to denote the half step of both our predictor and of axial attention.

In the context of TSFMs, sequences are typically appended with designated \texttt{[REG]} tokens to act as readout tokens or attention sinks (see \citet{chronos-2}, for example). Assuming that the input $\X$ is appended with a $(N+1)$th column of row-aggregation \texttt{[REG\_ROW]} tokens and a column aggregation \texttt{[REG\_COL]} token at position $(M+1, N+1)$, we have the following lemma
\begin{lemma}\label{lem:equivalence}
    For $p_r = p_c = 1$, $o_c = 1$, $O_c = 1$, together with shared keys $K_r = K_c$ and $V_r = I_d$, we have the following equivalence:
    \begin{equation*} 
        T^{(Q_{r,c}, K_{r,c}, V_{r,c}, O_{c})}_{\lambda_r, \lambda_c, \sigma}(\X)_{M + 1, N + 1} = T^{( K_c, V_c)}_{\lambda_r a, \lambda_c a^{(\half)}, \sigma}(\X),
    \end{equation*}
    for some constants $a$ and $a^{(\half)}$. The proof and further comments on the constraints are in Appendix~\ref{sec:app_equivalence}.
\end{lemma}

The main bottleneck in trying to analyze $\sigma = \text{softmax}$ in \eqref{eq:T_sigma} is the softmax operation that induces a global coupling between the tokens, which acts as a major roadblock to rigorous statistical analysis. To bypass this issue without losing the core behavioral properties of the network, we substitute softmax with the component-wise nonlinear function $\erf(x) = \frac{2}{\sqrt{\pi}}\int_{0}^x e^{-u^2}\text{d}u$, following \citet{marion}. This adjustment leads to the definition of the analytically tractable predictor that we investigate throughout the remainder of the paper,
\begin{equation} \label{eq:T}
    T^{(k,v)}_{\lambda_r, \lambda_c}(\X) \coloneqq T^{(k,v)}_{\lambda_r, \lambda_c, \erf}(\X) = \erf(\lambda_c \Xhalf k) ^T \; \Xhalf v, \;  \text{ where } \Xhalf_i \coloneqq \erf(\lambda_r \X_{i \cdot}k)^T \; \X_{i \cdot}\ .
\end{equation}

\section{Statistical analysis} \label{sec:sec04}

Having established the  predictor $T^{(k,v)}_{\lambda_r, \lambda_c}$, we now turn to evaluating its statistical properties. 

\subsection{Asymptotic optimality of the oracle predictor} \label{sec:sec04a}

In this section, we follow the approach of \citet{marion}: we first observe that our predictor only depends on the alignments between $k, v, k^*$ and $v^*$, projecting the high-dimensional risk landscape down to five scalars, which is essential to the analysis of the training dynamics in Section~\ref{sec:sec05}. Second, under the assumption of orthogonality between $k^*$ and $v^*$, we obtain a simplified formula for the risk. Finally, using this formula, we find an asymptotic regime in which our predictor is optimal. Although we follow the same structure and maintain similar notation as \citet{marion}, in order to optimize readability, the sequential nature of axial attention makes the results and the proofs more complicated as we lack a closed form formula for the risk. Detailed proofs for all theoretical claims in this section can be found in Appendix~\ref{sec:app_sec04}.

We quantify the model's performance via its risk:
\begin{equation} \label{eq:Risk}
    \mathcal{R}_{\lambda_r, \lambda_c}(k, v) \coloneqq \mathbb{E}\left[\left(Y - T^{(k,v)}_{\lambda_r, \lambda_c}(\X)\right)^2\right].
\end{equation}

Furthermore, we make the following assumption throughout the rest of the paper.
\begin{assumption} \label{ass:ass01}
    The vectors $k^*, v^* \in \mathbb{S}^{d-1}$ are orthogonal to each other i.e. ${k^*}^Tv^* = 0$.
\end{assumption}
Note that enforcing orthogonality between $k^*$ and $v^*$ is a mild requirement in high-dimensional embedding spaces, commonly found in TSFMs. The following lemma establishes a reduced representation of the risk of our axial estimator based on the alignments between $k, v, k^*$ and $v^*$.
\begin{lemma} \label{thm:thm01}
    There exists a function $\mathcal{R}^{<}_{\lambda_r, \lambda_c} \colon \mathbb{R}^5 \to \mathbb{R}$ such that for any $(k, v) \in \left(\mathbb{S}^{d-1}\right)^2$,
    \begin{equation}\label{eq:red_risk}
    \mathcal{R}_{\lambda_r, \lambda_c}(k, v) = \mathcal{R}^<_{\lambda_r, \lambda_c}(\kappa, \nu, \theta, \eta, \rho),
    \end{equation}
    where $\kappa = k^T k^*$, $\nu = v^T v^*$, $\theta = v^T k^*$, $\eta = k^T v^*$ and $\rho = k^T v$. An expression for $\mathcal{R}^<_{\lambda_r, \lambda_c}$ is given in Appendix~\ref{sec:app_sec04}.
\end{lemma}

Next, the Bayes risk for the regression task in \eqref{eq:task} is bounded below by the variance $\epsilon^2$ of the noise $\xi$. Building on Lemma~\ref{thm:thm01}, Proposition~\ref{cor:cor02} demonstrates that our oracle estimator $T^{(k^*,v^*)}_{\lambda_r, \lambda_c}$ recovers this theoretical lower bound under a specific asymptotic regime.
\begin{proposition} \label{cor:cor02}
    Assume a joint asymptotic scaling where $\lambda_r \ell \to \infty$, $\lambda_c \ell \to \infty$ and $\lambda_r = o\left(M^{-\frac{1}{2}} N^{-\frac{1}{2}}\right)$. Then,
    \[
    \mathcal{R}_{\lambda_r, \lambda_c}(k^*, v^*) \to \epsilon^2.
    \]
    Consequently, the oracle axial predictor $T^{(k^*,v^*)}_{\lambda_r, \lambda_c}$ achieves asymptotic Bayes optimality.
\end{proposition}

Note that Proposition~\ref{cor:cor02} does not impose any scaling on $d$, which is expected given Lemma~\ref{thm:thm01}. We provide a non-rigorous intuition for the imposed scalings. First, we wish $X_{I_0J_0}$ to pass through the first layer asymptotically unaltered, that is,
\begin{equation*}
    X_{I_0 J_0} \approx \Xhalf_{I_0} = \erf(\underbrace{\lambda_r X_{I_0 J_0}^Tk^*}_{\sim \lambda_r \ell}) X_{I_0J_0} + \sum_{j \neq J_0}\erf(\lambda_r X_{I_0j}^Tk^*) X_{I_0j}.
\end{equation*}
Therefore, we need $\lambda_r \ell \to \infty$ to ensure that the argument of the first $\erf$ goes to $\infty$ and $\lambda_r \to 0$ to ensure that the others go to 0. In an analogous manner, we wish $X_{I_0 J_0}$ to pass unaltered through the second layer, which leads to $\lambda_c \ell \to \infty$. Lastly, in the expansion of $T^{(k^*,v^*)}_{\lambda_r, \lambda_c}$, there appear $(M-1)N$ noise terms of the form $C \cdot \erf(\lambda_r X_{ij}^T k^*)X_{ij}^Tv^*$, with $C < 1$, one for each $(i, j)$ with $i \neq I_0$. Multiple scalings could make the sum of these terms go to 0, one of them being $\lambda_r = o\left(M^{-\frac{1}{2}} N^{-\frac{1}{2}}\right)$, as shown in the proof in Appendix~\ref{sec:app_sec04}.

\subsection{Removing the inner non-linearity is suboptimal} \label{sec:sec04b}
To fully appreciate the asymptotic optimality of our oracle predictor, it is instructive to compare it against semi-linear models, which replace the inner non-linearity with a fixed linear combination layer governed by $\beta \in \mathbb{R}^N$. This justifies the necessity for the inner non-linearity in the axial attention layer in order to achieve asymptotic optimality. More precisely, consider the \textit{semi-linear} predictor
\begin{equation} \label{eq:semi-linear_T}
    T_{\beta,\lambda_c}^{(k, v)}(\X) \coloneqq \erf(\lambda_c \Xhalf k) ^T \; \Xhalf v, \;  \text{ where } \Xhalf_i \coloneqq \beta^T \; \X_{i \cdot}\ .
\end{equation}
Then we have the following sub-optimality statement.
\begin{proposition} \label{prop:prop03}
    Let $p_j = \mathbb{P}(J_0 = j)$ for $j \in [N]$. Denote by $S_i = \beta^T \; \X_{i \cdot}$ Then, for any function $f \colon \mathbb{R}^{dM} \to \mathbb{R}$ in $L^2(\mathbb{R}^{dM})$, the semi-linear predictor achieves a risk strictly larger than $\epsilon^2$ for any choice of $\beta \neq 0$ and hyperparameters $M, N, \ell, \lambda_c$ i.e.
    \begin{equation*}
        \mathbb{E}\left[\left(Y - f(S_1, S_2, \dots, S_M)\right)^2\right] \geq \epsilon^2 + \gamma^2\Bigg(1 - \sum_{j_0=1}^N p_{j_0} \frac{\gamma^2 \beta_{j_0}^2}{\gamma^2 \beta_{j_0}^2 + \sum_{j \neq j_0}\beta_j^2} \Bigg) > \epsilon^2,
    \end{equation*}
    if there is no $j \in [N]$ with $p_j = 1$. We provide a proof of this statement in Appendix~\ref{sec:app_prop3}. 
\end{proposition}

\textbf{Remark 2.} Discarding the outer non-linearity leads to a less restrictive information flow compared to our predictor, and it amounts to a simplification of the standard attention mechanism. Indeed,
\begin{equation*} \label{eq:semi-linear_T_c}
    T_{\lambda_r, \beta}^{(k, v)}(\X) \coloneqq \sum_{i=1}^M \beta_i \sum_{j=1}^N \erf\left(\lambda_r X_{ij}^T k\right) X_{ij}^T v
\end{equation*}
is a modification of the trainable predictor proposed in \citet{marion}, which is asymptotically optimal for $\beta_i = 1$, $i \in [M]$, under the joint scaling $\lambda_r = o\left(M^{-\frac{1}{2}} N^{-\frac{1}{2}}\right)$ and $\lambda_r \ell \to \infty$.

\section{Gradient descent dynamics} \label{sec:sec05}
The last section of our analysis is concerned with the theoretical training dynamics of our predictor \eqref{eq:T} for the risk $\mathcal{R}_{\lambda_r, \lambda_c}(k, v)$. While in practice training minimizes an empirical risk evaluated over a finite dataset, we limit our investigation of the empirical risk minimization to numerical simulations, which confirm our findings about the theoretical risk.

\subsection{Convergence of gradient descent} \label{sec:sec05a}
Lemma~\ref{thm:thm01} shows that on the manifold $\left(\mathbb{S}^{d-1}\right)^2$ the risk can be described through 5 alignment metrics. We want to use this fact and formula since it renders the analysis tractable after further constraints. Therefore, it is natural to employ Projected Riemannian Gradient Descent (PGD) as an optimization algorithm for two main reasons: it keeps the parameters on the manifold and it guarantees that the Euclidean projection back onto the sphere remains well-defined by preventing the parameters from reaching the origin. We emphasize that this geometric constraint is necessary because our reduced risk formulation \eqref{eq:red_risk} is strictly valid only on this manifold.
\begin{algorithm}[H]
\caption{Projected Riemannian Gradient Descent}
\label{alg:pgd}
\begin{algorithmic}[1]
\Require Initialization $(k_0, v_0) \in (\mathbb{S}^{d-1})^2$, step size $\alpha > 0$, inverse temperatures sequence $(\lambda_r^t, \lambda_c^t)_{t \geq 0}$
\For{$t \geq 0$}
    \State $k_{t+1} \gets \text{Proj}_{\mathbb{S}^{d-1}} \left( k_t - \alpha(I_d - k_t k_t^T) \nabla_k \mathcal{R}_{\lambda_r^t, \lambda_c^t}(k_t, v_t) \right)$
    \State $v_{t+1} \gets \text{Proj}_{\mathbb{S}^{d-1}} \left( v_t - \alpha(I_d - v_t v_t^T) \nabla_v \mathcal{R}_{\lambda_r^t, \lambda_c^t}(k_t, v_t) \right)$
\EndFor
\end{algorithmic}
\end{algorithm}

Analyzing the dynamics of the resulting 5-dimensional system is intractable due to the highly complex nature of $\mathcal{R}^<_{\lambda_r, \lambda_c}$. To enable a rigorous analysis, we restrict our focus to scenarios in which the parameters are initialized on the submanifold $\mathcal{M}$ defined by
\begin{equation} \label{eq:manifold}
\mathcal{M} \coloneqq \left\{(k, v) \in \left(\mathbb{S}^{d-1}\right)^2 \; \Big\vert \; k^T v^* = 0, \; v^T k^* = 0, \; k^T v = 0 \right\}.
\end{equation}

Under Assumption~\ref{ass:ass01}, the true underlying parameters $(k^*, v^*)$ naturally reside within $\mathcal{M}$ and, crucially, the training parameters remain on the submanifold throughout the optimization process if they are initialized on it, from Lemma~\ref{lemma:manif_inv} in Appendix~\ref{sec:app_pgd}. Due to this invariance, we only need to characterize the dynamics of the system restricted to $\mathcal{M}$ when initialized on this manifold. Although in practice this cannot be enforced, we show empirically in Section~\ref{sec:sec05b} that penalizing a proxy for the distance to the manifold leads to similar dynamics. This lower-dimensional analysis is mathematically tractable and leads to the primary optimization result of this paper, which guarantees asymptotic recovery of the oracle parameters.

\begin{theorem} \label{thm:pgd_convergence}
    Consider $\tau > 0$ . Then, there exists $\bar{\lambda} > 0$ and $\bar{\alpha} > 0$ such that for constant inverse temperatures $\bar{\lambda} > \lambda_r > 0$ and $\lambda_c = \tau \lambda_r$, for any step size $\alpha \le \bar{\alpha}$ and for a generic initialization $(k_0, v_0) \in \mathcal{M}$,
    \begin{equation*}
        (k_t, v_t) \to \pm(k^*, v^*) \quad \text{as }\; t \to \infty.
    \end{equation*}
\end{theorem}
The convergence to $(-k^*, -v^*)$ is not a problem, as by symmetry 
around 0, 
$T^{(k^*,v^*)}_{\lambda_r, \lambda_c} = T^{(-k^*,-v^*)}_{\lambda_r, \lambda_c}$.

\textbf{Remark 3.} Although some assumptions made throughout the paper are identical to the ones in \citet{marion}, our main result Theorem~\ref{thm:pgd_convergence} is qualitatively different from theirs. In their simplified standard (single-axis) attention layer case,  convergence is always reached, regardless of the inverse temperatures. On the contrary, here, as shown analytically in the proof of Theorem~\ref{thm:pgd_convergence} and numerically in Section~\ref{sec:sec05b} and Appendix~\ref{sec:app_decoupled_lambda}, the values of the inverse temperatures greatly impact the optimization landscape and convergence is not attainable in general. In particular, global convergence is only achieved below a specific threshold $\bar \lambda$.

\textbf{Remark 4 (Training and inference temperatures).}
The inverse temperatures used during optimization and inference serve different purposes and are therefore allowed to differ. During training, we use inverse temperatures $(\lambda_r^{\mathrm{train}},\lambda_c^{\mathrm{train}})$ satisfying the conditions of Theorem~\ref{thm:pgd_convergence}. In particular, they are sufficiently small and satisfy $\lambda_c^{\mathrm{train}}=\tau\lambda_r^{\mathrm{train}}$, which guarantees recovery of $(k^*,v^*)$ up to their joint sign for generic initializations on $\mathcal M$. After training, the recovered parameters are frozen and the inverse temperatures are reset to inference values $(\lambda_r^{\mathrm{inf}},\lambda_c^{\mathrm{inf}})$ satisfying the scaling conditions of Proposition~\ref{cor:cor02}. The resulting predictor is then asymptotically Bayes optimal.

Thus, the two results provide a sequential two-stage guarantee: parameter recovery is performed in an optimization-friendly temperature regime, followed by prediction in a statistically optimal inference regime. This separation is also used in our experiments, where the evaluation risk is computed using fixed inverse temperatures independently of the schedule used for optimization.

\subsection{Understanding the dynamics - numerical experiments} \label{sec:sec05b}
In practice, one cannot use the theoretical risk for optimization using PGD, since there is no closed form formula for it. Furthermore, one also does not have knowledge of what values for $\lambda_r$ and $\lambda_c$ are small enough. In our experiments, we bypass these issues by using batch stochastic PGD for the optimization process and implementing a decreasing schedule for the inverse temperatures.
\begin{figure}[htbp]
    \centering
    \begin{subfigure}[b]{0.24\textwidth}
        \centering
        \includegraphics[width=\textwidth]{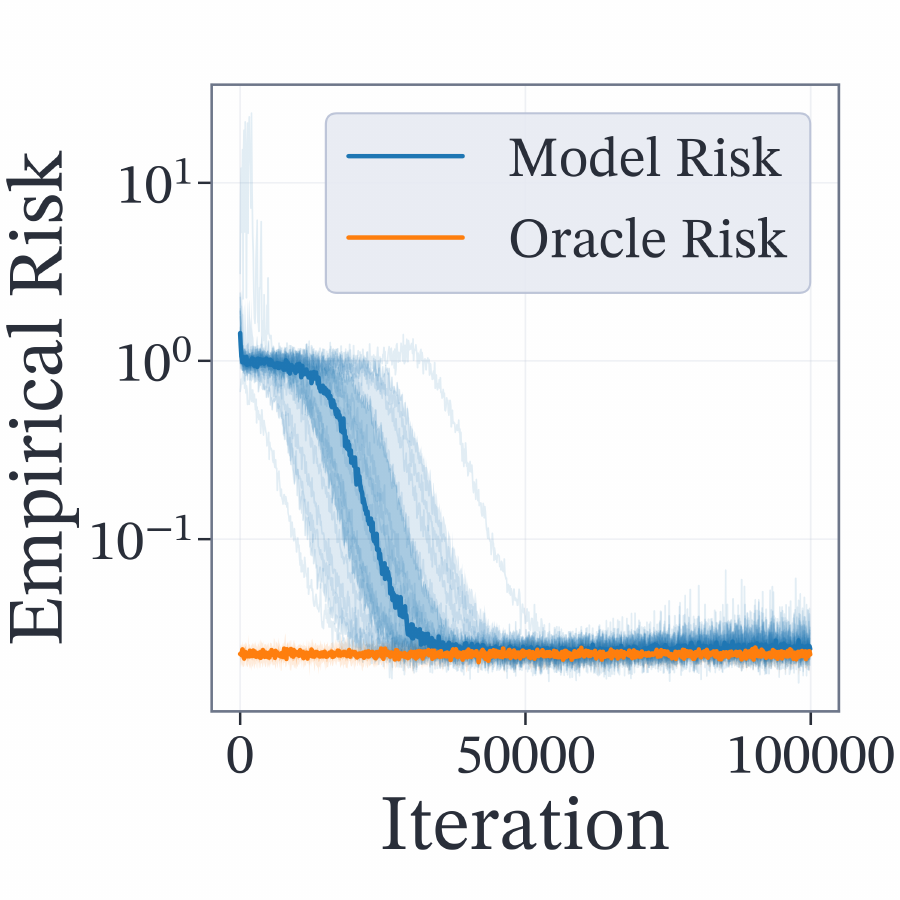}
        \label{fig:on_manifold_loss}
    \end{subfigure}
    \hfill
    \begin{subfigure}[b]{0.24\textwidth}
        \centering
        \includegraphics[width=\textwidth]{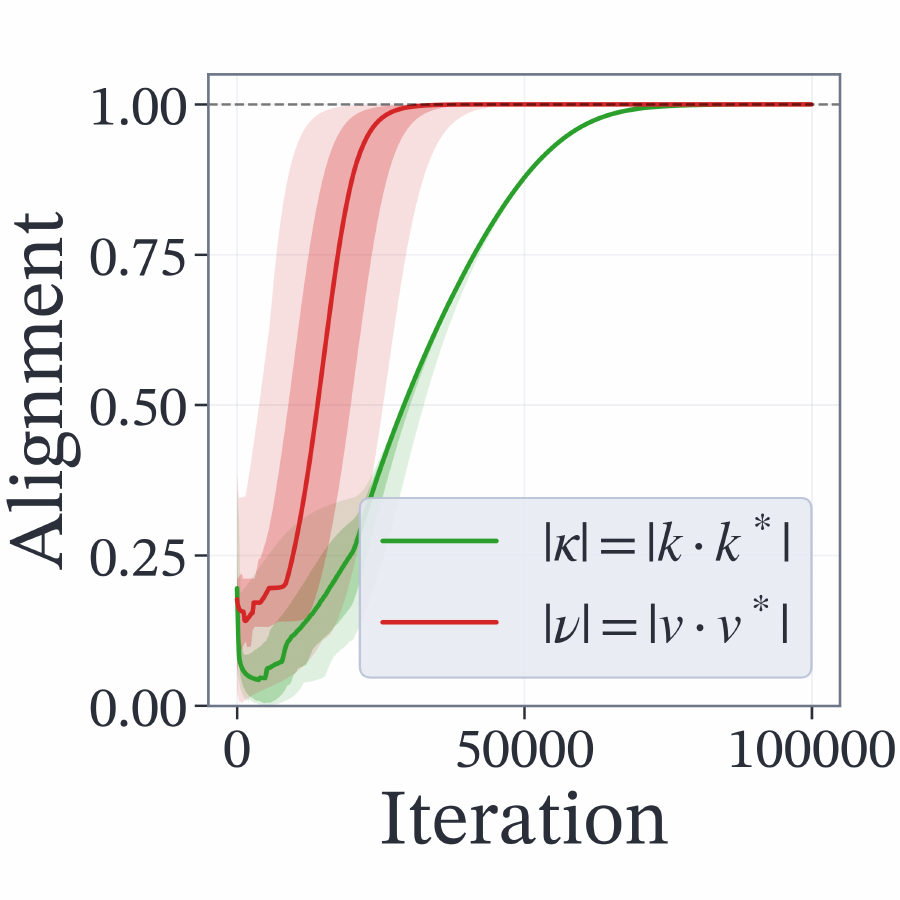}
        \label{fig:on_manifold_alignment}
    \end{subfigure}
    \hfill
    \begin{subfigure}[b]{0.24\textwidth}
        \centering
        \includegraphics[width=\textwidth]{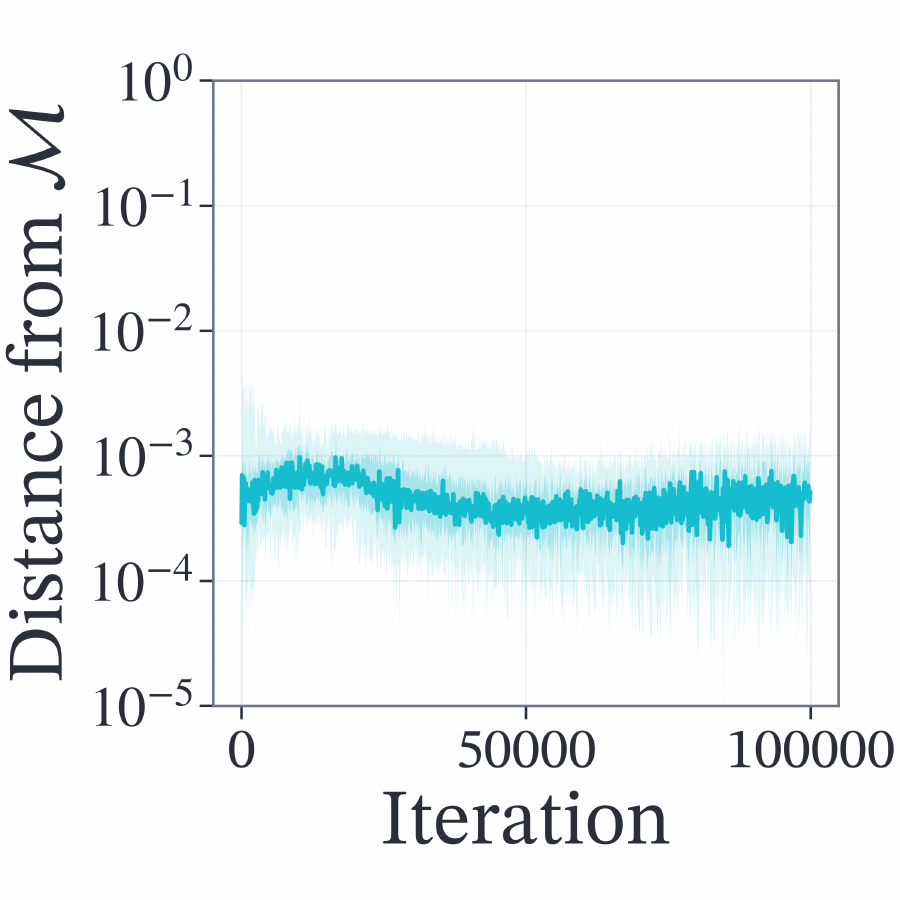}
        \label{fig:on_manifold_distance}
    \end{subfigure}
    \hfill
    \begin{subfigure}[b]{0.25\textwidth}
        \centering
        \includegraphics[width=\textwidth]{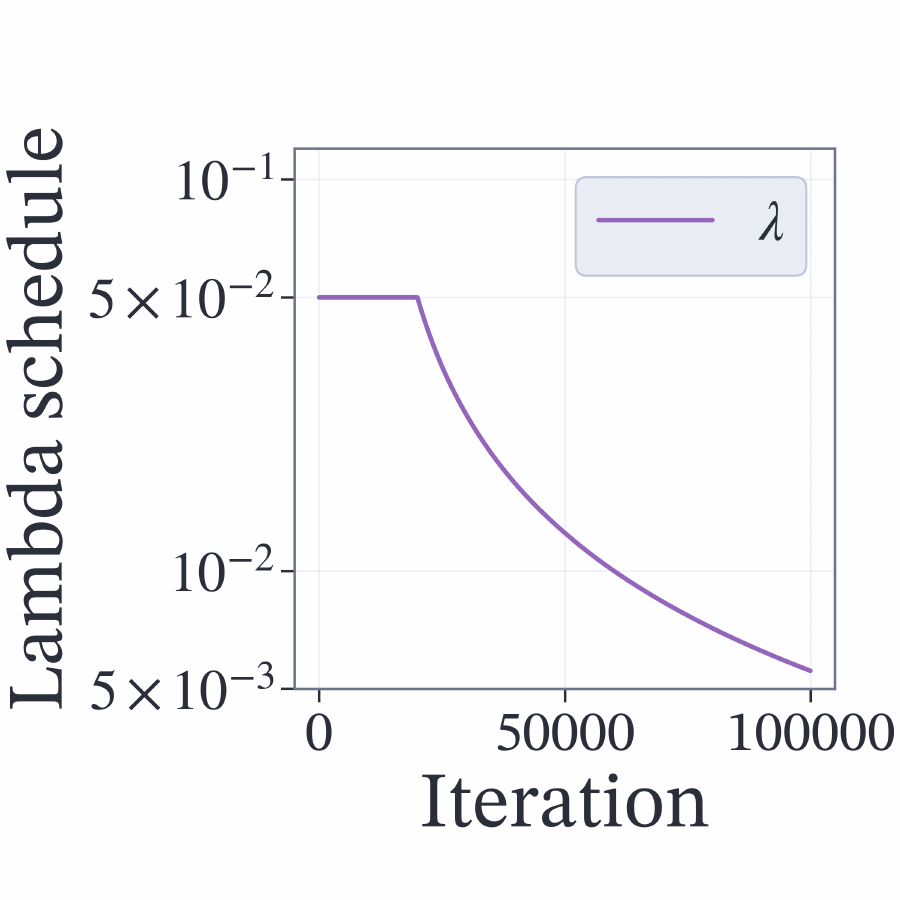}
        \label{fig:on_manifold_lambdas}
    \end{subfigure}
    
    \caption{Dynamics of stochastic PGD to the oracle parameters over 20 independent runs. The solid curve represents the empirical median across the runs and the bands represent the 0.05-0.95 and, respectively, the 0.25-0.75 quantile intervals (although not always visible). \textbf{Left:} The empirical risk obtained by our predictor on the training batches of data using both the training parameters $k$, $v$ and the oracle parameters $k^*$, $v^*$. For computing the oracle and model losses we use the initial inverse temperatures. \textbf{Middle-left:} The alignment between the training and oracle parameters $\kappa = k^Tk^*$, $\nu = v^Tv^*$. \textbf{Middle-right:} The distance from the manifold $\mathcal{M}$ of the training parameters $(k, v)$. \textbf{Right:} The schedule used for $\lambda = \lambda_r = \lambda_c$. The hyperparameters of the experiment are $M = N = 5$, $d = 20$, $\gamma = 1$, $\ell = 100$ and $\lambda_t = 5\cdot 10^{-2} / (1 + \max(t - 2\cdot 10^4, 0) \cdot 10^{-4})$.}
    \label{fig:on_manifold_full}
\end{figure}

Figure~\ref{fig:on_manifold_full} presents numerical results of the optimization obtained for initializations on the manifold using stochastic PGD. We perfectly recover the oracle parameters, up to numerical error, and there is  no departure from the manifold, again, up to numerical error and noise, as was predicted by theory. 

We emphasize the fact that the schedule for the inverse temperatures is crucial for the dynamics. To understand this, in Figure~\ref{fig:vary_lambda_gradients} we run an experiment in which we train 4 predictors using different schedulers: 20k warm-up steps with $\lambda_c = \lambda_r = 0.1$, followed by different values of constant inverse temperatures $\lambda = \lambda_c = \lambda_r$.  The leftmost plots show two different behaviors for $\nu$ and $\kappa$: $\nu$ converges fast, usually before the warm up period ends, while the convergence of $\kappa$ is highly dependent on the value of $\lambda$. The heat maps present the respective partial derivatives of $\mathcal{R}^<_{\lambda_r, \lambda_c}$ for the different values of $\lambda$. 
They show that for $\lambda$ small enough, minus the gradient always points towards the global minima when far from the origin.  For large values of $\lambda$, the dynamics push away from the origin much faster, but the gradients vanish and have the wrong sign close to the boundaries $\kappa = \pm 1$. As a consequence, we see that for those values, $\kappa$ does not converge to one, while it does for $\lambda = 0.01$.
The warm start is needed, as for small $\lambda$, the dynamics would remain stuck around the axis $\kappa = 0$ where the values of the partial derivatives are extremely small, as visible on the plot. 

\begin{figure}[htbp] 
    \centering
    \begin{subfigure}{\textwidth}
        \centering
        \includegraphics[width=1\textwidth]{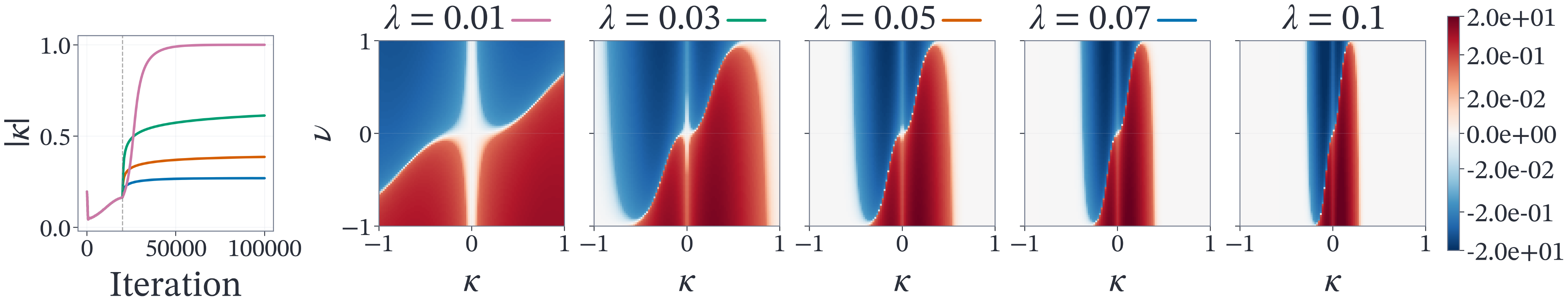}
        \caption{Alignment $\kappa$ and the partial derivative $\partial_\kappa\mathcal{R}^<_{\lambda, \lambda}$.}
        \label{fig:gradients_kappa}
    \end{subfigure}
    
    \vspace{1em}

    \begin{subfigure}{\textwidth}
        \centering
        \includegraphics[width=1\textwidth]{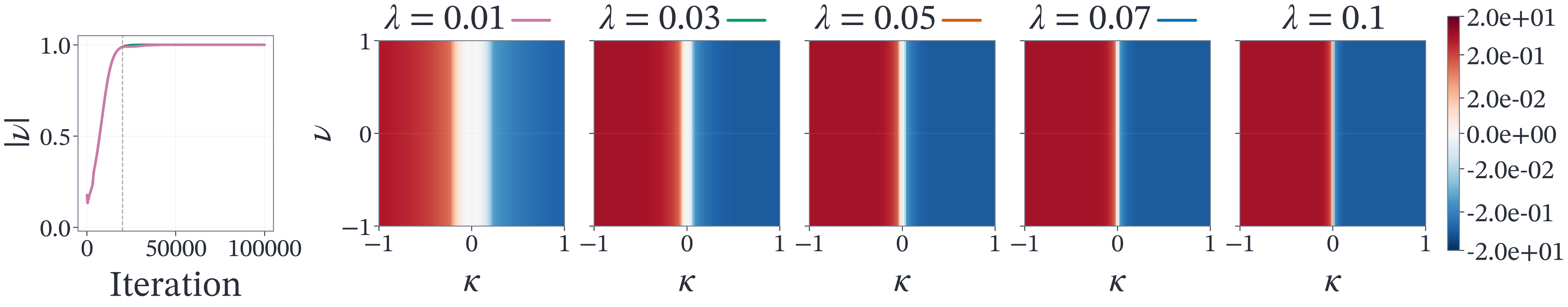}
        \caption{Alignment $\nu$ and the partial derivative $\partial_\nu\mathcal{R}^<_{\lambda, \lambda}$.}
        \label{fig:gradients_nu}
    \end{subfigure}
    \caption{\textbf{Leftmost:} Alignment dynamics during the training of our predictor with an inverse temperature warm-up period of 20k steps using $\lambda_c = \lambda_r = 0.1$, followed by constant $\lambda_c = \lambda_r = \lambda$ for $\lambda \in \{0.01, 0.03, 0.05, 0.07\}$. We plot the median across 20 runs for each value of $\lambda$. The dashed vertical line represents the warm up period. \textbf{Center-left to right:} The numerically computed gradients of $\mathcal{R}^<_{\lambda_r, \lambda_r}(\kappa, \nu)$ for $\lambda_r = \lambda_c = \lambda \in \{0.01, 0.03, 0.05, 0.07, 0.1\}$, $N = 5$ and $\ell = 100$. The gradient, given by \eqref{eq:app_partials}, is approximated using the Monte Carlo method with a sample size of 10,000.}
    \label{fig:vary_lambda_gradients}
\end{figure}

One last and very important problem remains: in practice we cannot enforce the initialization to be on the manifold $\mathcal{M}$, since that would require knowing the oracle parameters $k^*$ and $v^*$. We can ensure that $k^Tv = 0$ at initialization and penalize their inner product throughout the training. Furthermore, a noisy proxy for $v^Tk^*$ is $\frac{1}{\ell} v^T \sum_{i,j}X_{ij}$. Therefore, in the case of initialization outside the manifold, we impose $k^Tv = 0$ at initialization and optimize for the \textit{penalized} risk
\begin{equation} \label{eq:R_penalized}
\textstyle     \bar{\mathcal{R}}_{\lambda_r, \lambda_c}(k,v) = \mathcal{R}_{\lambda_r, \lambda_c}(k,v) + \mu_1 (k^Tv)^2 + \mu_2 \mathbb{E}\big[\big(v^T\sum_{i,j} X_{ij}\big)^2\big],
\end{equation}
where $\mu_1, \mu_2 \in (0, \infty)$ are fixed penalization coefficients.
A similar use of regularization to encourage the orthogonality conditions defining an invariant manifold is explored by \citet{attentionclustering} for attention-based clustering, using a different penalty to handle initializations outside the manifold.

Indeed, Figure~\ref{fig:comparison_manifold} presents the results obtained when optimizing for this penalized empirical risk from a generic random initialization on $(\mathbb{S}^{d-1})^2$. The top row (Figure~\ref{fig:no_manifold_erf}) displays the training dynamics for our trainable predictor \eqref{eq:T}, employing the $\erf$ non-linearity. We observe that the penalty terms effectively guide the parameters towards the invariant manifold $\mathcal{M}$ throughout the training and the convergence is slower, even though we have used a higher initial inverse temperature. The lower row (Figure~\ref{fig:no_manifold_softmax}) presents the identical experimental setup applied to the predictor \eqref{eq:T_sigma} with $\sigma = \text{softmax}$. Although the training dynamics are not identical, both models converge on the manifold $\mathcal{M}$ and there are phase transitions in the empirical loss curves for both models. In the case of the $\erf$ model, the phase transition appears to be correlated with $\nu$ reaching near-perfect alignment, while in the case of the softmax model, the phase transition presents itself when both $\nu$ and $\kappa$ reach near-perfect alignment.

As we have shown in the present experiments, the inverse temperature schedule is of the utmost importance to the performance of the training procedure, and finding a suitable one that performs well across all initializations adds a layer of complexity to the optimization process. A more in-depth analysis of and further details related to the numerical experiments can be found in Appendix~\ref{sec:app_numerical}.

\begin{figure}[htbp] 
    \centering
    \begin{subfigure}{\textwidth}
        \centering
        \includegraphics[width=0.24\textwidth]{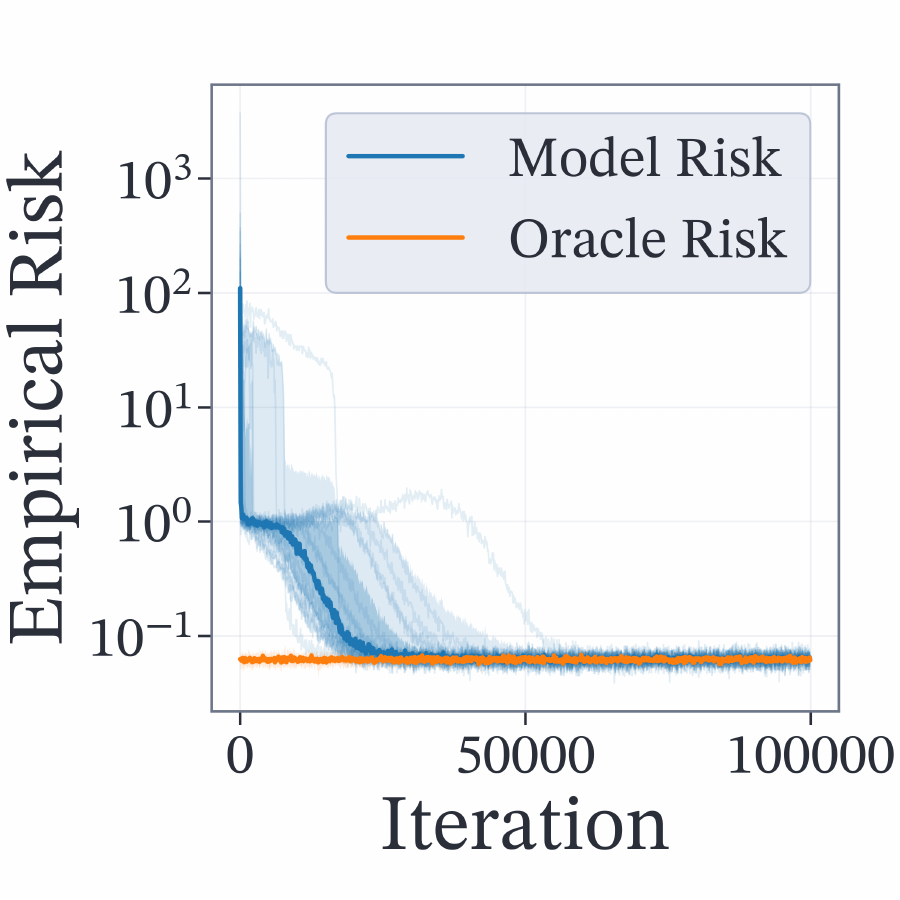}\hfill
        \includegraphics[width=0.24\textwidth]{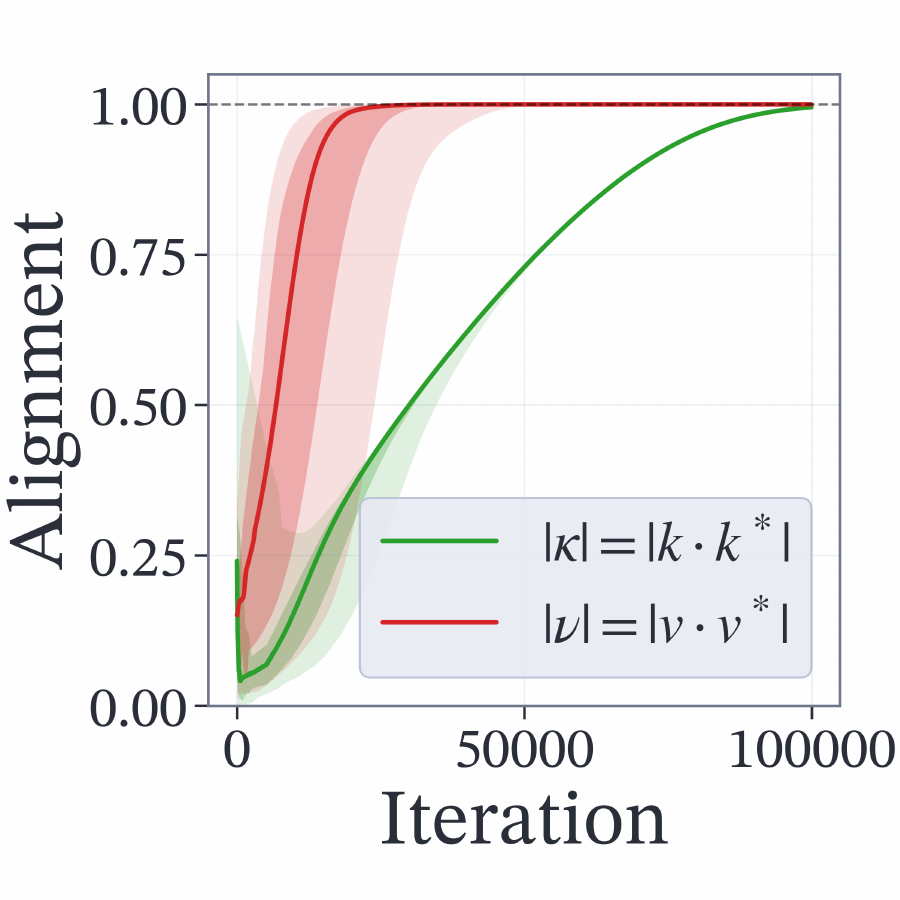}\hfill
        \includegraphics[width=0.24\textwidth]{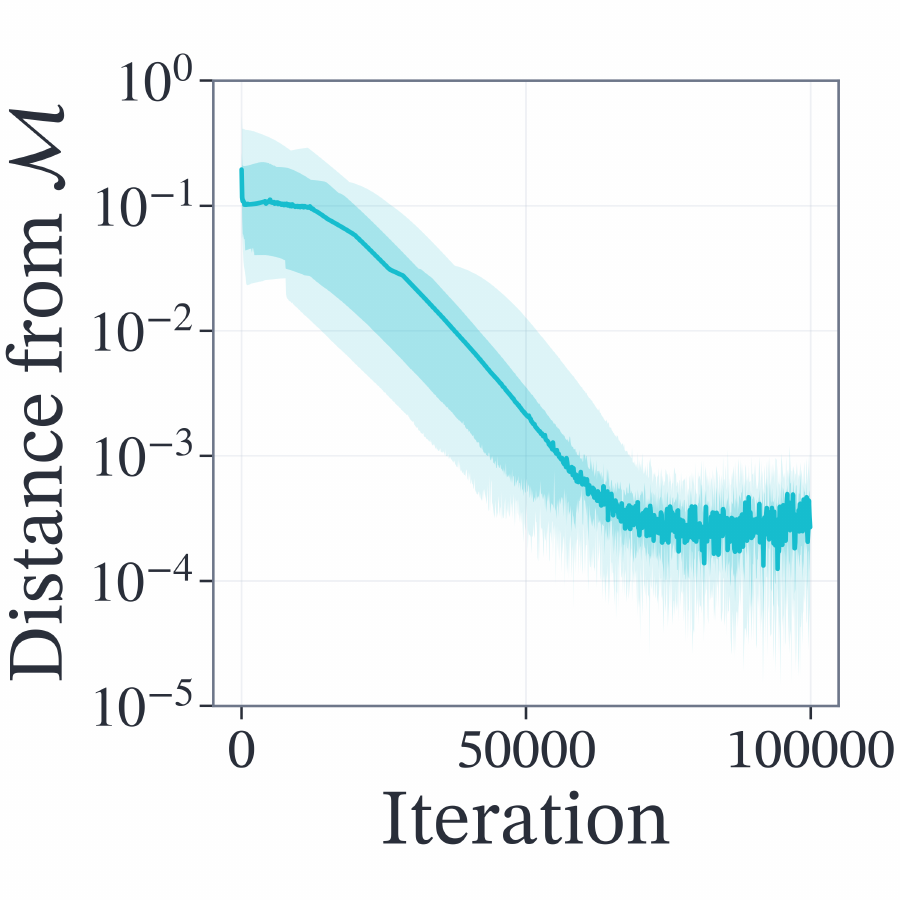}\hfill
        \includegraphics[width=0.24\textwidth]{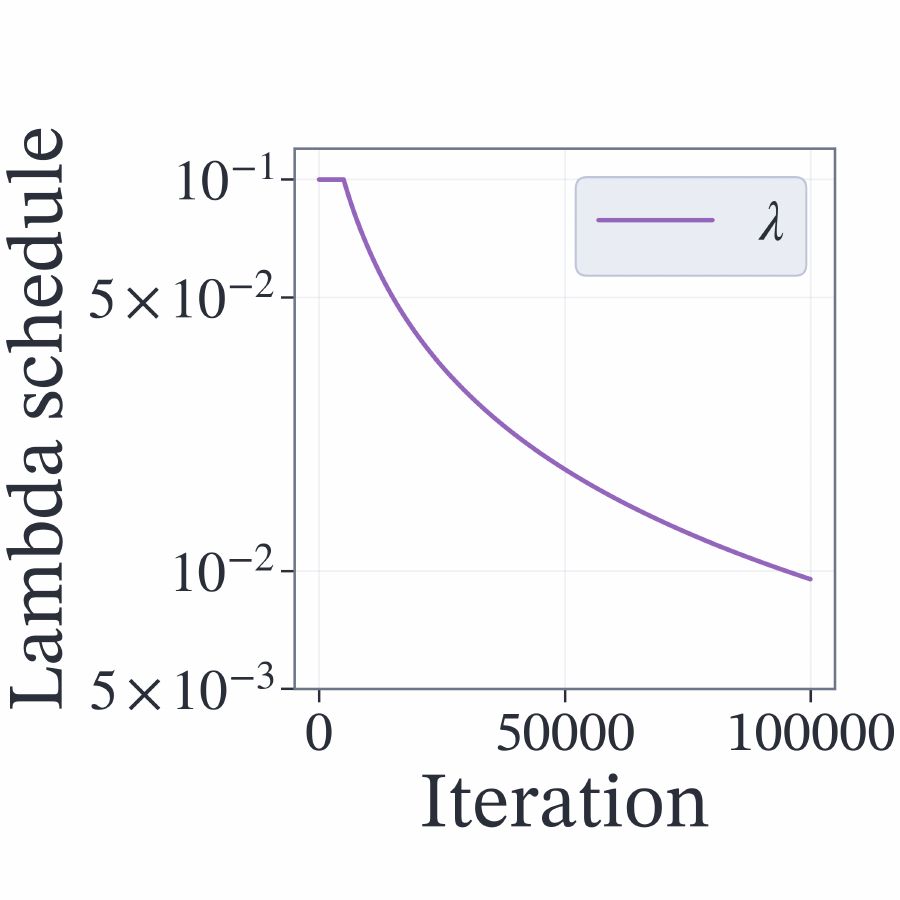}
        \caption{Training dynamics of the trainable predictor \eqref{eq:T} (i.e. \eqref{eq:T_sigma} with $\sigma = \erf$).}
        \label{fig:no_manifold_erf}
    \end{subfigure}
    
    \vspace{1em}

    \begin{subfigure}{\textwidth}
        \centering
        \includegraphics[width=0.24\textwidth]{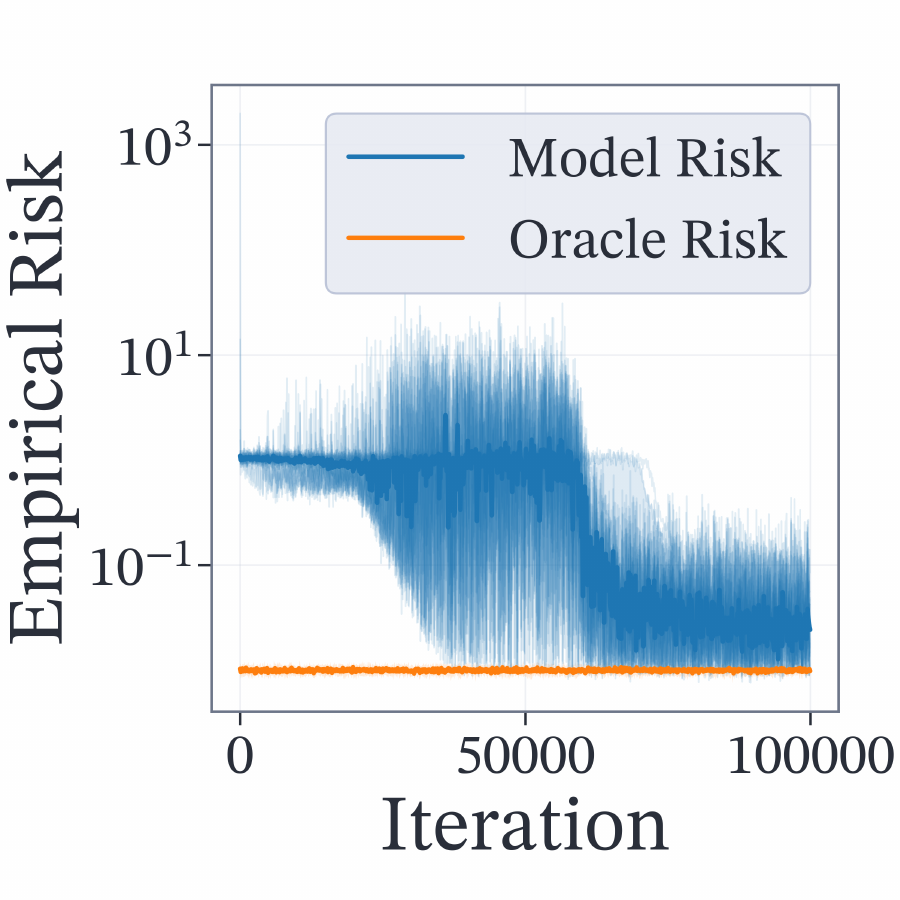}\hfill
        \includegraphics[width=0.24\textwidth]{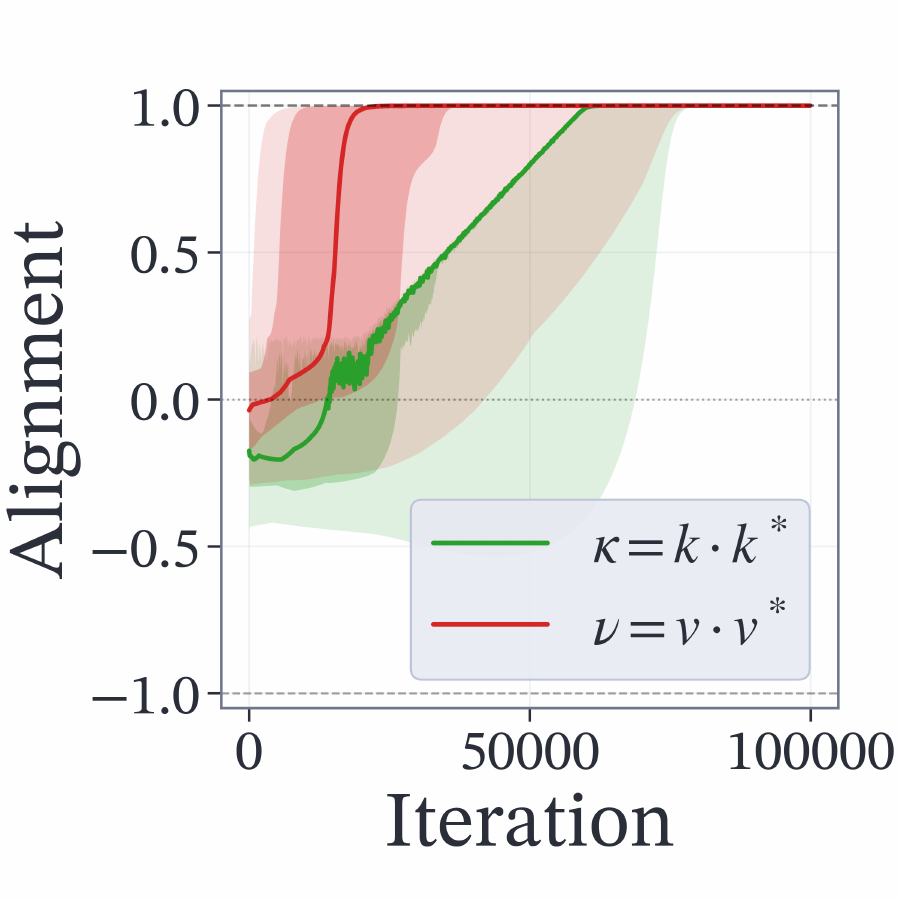}\hfill
        \includegraphics[width=0.24\textwidth]{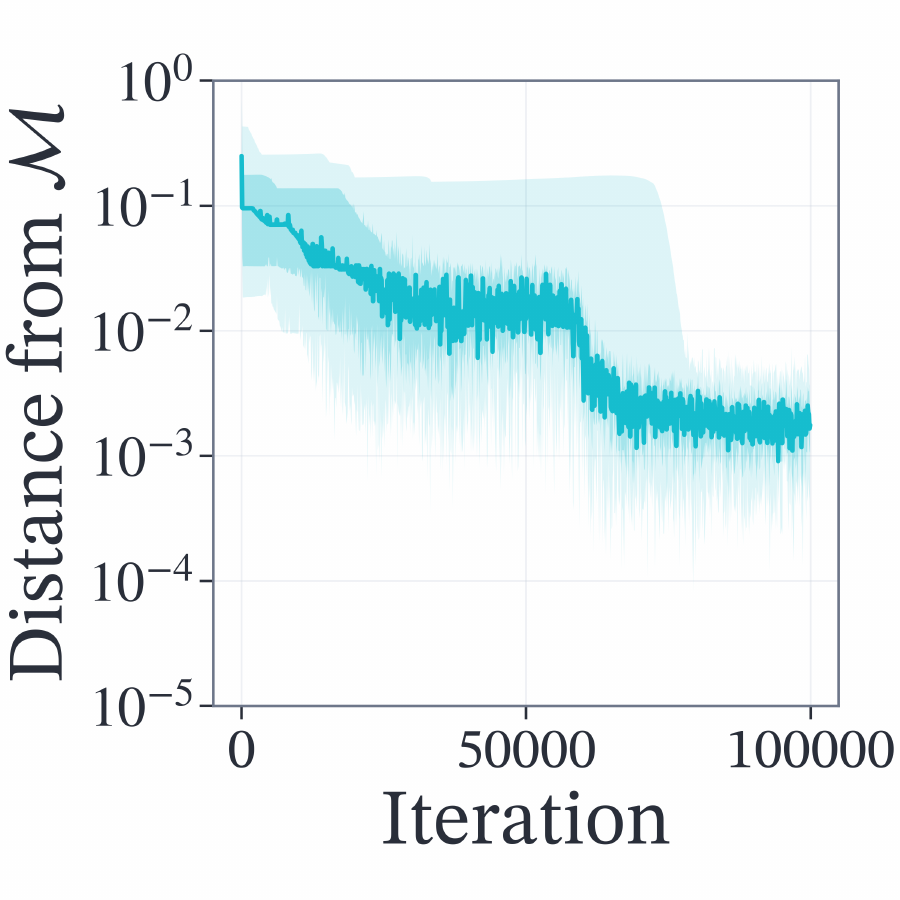}\hfill
        \includegraphics[width=0.24\textwidth]{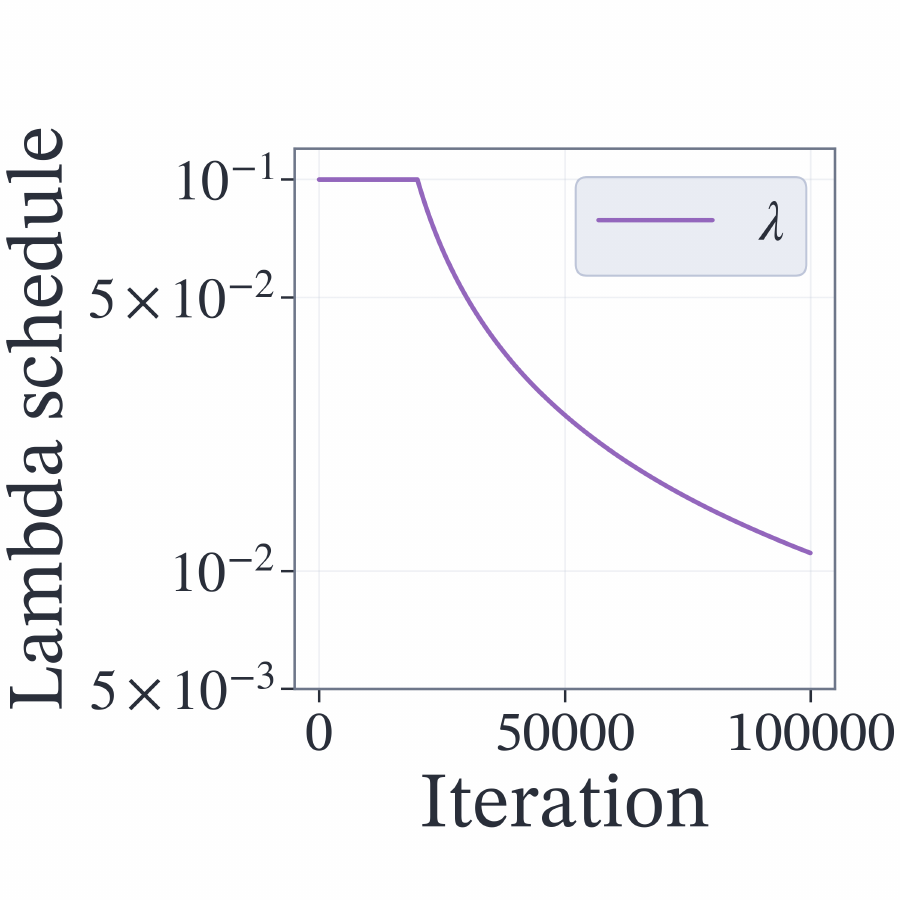}
        \caption{Training dynamics of the softmax predictor \eqref{eq:T_sigma} with $\sigma = \text{softmax}$.}
        \label{fig:no_manifold_softmax}
    \end{subfigure}
    \caption{Convergence of stochastic PGD over 20 runs with random initializations on $\mathbb{S}^{d-1} \times \mathbb{S}^{d-1}$, optimizing for the penalized empirical risk \eqref{eq:R_penalized}. The solid curve represents the empirical median across the runs. The hyperparameters of the experiment are $M = N = 5$, $d = 20$, $\gamma = 1$, $\ell = 100$, $\mu_1 = 10$, $\mu_2 = 1$. The inverse temperature schedules are $\lambda_t =  10^{-1} / (1 + \max(t - 5\cdot 10^3, 0) \cdot 10^{-4})$, for the $\erf$ predictor and $\lambda_t =  10^{-1} / (1 + \max(t - 2\cdot 10^4, 0) \cdot 10^{-4})$ for the softmax predictor. For computing the oracle and model losses we use the initial inverse temperatures.}
    \label{fig:comparison_manifold}
\end{figure}

\section{Conclusion and discussion} \label{sec:sec06}

We presented a rigorous theoretical framework for understanding axial attention in TSFMs by extending the SLR model to 2D grids. We proved that an analytically tractable simplification of axial attention achieves asymptotic Bayes optimality for the oracle parameters, whereas semi-linear models in which the inner non-linearity is removed strictly fail to do so. Furthermore, we established that Projected Gradient Descent converges to these optimal parameters for small enough inverse temperatures. Extensive numerical experiments validate our theoretical findings and highlight the critical role of inverse temperature scheduling in converging to the global minima within the non-convex optimization landscape.

Our statistical framework assumes fixed temporal and spatial dimensions, independent background tokens, and a target that depends on a single localized token. Relaxing these assumptions provides several natural directions for future work. In particular, incorporating autoregressive temporal dynamics and dense cross-channel correlations would better reflect the structure of real multivariate time series, while allowing the dimensions or number of informative tokens to grow would broaden the range of sparse information settings covered by the model. These extensions would introduce additional dependencies into both the statistical model and the optimization landscape and would therefore require analytical tools beyond those used in the present work.

It is also necessary to explicitly reiterate the architectural constraints that make the analysis possible. Our results are heavily dependent on a rank-1 constraint applied to the query and key projection matrices, as well as the substitution of the standard softmax activation with the $\erf$ function and the shared-key constraint $K_r=K_c$. The theoretical results in this paper apply to this simplified $\erf$ predictor, while our experiments with the corresponding softmax predictor provide encouraging evidence that similar optimization behavior persists beyond the activation function used in the analysis. Full TSFMs additionally contain realistic projection matrices, residual connections, normalization layers, and multiple axial-attention blocks; the implications of our results for these architectures are therefore qualitative. Nevertheless, our Single Location Regression model preserves the core structural challenge of identifying and exploiting a sparsely located informative token, making it highly relevant for the theoretical study of axial attention in TSFMs.

Finally, while our current focus is restricted to a single informative token, the same framework suggests a natural extension to multiple-token regression through multi-head attention. If the target depends on several informative locations, different heads could learn distinct localization directions and specialize in selecting different tokens, while their outputs could be combined to recover the target. Under suitable identifiability and separation assumptions, parts of the statistical analysis could then be applied head-wise, following an approach similar to the one developed here. The optimization problem would, however, become substantially richer: the heads would be trained jointly, their roles would be interchangeable, and several heads could compete for or collapse onto the same informative location. We leave this extension, together with the study of the corresponding statistical and optimization guarantees, for future work.

\paragraph{AI use statement.} 
We have used generative AI tools to assist with writing and cleaning code, and with editing the manuscript.
The authors reviewed all AI-assisted outputs and take responsibility for the final content of this work.

\paragraph{Acknowledgments.}
We thank Pierre Marion, whose insightful comments helped us improve the final version of this paper. The work of Andrei Tiberiu Pantea and Aymeric Dieuleveut is supported by French State aid managed by the Agence Nationale de la Recherche (ANR) under the France 2030 program ANR-23-IACL-0005 (Hi! PARIS).

\bibliographystyle{plainnat}
\bibliography{references}


\appendix

\section{Proof of Lemma~\ref{lem:equivalence} and further discussion about architectural constraints} \label{sec:app_equivalence}

We recall the definition of the standard axial self-attention layer.
\begin{equation*}
    T_{\lambda_r, \lambda_c, \sigma}^{(Q_{r,c}, K_{r,c}, V_{r,c}, O_{c})}(\X)_{\cdot j} = \sigma\left(\lambda_c \underbrace{\Xhalf_{\cdot j} Q_c}_{M \times p_c} \underbrace{K_c^T \left(\Xhalf_{\cdot j}\right)^T}_{p_c \times M}\right) \underbrace{\Xhalf_{\cdot j} V_c}_{M \times p_c} \underbrace{O_c}_{p_c \times o_c},
\end{equation*}
where the initial row-wise computation is:
\begin{equation*}
    \Xhalf_{i \cdot} = \sigma(\lambda_r \underbrace{\X_{i \cdot} Q_r}_{N \times p_r} \underbrace{K_r^T \X^T_{i \cdot}}_{p_r \times N}) \underbrace{\X_{i \cdot} V_r}_{N \times d}
\end{equation*}

In order to justify our simplified estimator \eqref{eq:T_sigma}, we explicitly map it to the mechanism of an actual axial attention block. As mentioned in the main paper, sequences are typically appended with designated \texttt{[REG]} tokens to act as either readout tokens or attention sinks in the context of TSFMs. Following this convention, we assume that the input's final column consists of specialized \texttt{[REG\_ROW]} tokens. To collapse the 2D grid into a single regression scalar, we then incorporate an additional \texttt{[REG\_COL]} token at the specific index $(M + 1, N + 1)$, which acts as a global aggregator during the subsequent column-wise operation. We further assume the two \texttt{[REG]} tokens do not attend to themselves, which is a very light assumption, since a priori these tokens contain no information. Isolating the output for this specific token yields the following
\begin{equation*}
    \Xhalf_{i ,N+1} = \sum_{l=1}^{N} \sigma(\lambda_r X_{\text{[REG\_ROW]}}^TQ_r K_r^T \X_{i \cdot} ^T)_l X_{i l}^TV_r  = \sigma(\lambda_r a K_r^T \X_{i \cdot}^T) \X_{i \cdot} \in \mathbb{R}^d
\end{equation*}
if we assume that the projection $V_r\in \mathbb{R}^{d \times d}$ is the identity matrix and encapsulate the query interaction as a vector $a \coloneqq X_{\text{[REG\_ROW]}}^TQ_r \in \mathbb{R}^{1 \times p_r}$. By applying an analogous definition for the column-wise query $a^{(1/2)} \coloneqq X_{\text{[REG\_COL]}}^T Q_c \in \mathbb{R}^{1 \times p_c}$, the output of the final attention block becomes:
\begin{equation*}
    T_{\lambda_r, \lambda_c, \sigma}^{(Q_{r,c}, K_{r,c}, V_{r,c}, O_{c})}(\X)_{M+1, N+1} = \sum_{i = 1}^{M} \sigma\left(\lambda_c a^{(1/2)} K_c^T \Xhalf_{i , N+1}\right)_i \left(\Xhalf_{i , N+1}\right)^T V_c O_c 
\end{equation*}

To perfectly align this with our theoretical model, we enforce rank-1 restrictions on the projections, setting $p_r = p_c = 1$. With these constraints in place, along with $o_c = 1$, $O_c = 1$, and shared keys $K_r = K_c$, the full axial mechanism cleanly reduces to our formulated predictor:
\begin{equation*} 
    T^{(Q_{r,c}, K_{r,c}, V_{r,c}, O_{c})}_{\lambda_r, \lambda_c}(\X)_{M+1,N + 1} = T^{(K_c, V_c)}_{\lambda_r a, \lambda_c a^{(1/2)}, \sigma}(\X)
\end{equation*}

\paragraph{Discussion. }

In this paragraph, we address the main architectural constraints and modifications imposed in our analysis.
\begin{itemize}
    \item \textbf{\texttt{REG} tokens.} The addition of \texttt{REG} tokens is not a relevant modification, since any transformer layer whose output must be a single value, be it scalar of vector, will implement a global aggregation mechanism for the input tokens and will be trained for it. \texttt{REG} tokens remove the inductive biases that come with pooling methods or positional decoding and allow for a learned aggregation mechanism.
    \item \textbf{Rank-1 restriction on projection matrices.} This constraint is not at all harsh in the context of SLR, where there exists a single informative direction in the latent space, namely $k^*$. Moreover, \citet{yu2025understanding} has shown that patch-based TSFMs admit low-rank approximations of the projection matrices.
    \item \textbf{Shared keys.} $K_r = K_c$ is indeed a strong constraint in practice. However, in the particular case of the studied SLR, there is a single special direction in the embedding space, $k^*$, and, heuristically, token selection should be made according to this direction. 
    \item \textbf{Non-linearity.} For the theoretical analysis, we have restricted the value of $\sigma$ to $\erf$, maintaining the same convention as in \citet{marion}. However, only a couple of properties of the $\erf$ function are essential to the present proofs, such as global boundedness of $\erf$ and of its derivatives, the asymptotics $\lim_{x\to\infty}\erf(x) = 1$ and $\lim_{x\to-\infty}\erf(x) = -1$, $\erf(x) = 0$ etc. Extending the proofs and the numerical analysis to other activation functions seems feasible and would be an interesting further research direction.
    Furthermore, the statistical role of the activation function in single-location regression is investigated by \citet{duranthon2026statistical}, who establish advantages of softmax over alternative activations in a single-axis setting. Understanding how these comparisons extend to sequential axial attention is an interesting direction for future work.
\end{itemize}

\section{Proofs for Section~\ref{sec:sec04}} \label{sec:app_sec04}
In what follows, we shall denote by
\begin{equation*}
    \mathcal{R}^<(\kappa, \nu) \coloneqq \mathcal{R}^<(\kappa, \nu, 0, 0, 0),
\end{equation*}
the value of the risk when restricted to $\mathcal{M}$.

\subsection{Proofs of Lemma~\ref{thm:thm01} and of Proposition~\ref{cor:cor02}} \label{sec:app_sec04a}

We recall the formula for the risk
\begin{equation} \label{eq:app_risk}
\mathcal{R}_{\lambda_r, \lambda_c}(k, v) = \mathbb{E}\left[\left(Y - \sum_{i = 1}^M \erf\left(\lambda_c \sum_{j=1}^N \erf(\lambda_r X_{ij}^Tk)X_{ij}^Tk\right) \; \sum_{j=1}^N \erf(\lambda_r X_{ij}^Tk)X_{ij}^T v\right)^2\right]
\end{equation}
Let us now denote by $G_{ij}^u = X_{ij}^T u$, for $u \in \mathbb{R}^d$. Note that we consider $X_{I_0 J_0} = \ell k^* + \gamma \tilde{X}_{I_0 J_0}$ and $G_{I_0 J_0}^u = \tilde{X}_{I_0 J_0}^T u$. Then, our predictor only depends on $\kappa$, $\theta$ and the alignment random variables $G_{ij}^v, G_{ij}^{v^*}, G_{ij}^k$ which are i.i.d. and follow the distribution
\[
\begin{pmatrix}
G_{ij}^k \\
G_{ij}^v \\
G_{ij}^{v^*}
\end{pmatrix}
\sim
\mathcal{N}\left(
0,
\begin{pmatrix}
1 \quad \rho \quad \eta \\
\rho \quad 1 \quad \nu \\
\eta \quad \nu \quad 1
\end{pmatrix}
\right).
\]
We further assume, without loss of generality, that $I_0 = J_0 = 1$. Then $Y = \gamma G_{11}^{v^*} + \xi$ and, dropping the subscripts of $\mathcal{R}_{\lambda_r, \lambda_c}$, one indeed has
\begin{equation} \label{eq:app01}
\begin{split}
    \mathcal{R}(k,v) &= \mathbb{E}\left[\left(\gamma G_{11}^{v^*} + \xi - T_{\lambda_r, \lambda_c}^{(k,v)}(\X)\right)^2\right] \\
    &=\epsilon^2 + \mathbb{E}\left[\left(\gamma G_{11}^{v^*} - T_{\lambda_r, \lambda_c}^{(k,v)}(\X)\right)^2\right] \\
    &= \epsilon^2 + \gamma^2 - 2\gamma\,\mathbb{E}\left[G_{11}^{v^*} T_{\lambda_r, \lambda_c}^{(k,v)}(\X)\right] + \mathbb{E}\left[T_{\lambda_r, \lambda_c}^{(k,v)}(\X)^2\right].
\end{split}
\end{equation}
where
\begin{equation} \label{eq:app02}
\begin{split}
    T_{\lambda_r, \lambda_c}^{(k,v)}(\X) & = \sum_{i = 2}^M \erf\left(\lambda_c \sum_{j=1}^N \erf(\lambda_r G_{ij}^k) G_{ij}^k \right) \sum_{j=1}^N \erf(\lambda_r G_{ij}^k) G_{ij}^v \\
    & + \erf\left(\lambda_c \erf\left(\lambda_r \kappa \ell + \lambda_r \gamma G_{11}^k \right)\left(\kappa \ell + \gamma G_{11}^k \right) + \lambda_c \sum_{j=2}^N \erf\left(\lambda_r G_{1j}^k\right)G_{1j}^k\right) \cdot \\
    & \quad \cdot \left(\erf\left(\lambda_r \kappa \ell + \lambda_r \gamma G_{11}^k \right) \left( \theta \ell + \gamma G_{11}^v \right) + \sum_{j=2}^N \erf\left(\lambda_r G_{1j}^k\right) G_{1j}^v \right)
\end{split}
\end{equation}

Since Equation~\eqref{eq:app01} and Equation~\eqref{eq:app02} only depend on the 5 parameters $\kappa, \nu, \rho, \eta, \theta$, this ends the proof of Lemma~\ref{thm:thm01}. On the manifold $\mathcal{M}$, where $\rho = \theta = \eta = 0$, we can simplify the expression by dropping the $\theta \ell$ term in the last row of Equation~\eqref{eq:app02}. We continue our proof by analyzing each term in Equation~\eqref{eq:app01} individually.

The first term can be rewritten as
\begin{align} \label{eq:app03_nomanifold}
    \mathbb{E}&\left[G_{11}^{v^*} T_{\lambda_r, \lambda_c}^{(k,v)}(\X)\right] = \nonumber \\
    &= \textcolor{red}{\sum_{i = 2}^M \mathbb{E}\left[\erf\left(\lambda_c \sum_{j=1}^N \erf(\lambda_r G_{ij}^k) G_{ij}^k \right) \sum_{j=1}^N \erf(\lambda_r G_{ij}^k) G_{ij}^vG_{11}^{v^*}\right]} \nonumber \\
    &\phantom{=} + \textcolor{blue}{\mathbb{E}\left[\erf\left(\lambda_c \erf\left(\lambda_r \kappa \ell + \lambda_r \gamma G_{11}^k \right)\left(\kappa \ell + \gamma G_{11}^k \right) + \lambda_c \sum_{j=2}^N \erf\left(\lambda_r G_{1j}^k\right)G_{1j}^k\right) \cdot \right.} \nonumber \\
    &\phantom{=+} \textcolor{blue}{\left.\cdot \left(\erf\left(\lambda_r \kappa \ell + \lambda_r \gamma G_{11}^k \right) (\theta \ell + \gamma G_{11}^v) G_{11}^{v^*} + \sum_{j=2}^N \erf\left(\lambda_r G_{1j}^k \right) G_{1j}^v G_{11}^{v^*}\right)\right]}
\end{align}
On the manifold $\mathcal{M}$, this becomes:
\begin{align} \label{eq:app03}
    \mathbb{E}&\left[G_{11}^{v^*} T_{\lambda_r, \lambda_c}^{(k,v)}(\X)\right] = \nonumber \\
    &= \textcolor{red}{\sum_{i = 2}^M \mathbb{E}\left[\erf\left(\lambda_c \sum_{j=1}^N \erf(\lambda_r G_{ij}^k) G_{ij}^k \right) \sum_{j=1}^N \erf(\lambda_r G_{ij}^k) G_{ij}^vG_{11}^{v^*}\right]} \nonumber \\
    &\phantom{=} + \textcolor{blue}{\mathbb{E}\left[\erf\left(\lambda_c \erf\left(\lambda_r \kappa \ell + \lambda_r \gamma G_{11}^k \right)\left(\kappa \ell + \gamma G_{11}^k \right) + \lambda_c \sum_{j=2}^N \erf\left(\lambda_r G_{1j}^k\right)G_{1j}^k\right) \cdot \right.} \nonumber \\
    &\phantom{=+} \textcolor{blue}{\left.\cdot \left(\erf\left(\lambda_r \kappa \ell + \lambda_r \gamma G_{11}^k \right) \gamma G_{11}^v G_{11}^{v^*} + \sum_{j=2}^N \erf\left(\lambda_r G_{1j}^k \right) G_{1j}^v G_{11}^{v^*}\right)\right]} \nonumber \\
    &= \textcolor{red}{\sum_{i = 2}^M \sum_{j=1}^N\mathbb{E}\left[\erf\left(\lambda_c \sum_{j=1}^N \erf(\lambda_r G_{ij}^k) G_{ij}^k \right)\erf(\lambda_r G_{ij}^k)G_{ij}^v\right] \underbrace{\mathbb{E}\left[G_{11}^{v^*}\right]}_{=0}} \nonumber\\
    &\phantom{=} + \textcolor{blue}{\mathbb{E}\left[\erf\left(\lambda_c \erf\left(\lambda_r \kappa \ell + \lambda_r \gamma G_{11}^k \right)\left(\kappa \ell + \gamma G_{11}^k \right) + \lambda_c \sum_{j=2}^N \erf\left(\lambda_r G_{1j}^k\right)G_{1j}^k\right) \cdot \right.} \nonumber \\
    &\phantom{= +} \textcolor{blue}{\cdot \; \erf\left(\lambda_r \kappa \ell + \lambda_r \gamma G_{11}^k \right)\Bigg] \gamma \underbrace{\mathbb{E}\left[G_{11}^v G_{11}^{v^*}\right]}_{= \nu}} \nonumber \\
    &\phantom{=} + \textcolor{blue}{ \sum_{j=2}^N\mathbb{E}\left[\erf\left(\lambda_c \erf\left(\lambda_r \kappa \ell + \lambda_r \gamma G_{11}^k \right)\left(\kappa \ell + \gamma G_{11}^k \right) + \lambda_c \sum_{j=2}^N \erf\left(\lambda_r G_{1j}^k\right)G_{1j}^k\right)\right. \cdot} \nonumber \\
    &\phantom{= + \sum \mathbb{E}} \textcolor{blue}{\cdot \; \erf\left(\lambda_r G_{1j}^k \right) G_{1j}^v\Bigg] \underbrace{\mathbb{E}\left[G_{11}^{v^*}\right]}_{=0}} \nonumber \\
    &= \gamma \nu \mathbb{E}\!\left[ \erf\!\left(\lambda_c \erf\!\left(\lambda_r \kappa \ell + \lambda_{r} \gamma G_{11}^k\right) (\kappa \ell + \gamma G_{11}^k) + \lambda_c \sum_{j=2}^{N} \erf\!\left(\lambda_r G_{1j}^k\right) G_{1j}^k\right)\right. \cdot\ \nonumber \\
    &\phantom{=\gamma \nu \mathbb{E}}\cdot \erf\!\left(\lambda_r \kappa \ell + \lambda_{r} \gamma G_{11}^k\right)\Bigg],
\end{align}
where we have used the independence between $G_{11}^{v^*}$ and $G_{ij}^v, G_{ij}^k$. 

Due to the independence between $G_{ij}^\cdot$ and $G_{rs}^\cdot$ for $(i,j) \neq (r,s)$, there appear no cross-terms when expanding the expectation of $T_{\lambda_r, \lambda_c}^{(k,v)}(\X)^2$. Therefore, under $\rho = 0$, the second term simplifies as
\begin{align} \label{eq:app04_nomanifold}
    \mathbb{E}\!&\left[ T_{\lambda_r, \lambda_c}^{(k,v)}(\X)^2\right] \nonumber \\
    &= \textcolor{ForestGreen}{\sum_{i=2}^{M} \sum_{j=1}^{N} \mathbb{E}\!\left[\erf^2\!\left(\lambda_c \sum_{j=1}^{N} \erf\!\left(\lambda_r G_{ij}^k\right) G_{ij}^k\right) \cdot \erf^2\!\left(\lambda_r G_{ij}^k\right)(G_{ij}^v)^2\right]} \nonumber \\
    &\phantom{=}+ \textcolor{teal}{\mathbb{E}\!\left[\erf^2\!\left(\lambda_c \erf\!\left(\lambda_r \kappa \ell + \lambda_{r} \gamma  G_{11}^k\right)\!\left(\kappa \ell + \gamma G_{11}^k\right) + \lambda_c \sum_{j=2}^{N} \erf\!\left(\lambda_r G_{1j}^k\right) G_{1j}^k\right)\right.} \nonumber \\
    & \phantom{= + \mathbb{E}}\textcolor{teal}{\cdot \left(\erf^2\!\left(\lambda_r \kappa \ell + \lambda_{r} \gamma G_{11}^k\right) (\theta \ell + \gamma G_{11}^v)^2 + \sum_{j=2}^{N} \erf^2\!\left(\lambda_r G_{1j}^k\right)(G_{1j}^v)^2\right.} \nonumber \\
    &\phantom{= + \mathbb{E}[ \cdot \;} \textcolor{teal}{+ 2 \theta \ell\;  \erf\left(\lambda_r \kappa \ell + \lambda_r \gamma G_{11}^k \right) \sum_{j=2}^N\erf\left(\lambda_r G_{1j}^k\right)G_{1j}^v \Bigg)\Bigg]}
\end{align}
On the manifold $\mathcal{M}$ ($\rho = \theta = \eta = 0$), this becomes
\begin{align*}
    \mathbb{E}\!&\left[ T_{\lambda_r, \lambda_c}^{(k,v)}(\X)^2\right] \nonumber \\
    &= \textcolor{ForestGreen}{\sum_{i=2}^{M} \sum_{j=1}^{N} \mathbb{E}\!\left[\erf^2\!\left(\lambda_c \sum_{j=1}^{N} \erf\!\left(\lambda_r G_{ij}^k\right) G_{ij}^k\right) \cdot \erf^2\!\left(\lambda_r G_{ij}^k\right)\right]} \nonumber \\
    &\phantom{=}+ \textcolor{teal}{\mathbb{E}\!\left[\erf^2\!\left(\lambda_c \erf\!\left(\lambda_r \kappa \ell + \lambda_{r} \gamma  G_{11}^k\right)\!\left(\kappa \ell + \gamma G_{11}^k\right) + \lambda_c \sum_{j=2}^{N} \erf\!\left(\lambda_r G_{1j}^k\right) G_{1j}^k\right)\right.} \nonumber \\
    & \phantom{= + \mathbb{E}}\textcolor{teal}{\cdot \left.\left(\gamma^2 \erf^2\!\left(\lambda_r \kappa \ell + \lambda_{r} \gamma G_{11}^k\right) + \sum_{j=2}^{N} \erf^2\!\left(\lambda_r G_{1j}^k\right)\right)\!\right]}
\end{align*}

In what follows, we shall prove that under the joint scaling $\lambda_r \ell \to \infty$, $\lambda_c \ell \to \infty$ and $\lambda_r = o\left(M^{-\frac{1}{2}} N^{-\frac{1}{2}}\right)$, the risk converges to $\epsilon^2$.
For the oracle parameters, $\kappa = \nu = 1$. Moreover, $\erf(\lambda_r x) x \geq 0$ for all $x \in \mathbb{R}$ and $|\erf(x)| \leq 1$. Therefore, Equation~\eqref{eq:app03} implies
\begin{align*}
    \mathbb{E}\left[G_{11}^{v^*} T_{\lambda_r, \lambda_c}^{(k,v)}(\X)\right] &\geq \gamma \mathbb{E}\Bigg[ \erf\bigg(\lambda_c \erf\left(\lambda_r \ell + \lambda_{r} \gamma G_{11}^k\right) \left( \ell + \gamma G_{11}^k\right)\bigg)\erf\left(\lambda_r \ell + \lambda_{r} \gamma G_{11}^k\right)\Bigg] \nonumber \\
    & \phantom{=} - \gamma \mathbb{P}\left(\gamma G_{11}^k \leq -\ell \right)\\
    & = \gamma \mathbb{E}\left[\erf\left(\frac{\lambda_c}{\lambda_r} \erf(G) G\right) \erf(G)\right] - \gamma \Phi\left(\frac{- \ell}{\gamma}\right),
\end{align*}
where $G \sim \mathcal{N}\left(\lambda_r \ell, \lambda_r^2 \gamma^2\right)$ and $\Phi$ is the cumulative distribution function of $\mathcal{N}(0, 1)$.

Next, since $\vert \erf(x)\vert \leq 1$ and $\erf^2(x) \leq \frac{4}{\pi}x^2$ from Lemma~\ref{lemma:erf_ineq} for all $x \in \mathbb{R}$, one has
\begin{align*}
    \textcolor{ForestGreen}{\sum_{i=2}^{M} \sum_{j=1}^{N} \mathbb{E}\!\left[\erf^2\!\left(\lambda_c \sum_{j=1}^{N} \erf\!\left(\lambda_r G_{ij}^k\right) G_{ij}^k\right) \cdot \erf^2\!\left(\lambda_r G_{ij}^k\right)\right]} &\leq \sum_{i=2}^{M} \sum_{j=1}^{N} \mathbb{E}\!\left[ \erf^2\!\left(\lambda_r G_{ij}^k\right)\right] \nonumber \\
    & \leq \sum_{i=2}^{M} \sum_{j=1}^{N} \frac{4\lambda_r^2}{\pi} \mathbb{E}\left[\left(G_{ij}^k\right)^2\right] \nonumber \\
    &=\frac{4}{\pi} \lambda_r^2 (M-1)N \to 0,
\end{align*}
due to the scaling $\lambda_r = o\left(M^{-\frac{1}{2}}N^{-\frac{1}{2}}\right)$.

Finally, substituting $\kappa = 1$ and using $\vert \erf(x) \vert \leq 1$ and $\erf^2(x) \leq \frac{4}{\pi}x^2$ again from Lemma~\ref{lemma:erf_ineq}, one has that
\begin{align*}
    &\textcolor{teal}{\mathbb{E}\!\left[\erf^2\!\left(\lambda_c \erf\!\left(\lambda_r \kappa \ell + \lambda_{r} \gamma  G_{11}^k\right)\!\left(\kappa \ell + \gamma G_{11}^k\right) + \lambda_c \sum_{j=2}^{N} \erf\!\left(\lambda_r G_{1j}^k\right) G_{1j}^k\right)\right.} \nonumber \\
    & \phantom{\mathbb{E}}\textcolor{teal}{\cdot \left.\left(\gamma^2 \erf^2\!\left(\lambda_r \kappa \ell + \lambda_{r} \gamma G_{11}^k\right) + \sum_{j=2}^{N} \erf^2\!\left(\lambda_r G_{1j}^k\right)\right)\!\right]} \nonumber \\
    &\leq \gamma^2 \mathbb{E}\!\left[\erf^2\!\left(\lambda_c \erf\!\left(\lambda_r \ell + \lambda_{r} \gamma  G_{11}^k\right)\!\left(\ell + \gamma G_{11}^k\right) + \lambda_c \sum_{j=2}^{N} \erf\!\left(\lambda_r G_{1j}^k\right) G_{1j}^k\right)\right. \nonumber \\
    &\phantom{\leq \gamma^2\mathbb{E}}\erf^2\!\left(\lambda_r \ell + \lambda_{r} \gamma G_{11}^k\right) \Bigg] + \sum_{j=2}^N\mathbb{E}\left[\erf^2\left(\lambda_r G_{1j}^k\right)\right] \nonumber \\
    & \leq \gamma^2 \mathbb{E}\left[\erf^2\left(\frac{\lambda_c}{\lambda_r} \erf(G)G + \lambda_c \sum_{j=2}^N\erf\left(\lambda_rG_{1j}^k\right)G_{1j}^k\right) \erf^2(G)\right] + \underbrace{\frac{4}{\pi} \lambda_r^2 (N-1)}_{\to 0},
\end{align*}
where $G \sim \mathcal{N}\left(\lambda_r \ell, \lambda_r^2 \gamma^2\right)$.

Aggregating everything and considering the oracle parameters $k^*, v^*$, one now has that
\begin{align*}
    \mathcal{R}(k^*, v^*) &\leq \epsilon^2 + \gamma^2 - 2 \gamma^2 \mathbb{E}\left[\erf\left(\frac{\lambda_c}{\lambda_r} \erf(G) G\right) \erf(G)\right] + 2\gamma^2 \Phi\left(\frac{-\ell}{\gamma}\right)\nonumber \\
    & \phantom{\leq}+ \gamma^2 \mathbb{E}\left[\erf^2\left(\frac{\lambda_c}{\lambda_r} \erf(G)G + \lambda_c \sum_{j=2}^N\underbrace{\erf\left(\lambda_rG_{1j}^k\right)G_{1j}^k}_{\geq 0}\right) \erf^2(G)\right] + o(1),
\end{align*}
where $G \sim \mathcal{N}(\lambda_r \ell, \lambda_r^2 \gamma^2)$.

To conclude the proof of Proposition~\ref{cor:cor02}, it is enough to prove 
\[
\mathbb{E}\left[\erf\left(\frac{\lambda_c}{\lambda_r} \erf(G) G\right) \erf(G)\right] \to 1,
\]
as $\ell \to \infty$ implies $\Phi\left(\frac{-\ell}{\gamma}\right) \to 0$.
Since $x \mapsto \erf\left(\frac{\lambda_c}{\lambda_r} \erf(x) x\right) \erf(x)$ is increasing on $(0, \infty)$ and less than 1 in absolute value, it is enough to consider
\begin{align*}
    \mathbb{E}\left[\erf\left(\frac{\lambda_c}{\lambda_r} \erf(G) G\right) \erf(G)\right] &\geq \erf\left(\frac{\lambda_c}{\lambda_r} \erf\left(\frac{\lambda_r \ell}{2}\right) \frac{\lambda_r \ell}{2}\right) \erf\left(\frac{\lambda_r \ell}{2}\right) \mathbb{P}\left(G \geq \frac{\lambda_r \ell}{2} \right) \nonumber \\
    & \phantom{\leq } - \mathbb{P}(G < 0) \nonumber \\
    & = \erf\left(\frac{\lambda_c \ell}{2} \erf\left(\frac{\lambda_r \ell}{2} \right)\right) \erf\left(\frac{\lambda_r \ell}{2}\right) \left(1 - \Phi\left(-\frac{\ell}{2\gamma}\right)\right) \nonumber\\
    & \phantom{\leq} - \Phi\left(- \frac{\ell}{\gamma}\right).
\end{align*}
The conclusion now follows since $\lambda_r \ell \to \infty$, $\ell \to \infty$, $\lambda_c \ell \to \infty$.

\subsection{Proof of Proposition~\ref{prop:prop03}} \label{sec:app_prop3}
The proof aims to make explicit the loss of information from $\mathbb{X}$ to $(S_1, S_2, \dots, S_M)$. To this end, one has
\begin{align*}
    \mathbb{E}\left[\left(Y - f(S_1, S_2, \dots, S_M)\right)^2\right] &\geq \mathbb{E}\left[\text{Var}\left(Y | S_1, S_2, \dots, S_M\right)\right] \\
    &\geq \mathbb{E}\left[\text{Var}\left(Y | S_1, S_2, \dots, S_M, I_0, J_0\right)\right].
\end{align*}
Now, consider $i_0 \in [M]$ and $j_0 \in [N]$. Conditional on $(I_0, J_0) = (i_0, j_0)$, $S_i$ is independent of $(Y, S_{i_0})$, for $i \neq i_0$. Then, one has
\begin{equation}
    \text{Var}\left(Y | S_1, S_2, \dots, S_N, I_0 = i_0, J_0 = j_0\right) = \text{Var}\left(Y | S_{i_0}, I_0 = i_0, J_0 = j_0\right).
\end{equation}
For multivariate Gaussian random variables $(U, V)$, one has
\begin{equation} \label{eq:cond_var_gaussian}
    \text{Var}(U|V) = \Sigma_{UU} - \Sigma_{UV} \Sigma_{VV}^{-1}\Sigma_{VU},
\end{equation}
where $\Sigma_{AB}$ stands for $\text{Cov}(A, B)$, for random vectors $A$ and $B$.

Conditioned on $I_0 = i_0$ and $J_0 = j_0$, one has:
\begin{itemize}
    \item $X_{i_0j_0} = \ell k^* + \gamma Z_{j_0}$ and $X_{i_0 j} = Z_j$ for $j \in [N] \setminus \{j_0\}$ with standard normal random variables $Z_1, \dots, Z_N$;
    \item $S_{i_0} = \beta_{j_0}\ell k^* + \gamma \beta_{j_0} Z_{j_0} + \sum_{j \neq j_0} \beta_j Z_j$;
    \item $Y = \gamma Z_{j_0}^Tv^* + \xi$.
\end{itemize} 
One then has the following:
\begin{itemize}
    \item $\text{Var}(Y) = \gamma^2 + \epsilon^2$;
    \item $\text{Cov}(S_{i_0}, Y) = \gamma^2 \beta_{j_0} v^*$
    \item $\text{Cov}(S_{i_0}, S_{i_0}) = \left(\gamma ^2 \beta_{j_0}^2 + \sum_{j \neq j_0} \beta_j^2 \right) I$
\end{itemize}
Introducing in \eqref{eq:cond_var_gaussian}, one obtains
\begin{equation}
    \text{Var}\left(Y | S_{i_0}, I_0 = i_0, J_0 = j_0\right) = \gamma^2 + \epsilon^2 - \frac{\gamma^4 \beta_{j_0}^2}{\gamma ^2 \beta_{j_0}^2 + \sum_{j \neq j_0} \beta_j^2} \|v^*\|^2 = \epsilon^2 + \gamma^2\left(1 - \frac{\gamma^2 \beta_{j_0}^2}{\gamma ^2 \beta_{j_0}^2 + \sum_{j \neq j_0} \beta_j^2}\right).
\end{equation}
Finally,
\begin{align*}
    \mathbb{E}\left[\text{Var}\left(Y | S_1, S_2, \dots, S_M, I_0, J_0\right)\right] &= \sum_{j_0 = 1}^N \sum_{i_0 = 1}^M \mathbb{P}(I_0 = i_0 | J_0 = j_0) p_{j_0} \text{Var}\left(Y | S_{i_0}, I_0 = i_0, J_0 = j_0\right) \\
    & = \sum_{j_0 = 1}^N p_{j_0} \left(\epsilon^2 + \gamma^2\left(1 - \frac{\gamma^2 \beta_{j_0}^2}{\gamma ^2 \beta_{j_0}^2 + \sum_{j \neq j_0} \beta_j^2}\right)\right) \\
    & = \epsilon^2 + \gamma^2\left(1 - \sum_{j_0=1}^N p_{j_0} \frac{\gamma^2 \beta_{j_0}^2}{\gamma ^2 \beta_{j_0}^2 + \sum_{j \neq j_0} \beta_j^2}\right) \\
    & \geq \epsilon^2,
\end{align*}
with equality in the last step only if there exists $j_0$ such that $p_{j_0} = 1$ and $\beta_j = 0$ for $j \neq j_0$.

\section{Proof of Theorem~\ref{thm:pgd_convergence}} \label{sec:app_pgd}

In this section, we establish the global convergence of the Projected Gradient Descent (PGD) dynamics to the true oracle parameters. We follow the proof of Theorem 5 in \citet{marion} as closely as possible while adapting it to our model $T_{\lambda_r, \lambda_c}^{(k,v)}$. The only result which requires a completely different and much more involved proof is that of Proposition~\ref{prop:fixed_points}. Throughout these proofs, $C > 0$ denotes a generic constant that may vary from line to line but is strictly independent of the step size $\alpha$ and the iteration step $t$. Moreover, we shall drop the subscript ${\lambda_r, \lambda_c}$ and assume an implicit dependence on the inverse temperatures.

The theoretical guarantees of our optimization trajectory are established through the following five primary results. We first establish the invariance of the manifold to PGD steps.

\begin{lemma} \label{lemma:manif_inv}
    The manifold $\mathcal{M}$ is invariant under the PGD dynamics (Algorithm~\ref{alg:pgd}), in the sense that if $(k_t, v_t) \in \mathcal{M}$, then $(k_{t+1}, v_{t+1}) \in \mathcal{M}$.
\end{lemma}

By restricting our analysis to the manifold $\mathcal{M}$, we can project the high-dimensional dynamics down to the scalar alignment variables $\kappa_t = k_t^T k^*$ and $\nu_t = v_t^T v^*$. We then map the PGD updates to an autonomous dynamical system:

\begin{lemma}\label{lemma:macro_mapping}
    When initialized on the manifold $\mathcal{M}$, the PGD iterations (Algorithm~\ref{alg:pgd}) can be reformulated as the autonomous discrete dynamical system $(\kappa_{t+1}, \nu_{t+1}) = g(\kappa_t, \nu_t)$, where $g \colon [-1, 1]^2 \to [-1, 1]^2$ is given by
    \begin{equation} \label{eq:mapping_g}
        g(\kappa, \nu) = \left( 
        \frac{\kappa - \alpha (\partial_\kappa \mathcal{R}^<(\kappa, \nu))(1 - \kappa^2)}{\sqrt{1 + \alpha^2 (\partial_\kappa \mathcal{R}^<(\kappa, \nu))^2 (1 - \kappa^2)}}, 
        \frac{\nu - \alpha (\partial_\nu \mathcal{R}^<(\kappa, \nu))(1 - \nu^2)}{\sqrt{1 + \alpha^2 (\partial_\nu \mathcal{R}^<(\kappa, \nu))^2 (1 - \nu^2)}} 
        \right).
    \end{equation}
\end{lemma}

With the dynamics formalized, we establish that the risk strictly decreases, forcing convergence to the stationary points:

\begin{proposition} \label{prop:risk_decrease}
    For a sufficiently small step size $\alpha$, the risk $\mathcal{R}^<$ is strictly decreasing along the PGD iterates unless it is at a fixed point. Furthermore, if $(\kappa, \nu)$ is an accumulation point of the sequence $(\kappa_t, \nu_t)_{t \ge 0}$, it must satisfy:
    \begin{equation} \label{eq:fixed_points}
        (1 - \kappa^2)\partial_\kappa \mathcal{R}^<(\kappa, \nu) = 0 \quad \text{and} \quad (1 - \nu^2)\partial_\nu \mathcal{R}^<(\kappa, \nu) = 0.
    \end{equation}
\end{proposition}

We then rigorously classify the geometry of these accumulation points, isolating the oracle states as the unique global minima:

\begin{proposition} \label{prop:fixed_points}
    For $\lambda_c$ and $\lambda_r$ small enough, the fixed points of the dynamics, given by \eqref{eq:fixed_points}, on $[-1, 1]^2$ are $(\pm 1, \pm 1)$ and $(0, 0)$. Specifically:
    \begin{itemize}
        \item $(\pm 1, \mp 1)$ are global maxima.
        \item $(\pm 1, \pm 1)$ are global minima.
        \item $(0, 0)$ is a saddle point.
    \end{itemize}
\end{proposition}

Finally, we characterize the local geometry around the saddle point to guarantee that random initializations escape it:

\begin{proposition} \label{prop:diffeomorphism}
    For a sufficiently small step size $\alpha$, the mapping $g$ is a local diffeomorphism around $(0,0)$, whose Jacobian matrix has one eigenvalue in $(0, 1)$ and one in $(1, \infty)$. Furthermore, $g$ is injective on $[-1, 1]^2$.
\end{proposition}

The global convergence to the oracle states then follows directly from the exclusion of the saddle point and maxima. The proofs for all stated results are provided in the subsections below.

\subsection{Proof of Lemma~\ref{lemma:manif_inv}} \label{sec:proof_manif_inv}

To prove that the manifold $\mathcal{M}$ is invariant under the Projected Gradient Descent (PGD) dynamics, we must first establish that the gradient of the risk does not "push" the parameters outside of $\mathcal{M}$. We do this by proving an auxiliary symmetry result regarding the cross-term derivatives.

\begin{lemma} \label{lemma:lemma02}
    At any point $(\kappa, \nu, \theta, \eta, \rho)$ on the manifold $\mathcal{M}$ (i.e., where $\theta = \eta = \rho = 0$), the partial derivatives of the risk with respect to the cross-terms vanish:
    \begin{equation*}
        \partial_\theta \mathcal{R}^< = \partial_\eta \mathcal{R}^< = \partial_\rho \mathcal{R}^< = 0.
    \end{equation*}
\end{lemma}

The proof follows by recalling that the risk can be reduced to a function $\mathcal{R}^<$ depending only on the covariance structure of the Gaussian projections $G_{ij}^k = X_{ij}^T k$, $G_{ij}^v = X_{ij}^T v$, and $G_{ij}^{v^*} = X_{ij}^T v^*$. The relevant covariance matrix is:
\begin{equation*}
    \begin{pmatrix}
    G_{ij}^k \\ G_{ij}^v \\ G_{ij}^{v^*}
    \end{pmatrix}
    \sim
    \mathcal{N}\left(
    0,
    \begin{pmatrix}
    1 & \rho & \eta \\
    \rho & 1 & \nu \\
    \eta & \nu & 1
    \end{pmatrix}
    \right)
\end{equation*}
and the alignment between $v$ and $k^*$ is given by $\theta = v^T k^*$. 

It is clear from Equation~\ref{eq:app_risk} that the risk is invariant to the sign inversion of $(v, v^*)$ to $(-v, -v^*)$. This amounts to
\[
\mathcal{R}^<(\kappa, \nu, \theta, \eta, \rho) = \mathcal{R}^<(\kappa, \nu, -\theta, -\eta, -\rho).
\]

Taking the partial derivative with respect to $\theta$ now yields the following result:
\begin{equation*}
    \partial_\theta \mathcal{R}^<(\kappa, \nu, \theta, \eta, \rho) = - \partial_\theta \mathcal{R}^<(\kappa, \nu, -\theta, -\eta, -\rho).
\end{equation*}
Evaluating this at the manifold $\mathcal{M}$ where $\theta = \eta = \rho = 0$, we immediately obtain $\partial_\theta \mathcal{R}^<(\kappa, \nu, 0, 0, 0) = 0$. Identical arguments hold for $\partial_\eta \mathcal{R}^<$ and $\partial_\rho \mathcal{R}^<$, concluding the proof of Lemma~\ref{lemma:lemma02}.

We now return to the proof of Lemma~\ref{lemma:manif_inv}. By the chain rule, the Euclidean gradients of the risk with respect to $k$ and $v$ can be expanded in the basis of the parameter vectors:
\begin{align*}
    \nabla_k \mathcal{R} &= (\partial_\kappa \mathcal{R}^<) k^* + (\partial_\eta \mathcal{R}^<) v^* + (\partial_\rho \mathcal{R}^<) v \\
    \nabla_v \mathcal{R} &= (\partial_\nu \mathcal{R}^<) v^* + (\partial_\theta \mathcal{R}^<) k^* + (\partial_\rho \mathcal{R}^<) k
\end{align*}
Applying Lemma~\ref{lemma:lemma02}, on the manifold $\mathcal{M}$ (where $\theta=\eta=\rho=0$), the gradients simplify strictly to:
\begin{align} \label{eq:simplified_gradients}
    \nabla_k \mathcal{R}(k_t, v_t) &= \partial_\kappa \mathcal{R}^<(\kappa_t, \nu_t) k^* \nonumber \\
    \nabla_v \mathcal{R}(k_t, v_t) &= \partial_\nu \mathcal{R}^<(\kappa_t, \nu_t) v^*
\end{align}

Recall the PGD updates for $(k_{t+1}, v_{t+1})$:
\begin{align*}
    k_{t+1} &= \frac{k_t - \alpha(I_d - k_t k_t^T)\nabla_k \mathcal{R}}{\|k_t - \alpha(I_d - k_t k_t^T)\nabla_k \mathcal{R}\|_2} \coloneqq \frac{\tilde{k}_{t+1}}{c_k} \\
    v_{t+1} &= \frac{v_t - \alpha(I_d - v_t v_t^T)\nabla_v \mathcal{R}}{\|v_t - \alpha(I_d - v_t v_t^T)\nabla_v \mathcal{R}\|_2} \coloneqq \frac{\tilde{v}_{t+1}}{c_v}
\end{align*}
where $c_k, c_v > 0$ are the respective normalization constants. To prove $(k_{t+1}, v_{t+1}) \in \mathcal{M}$, we must verify the three orthogonality conditions: $(v^*)^T k_{t+1} = 0$, $(k^*)^T v_{t+1} = 0$, and $v_{t+1}^T k_{t+1} = 0$.

\paragraph{Condition 1: $(v^*)^T k_{t+1} = 0$.}
Using the simplified gradient \eqref{eq:simplified_gradients}, we project $\tilde{k}_{t+1}$ onto $v^*$:
\begin{equation*}
    (v^*)^T \tilde{k}_{t+1} = (v^*)^T k_t - \alpha \partial_\kappa \mathcal{R}^< \left( (v^*)^T k^* - (v^*)^T k_t (k_t^T k^*) \right).
\end{equation*}
Because $(k_t, v_t) \in \mathcal{M}$, we have $(v^*)^T k_t = 0$. By Assumption 1, the oracle parameters are orthogonal, so $(v^*)^T k^* = 0$. Thus:
\begin{equation*}
    (v^*)^T \tilde{k}_{t+1} = 0 - \alpha \partial_\kappa \mathcal{R}^< (0 - 0) = 0 \implies (v^*)^T k_{t+1} = 0.
\end{equation*}

\paragraph{Condition 2: $(k^*)^T v_{t+1} = 0$.}
By an exactly symmetric argument using $\tilde{v}_{t+1}$ and the fact that $(k^*)^T v_t = 0$, we obtain $(k^*)^T v_{t+1} = 0$.

\paragraph{Condition 3: $v_{t+1}^T k_{t+1} = 0$.}
We compute the inner product of the unnormalized updates $\tilde{v}_{t+1}^T \tilde{k}_{t+1}$:
\begin{align*}
    \tilde{v}_{t+1}^T \tilde{k}_{t+1} &= \left( v_t - \alpha \partial_\nu \mathcal{R}^< (v^* - \nu_t v_t) \right)^T \left( k_t - \alpha \partial_\kappa \mathcal{R}^< (k^* - \kappa_t k_t) \right) \nonumber \\
    &= v_t^T k_t - \alpha \partial_\kappa \mathcal{R}^< \left( v_t^T k^* - \kappa_t v_t^T k_t \right) - \alpha \partial_\nu \mathcal{R}^< \left( (v^*)^T k_t - \nu_t v_t^T k_t \right) \nonumber \\
    &\quad + \alpha^2 \partial_\kappa \mathcal{R}^< \partial_\nu \mathcal{R}^< \left( (v^* - \nu_t v_t)^T (k^* - \kappa_t k_t) \right).
\end{align*}
Since $(k_t, v_t) \in \mathcal{M}$, we already know $v_t^T k_t = 0$, $v_t^T k^* = 0$, and $(v^*)^T k_t = 0$. Thus, all the linear terms vanish. For the quadratic term, we expand:
\begin{equation*}
    (v^* - \nu_t v_t)^T (k^* - \kappa_t k_t) = (v^*)^T k^* - \kappa_t (v^*)^T k_t - \nu_t v_t^T k^* + \nu_t \kappa_t v_t^T k_t.
\end{equation*}
Substituting the orthogonalities $(v^*)^T k^* = 0$, $(v^*)^T k_t = 0$, $v_t^T k^* = 0$, and $v_t^T k_t = 0$ yields the desired orthogonality. 

Therefore, $\tilde{v}_{t+1}^T \tilde{k}_{t+1} = 0$, meaning $v_{t+1}^T k_{t+1} = 0$. Having verified all three conditions, we conclude that $(k_{t+1}, v_{t+1}) \in \mathcal{M}$. q.e.d.

\subsection{Proof of Lemma~\ref{lemma:macro_mapping}}

Using the simplified gradient \eqref{eq:simplified_gradients} valid on the manifold $\mathcal{M}$, we compute the unnormalized update for $k$:
\begin{equation*}
    \tilde{k}_{t+1} = k_t - \alpha (I_d - k_t k_t^T) (\partial_\kappa \mathcal{R}^<) k^* = k_t - \alpha (\partial_\kappa \mathcal{R}^<) (k^* - \kappa_t k_t).
\end{equation*}
Projecting this unnormalized update onto $k^*$ yields the numerator for $\kappa_{t+1}$:
\begin{equation*}
    (k^*)^T \tilde{k}_{t+1} = (k^*)^T k_t - \alpha (\partial_\kappa \mathcal{R}^<) ( (k^*)^T k^* - \kappa_t (k^*)^T k_t ) = \kappa_t - \alpha (\partial_\kappa \mathcal{R}^<) (1 - \kappa_t^2).
\end{equation*}
To find the normalization constant $c_k = \|\tilde{k}_{t+1}\|_2$, we apply the Pythagorean theorem. Since $(k^* - \kappa_t k_t)$ is the projection of $k^*$ onto the orthogonal complement of $k_t$, it is strictly orthogonal to $k_t$. Thus:
\begin{align*}
    \|\tilde{k}_{t+1}\|_2^2 &= \|k_t\|_2^2 + \alpha^2 (\partial_\kappa \mathcal{R}^<)^2 \|k^* - \kappa_t k_t\|_2^2 \nonumber \\
    &= 1 + \alpha^2 (\partial_\kappa \mathcal{R}^<)^2 (1 - \kappa_t^2).
\end{align*}
Dividing the numerator by $\|\tilde{k}_{t+1}\|_2$ directly yields $g_1(\kappa_t, \nu_t)$. Analogously, one obtains the result for $g_2(\kappa_t, \nu_t)$.

\subsection{Proof of Proposition~\ref{prop:risk_decrease}}

The theoretical risk $\mathcal{R}^<$ is composed of smooth error functions evaluated over Gaussian expectations, ensuring it is $\mathcal{C}^\infty$ on the compact square $[-1, 1]^2$. Consequently, it is $\Lambda$-smooth, meaning we can bound its variation:
\begin{align} \label{eq:smoothness}
    \mathcal{R}^<(\kappa_{t+1}, \nu_{t+1}) - \mathcal{R}^<(\kappa_t, \nu_t) &\leq (\partial_\kappa \mathcal{R}^<) (\kappa_{t+1} - \kappa_t) + (\partial_\nu \mathcal{R}^<) (\nu_{t+1} - \nu_t) \nonumber\\
    &\phantom{=}+ \frac{\Lambda}{2} \left[ (\kappa_{t+1} - \kappa_t)^2 + (\nu_{t+1} - \nu_t)^2 \right].
\end{align}
Using the discrete mapping $g$ from Lemma~\ref{lemma:macro_mapping}, we rewrite the step difference for $\kappa$:
\begin{align} \label{eq:app3_00}
    \kappa_{t+1} - \kappa_t &= \kappa_t \left( \frac{1}{\sqrt{1 + \alpha^2 (\partial_\kappa \mathcal{R}^<)^2(1-\kappa_t^2)}} - 1 \right) - \frac{\alpha (\partial_\kappa \mathcal{R}^<) (1 - \kappa_t^2)}{\sqrt{1 + \alpha^2 (\partial_\kappa \mathcal{R}^<)^2(1-\kappa_t^2)}} \nonumber \\
    &= - \alpha (\partial_\kappa \mathcal{R}^<) (1 - \kappa_t^2) \nonumber \\
    &\phantom{=} + \left(\frac{1}{\sqrt{1 + \alpha^2 (\partial_\kappa \mathcal{R}^<)^2(1-\kappa_t^2)}} - 1 \right)(\kappa_t - \alpha (\partial_\kappa \mathcal{R}^<) (1 - \kappa_t^2)).
\end{align}
As a consequence,
\begin{align} \label{eq:app_bound}
\vert \kappa_{t+1} - \kappa_t + \alpha (\partial_\kappa \mathcal{R}^<) (1 - \kappa_t^2) \vert & \leq \left\vert\frac{1}{\sqrt{1 + \alpha^2 (\partial_\kappa \mathcal{R}^<)^2(1-\kappa_t^2)}} - 1\right\vert \left\vert \kappa_t - \alpha (\partial_\kappa \mathcal{R}^<) (1 - \kappa_t^2) \right\vert \nonumber \\
&\leq \alpha^2 (\partial_\kappa \mathcal{R}^<)^2(1-\kappa_t^2) \left\vert \kappa_t - \alpha (\partial_\kappa \mathcal{R}^<) (1 - \kappa_t^2) \right\vert \nonumber \\
& \leq C \alpha^2 (\partial_\kappa \mathcal{R}^<)^2(1-\kappa_t^2),
\end{align}
where the second inequality holds by Lemma~\ref{lemma:inv_sqrt} and the last bound holds since the function $(\kappa, \nu) \mapsto \left\vert \kappa - \alpha (\partial_\kappa \mathcal{R}^<(\kappa, \nu)) (1 - \kappa^2) \right\vert$ is uniformly bounded for all $\alpha \leq 1$. This bound has two implications. First,
\begin{align}\label{eq:app3_01}
    (\partial_\kappa \mathcal{R}^<)(\kappa_{t+1}-\kappa_t) + \alpha(\partial_\kappa \mathcal{R}^<)^2(1-\kappa_t^2) &= (\partial_\kappa \mathcal{R}^<)\left((\kappa_{t+1}-\kappa_t) + \alpha (\partial_\kappa \mathcal{R}^<)(1-\kappa_t^2)\right) \nonumber\\
    & \leq \vert (\partial_\kappa \mathcal{R}^<) \vert \left\vert (\kappa_{t+1}-\kappa_t) + \alpha (\partial_\kappa \mathcal{R}^<)(1-\kappa_t^2) \right\vert \nonumber\\
    & \leq C\alpha^2 (\partial_\kappa \mathcal{R}^<)^2 (1-\kappa_t^2),
\end{align}
where we use the fact that $\vert \partial_\kappa \mathcal{R}^< \vert$ is bounded and the bound \eqref{eq:app_bound}. Second, since the square function is Lipschitz on compact sets, we have
\begin{equation*}
    \left\vert(\kappa_{t+1} - \kappa_t)^2 - \left(\alpha (\partial_\kappa \mathcal{R}^<) (1 - \kappa_t^2)\right)^2\right\vert \leq C \alpha^2(\partial_\kappa \mathcal{R}^<)^2 (1-\kappa_t^2).
\end{equation*}
Thus,
\begin{align}\label{eq:app3_02}
    (\kappa_{t+1} - \kappa_t)^2 &\leq \alpha^2(\partial_\kappa \mathcal{R}^<)^2 (1-\kappa_t^2)^2 + C \alpha^2(\partial_\kappa \mathcal{R}^<)^2 (1-\kappa_t^2) \nonumber\\
    &\leq C \alpha^2(\partial_\kappa \mathcal{R}^<)^2 (1-\kappa_t^2).
\end{align}
We obtain analogous bounds for $\nu$, namely
\begin{equation}\label{eq:app3_03}
    (\partial_\nu \mathcal{R}^<)(\nu_{t+1}-\nu_t) + \alpha(\partial_\nu \mathcal{R}^<)^2(1-\nu_t^2) \leq C \alpha^2(\partial_\nu \mathcal{R}^<)^2 (1-\nu_t^2)
\end{equation}
and
\begin{equation}\label{eq:app3_04}
    (\nu_{t+1} - \nu_t)^2 \leq C \alpha^2(\partial_\nu \mathcal{R}^<)^2 (1-\nu_t^2).
\end{equation}

Plugging the bounds \eqref{eq:app3_01}, \eqref{eq:app3_02}, \eqref{eq:app3_03}, \eqref{eq:app3_04} into Equation~\eqref{eq:smoothness}, we obtain the desired inequality
\begin{align*}
    \mathcal{R}^<(\kappa_{t+1}, \nu_{t+1}) - \mathcal{R}^<(\kappa_t, \nu_t) &\leq - \alpha(\partial_\kappa \mathcal{R}^<)^2 (1-\kappa_t^2) - \alpha(\partial_\nu \mathcal{R}^<)^2 (1-\nu_t^2) \nonumber\\
    &\phantom{\leq} + C \alpha^2(\partial_\kappa \mathcal{R}^<)^2 (1-\kappa_t^2) + C \alpha^2(\partial_\nu \mathcal{R}^<)^2 (1-\nu_t^2).
\end{align*}
By choosing the step size $\alpha \leq \frac{1}{2C}$, this ensures that
\begin{equation*}
    \mathcal{R}^<(\kappa_{t+1}, \nu_{t+1}) - \mathcal{R}^<(\kappa_t, \nu_t) \leq - \frac{\alpha}{2}(\partial_\kappa \mathcal{R}^<)^2 (1-\kappa_t^2) - \frac{\alpha}{2}(\partial_\nu \mathcal{R}^<)^2 (1-\nu_t^2).
\end{equation*}

This shows that the risk is decreasing along the PGD iterates. Next, introducing $\mathcal{R}_{\min}^< = \min_{(\kappa, \nu) \in [-1,1]^2}\mathcal{R}^<(\kappa, \nu)$ and using a telescopic sum, we have, for all $T \geq 0$,
\begin{align*}
    \mathcal{R}^<(\kappa_0, \nu_0) - \mathcal{R}_{\min}^< &\geq \mathcal{R}^<(\kappa_0, \nu_0) - \mathcal{R}^<(\kappa_T, \nu_T) \\
    &\geq \frac{\alpha}{2} \sum_{t=0}^{T-1}(\partial_\kappa \mathcal{R}^<)^2 (1-\kappa_t^2) +(\partial_\nu \mathcal{R}^<)^2 (1-\nu_t^2).
\end{align*}
Since the left-hand side is finite, and the terms of the sum are nonnegative, we conclude that the series converges as $T \to \infty$. In particular, the generic term $(\partial_\kappa \mathcal{R}^<)^2 (1-\kappa_t^2) +(\partial_\nu \mathcal{R}^<)^2 (1-\nu_t^2)$ of the series converges to 0 as $t \to \infty$. Therefore, the accumulation points $(\kappa_\infty, \nu_\infty)$ satisfy
\begin{equation*}
\begin{cases}
    \partial_\kappa \mathcal{R}^<(\kappa_\infty, \nu_\infty) = 0 \quad \text{or} \quad \kappa^2_\infty = 1 \\
    \partial_\nu \mathcal{R}^<(\kappa_\infty, \nu_\infty) = 0 \quad \text{or} \quad \nu^2_\infty = 1.
\end{cases}
\end{equation*}
Inspecting identity \eqref{eq:app3_00}, we observe that the convergence of the general term also implies $\kappa_{t+1} - \kappa_t \to 0$. Similarly, $\nu_{t+1} - \nu_t \to 0$. q.e.d.

\subsection{Proof of Proposition~\ref{prop:fixed_points}}

First we recall the formula of the risk
\begin{align}\label{eq:risk_kappa_nu}
    \mathcal{R}^<&(\kappa, \nu) = \nonumber\\
    &= -2\gamma^2\nu\mathbb{E}\!\left[ \erf\!\left(\lambda_c \erf\!\left(\lambda_r \kappa \ell + \lambda_{r} \gamma G_{11}^k\right) (\kappa \ell + \gamma G_{11}^k) + \lambda_c \sum_{j=2}^{N} \erf\!\left(\lambda_r G_{1j}^k\right) G_{1j}^k\right)\right. \cdot\ \nonumber \\
    &\phantom{=\gamma \nu \mathbb{E}}\quad\cdot \erf\!\left(\lambda_r \kappa \ell + \lambda_{r} \gamma G_{11}^k\right)\Bigg] \nonumber\\
    &\phantom{=} + \gamma^2\mathbb{E}\!\left[\erf^2\!\left(\lambda_c \erf\!\left(\lambda_r \kappa \ell + \lambda_{r} \gamma  G_{11}^k\right)\!\left(\kappa \ell + \gamma G_{11}^k\right) + \lambda_c \sum_{j=2}^{N} \erf\!\left(\lambda_r G_{1j}^k\right) G_{1j}^k\right)\right. \nonumber \\
    & \phantom{= + \mathbb{E}} \quad\cdot \erf^2\!\left(\lambda_r \kappa \ell + \lambda_{r} \gamma G_{11}^k\right)\Bigg] \nonumber\\
    &\phantom{=}+ \mathbb{E}\!\left[\erf^2\!\left(\lambda_c \erf\!\left(\lambda_r \kappa \ell + \lambda_{r} \gamma  G_{11}^k\right)\!\left(\kappa \ell + \gamma G_{11}^k\right) + \lambda_c \sum_{j=2}^{N} \erf\!\left(\lambda_r G_{1j}^k\right) G_{1j}^k\right)\right. \nonumber \\
    & \phantom{= + \mathbb{E}} \quad\cdot \sum_{j=2}^{N} \erf^2\!\left(\lambda_r G_{1j}^k\right)\!\Bigg] + C,
\end{align}
where $C$ does not depend on $\kappa$ and $\nu$.
\begin{align} \label{eq:app_def_V}
    V(x_1, x_2, \dots, x_N) \coloneqq \; &\erf\!\left(\frac{\lambda_c}{\lambda_r} \sum_{j=1}^{N} \erf\!\left(x_j\right) x_j\right)
\end{align}
and by
\begin{align} \label{eq:app_def_U}
    U(x_1, x_2, \dots, x_N) \coloneqq \; V(x_1, \dots, x_N) \erf(x_1)
\end{align}
Furthermore, we shall make a slight abuse of notation and denote by
\begin{align*}
    V(\kappa) : &= \; V\left(\lambda_r \kappa \ell + \lambda_r \gamma G_{11}^k, \lambda_r G_{12}^k, \dots, \lambda_r G_{1N}^k\right) \\
    & = \erf\!\left(\lambda_c \erf\!\left(\lambda_r \kappa \ell + \lambda_r \gamma G_{11}^k\right) \left(\kappa \ell + \gamma G_{11}^k\right) + \lambda_c \sum_{j=2}^{N} \erf\!\left(\lambda_r G_{1j}^k\right) G_{1j}^k\right)
\end{align*}
and by
\begin{equation*}
    U(\kappa) \coloneqq \; V(\kappa) \cdot \erf(\lambda_r \kappa \ell + \lambda_r \gamma G_{11}^k).
\end{equation*}
One can easily observe that $V(\kappa)$ is symmetric in all $G_{1j}^k$ and in $\kappa$ since $x \erf(x)$ is symmetric around 0 and that $\vert V(\kappa) \vert \leq 1$. Moreover, this implies that $U(\kappa)$ is strictly increasing in $\kappa$ and that $\vert U(\kappa)\vert \leq 1$. Finally, the risk can be rewritten as
\begin{equation*}
    \mathcal{R}^<(\kappa, \nu) = -2 \gamma^2\nu\mathbb{E}\left[ U(\kappa)\right] + \gamma^2\mathbb{E}\left[U(\kappa)^2\right] +  \mathbb{E}\!\left[V(\kappa)^2 \sum_{j=2}^{N} \erf^2\!\left(\lambda_r G_{1j}^k\right)\!\right] + C.
\end{equation*}

Now, $(\kappa, \nu)$ is a stationary point of the PGD dynamics if
\begin{equation*}
\begin{cases}
    \partial_\kappa \mathcal{R}^<(\kappa, \nu) = 0 \quad \text{or} \quad \kappa^2 = 1 \\
    \partial_\nu \mathcal{R}^<(\kappa, \nu) = 0 \quad \text{or} \quad \nu^2 = 1.
\end{cases}
\end{equation*}
The partial derivatives of the risk are
\begin{equation} \label{eq:app_partials}
\begin{cases}
    \partial_\kappa \mathcal{R}^<(\kappa, \nu) = 2\gamma^2\mathbb{E}\left[\partial_\kappa U(\kappa)(U (\kappa)-\nu)\right]  + 2\mathbb{E}\left[\partial_\kappa V(\kappa) V(\kappa) \sum_{j=2}^{N} \erf^2\!\left(\lambda_r G_{1j}^k\right)\right]\\
    \partial_\nu \mathcal{R}^<(\kappa, \nu) = -2\gamma^2\mathbb{E}\left[ U(\kappa)\right].
\end{cases}
\end{equation}
We split our analysis in different cases.

\paragraph{Case 1. $\partial_\nu\mathcal{R}^<(\kappa, \nu) = 0$}
This implies that $\mathbb{E}[U(\kappa)] = 0$. It is clear that $\mathbb{E}[U(0)] = 0$ and due to the fact that $U$ is strictly increasing, $\kappa = 0$ is also the unique such point.

Now,
\begin{equation} \label{eq:app_symm_k}
    S(\kappa) \coloneqq \gamma^2\mathbb{E}\left[U(\kappa)^2\right] +  \mathbb{E}\!\left[V(\kappa)^2 \sum_{j=2}^{N} \erf^2\!\left(\lambda_r G_{1j}^k\right)\!\right]
\end{equation}
is jointly symmetric in $\left(\kappa, \left(G_{1j}^k\right)_{j \in [N]} \right)$. Taking the expectation over $\left(G_{1j}^k\right)_{j \in [N]}$, we obtain that the term in \eqref{eq:app_symm_k} is symmetric in $\kappa$, so its derivative must satisfy $\partial_\kappa S(0) = 0$. This implies that
\begin{equation*}
    0 = \partial_\kappa \mathcal{R}^<(0, \nu) = -2 \gamma^2 \nu \mathbb{E}\left[\partial_\kappa U(0)\right]
\end{equation*}
But $U$ is strictly increasing, which implies that $\nu = 0$.

Therefore, the only critical point of $\mathcal{R}^<$ inside of $[-1,1]^2$ is $(0, 0)$. Moreover, since the Hessian
\begin{equation*}
    \mathcal{H}_{\mathcal{R}^<}(0,0) = 
    \begin{pmatrix}
        \partial_{\kappa \kappa} \mathcal{R}^< & -2\gamma^2\mathbb{E}\left[\partial_\kappa U(0)\right] \\
        -2\gamma^2\mathbb{E}\left[\partial_\kappa U(0)\right] & 0
    \end{pmatrix},
\end{equation*}
we obtain that $(0, 0)$ is a saddle point.

\paragraph{Case 2. $\nu = -1$}
It is clear that $\partial_\kappa \mathcal{R}^<(0, -1) > 0$, since $\mathbb{E}[U(\kappa)^2]$ and $\mathbb{E}[V(\kappa)^2]$ are even functions and $\mathbb{E}[U(\kappa)]$ is an increasing function in $\kappa$. Moreover, on $\kappa \in [0, 1]$, the function $\mathcal{R}^<(\kappa, -1)$ is clearly increasing in $\kappa$ since $\mathbb{E}[U(\kappa)]$ and $\mathbb{E}[V^2(\kappa)]$ are.

Let us now denote by $Z_1 \coloneqq \kappa \ell + \gamma G_{11}^k$ and by $Z_j \coloneqq  G_{1j}^k$, for $j \in \{2, 3, \dots, N\}$. Then, one has
\begin{equation}
\begin{cases}
    \partial_\kappa U(\kappa) = \partial_\kappa U\left(\lambda_r \kappa \ell + \lambda_r \gamma G_{11}^k, \lambda_r G_{12}^k, \dots, \lambda_r G_{1N}^k\right) = \lambda_r \ell \; \partial_1 U(\lambda_r Z_1, \dots, \lambda_r Z_N), \\
    \partial_\kappa V(\kappa) = \partial_\kappa V\left(\lambda_r \kappa \ell + \lambda_r \gamma G_{11}^k, \lambda_r G_{12}^k, \dots, \lambda_r G_{1N}^k\right) = \lambda_r \ell \; \partial_1 V(\lambda_r Z_1, \dots, \lambda_r Z_N).
\end{cases}
\end{equation}

By directly applying the chain rule to the expectation, the formula for the partial derivative of the risk becomes
\begin{align} \label{eq:app_r_lambda}
    \partial_\kappa \mathcal{R}^<(\kappa, -1) &= \partial_\kappa \mathbb{E}\left[\gamma^2\left(U(\kappa)^2 + 2U(\kappa)\right) + V(\kappa)^2 \sum_{j=2}^N \erf^2(\lambda_r Z_j) \right] \nonumber\\
    &= \lambda_r \ell  \;\mathbb{E}\left[\gamma^2\partial_1\left(U(\lambda_r Z)^2 + 2U(\lambda_r Z)\right) + \partial_1\left(V(\lambda_r Z)^2\right) \sum_{j=2}^N \erf^2(\lambda_r Z_j) \right] \nonumber\\
    &= 2\lambda_r \ell \; \mathbb{E}\left[\gamma^2 \partial_1 U(\lambda_r Z) \left(1 + U(\lambda_r Z)\right) + V(\lambda_r Z)\partial_1 V(\lambda_r Z) \sum_{j=2}^N \erf^2(\lambda_r Z_j) \right],
\end{align}
where $U(\lambda_r Z)$ stands for $U(\lambda_r Z_1, \dots, \lambda_r Z_N)$ and the same for the other functions. 

We will now show that if $\lambda_c = \tau \lambda_r$, for $\lambda_r$ small enough and the other hyperparameters fixed, the expectation above is strictly positive for all $\kappa \in [-1, 1]$. We shall make use of the Taylor expansions
\[
\erf(x) = \frac{2}{\sqrt{\pi}} x + \frac{\erf'''(\xi_x)}{6}x^3 = \frac{2}{\sqrt{\pi}} x + \mathcal{O}(x^3),
\]
for some $\xi_x \in [0,x]$, and
\[
e^{-x} = 1 - e^{-\zeta_x} x = 1 + \mathcal{O}(x),
\]
for some $\zeta_x \in [0, x]$. Note that $|\erf'''(x)| < 2.5$, for any $x \in \mathbb{R}$ and $|e^{-\zeta_x}| \leq 1$ for any $x \in [0, \infty)$. This in turn implies the following asymptotic scalings
\begin{align} \label{eq:app_r_asym}
    V(\lambda_r Z) &= \erf\left(\tau \lambda_r \sum_{i=1}^N Z_i \erf(\lambda_r Z_i)\right) = \frac{4\tau}{\pi} \lambda_r^2 \sum_{i=1}^N Z_i^2 + \mathcal{O}(\lambda_r^3), \nonumber\\
    U(\lambda_r Z) &= V(\lambda_r Z) \erf(\lambda_r Z_1) = \frac{8\tau}{\pi \sqrt{\pi}} \lambda_r^3 Z_1 \sum_{i=1}^N Z_i^2 + \mathcal{O}(\lambda_r^4), \nonumber\\
    \partial_1 V(\lambda_r Z) &= \frac{2\tau}{\sqrt{\pi}} e^{-\tau^2\left(\sum_{i=1}^N \lambda_r Z_i \erf(\lambda_r Z_i)\right)^2} \left(\frac{2}{\sqrt{\pi}}\lambda_r Z_1 e^{-\lambda_r^2 Z_1^2} + \erf(\lambda_r Z_1)\right) \nonumber\\
    &=\frac{2\tau}{\sqrt{\pi}}\left(1 - \mathcal{O}(\lambda_r^2) \right) \left(\frac{2}{\sqrt{\pi}}\lambda_r Z_1(1 - \mathcal{O}(\lambda_r^2)) + \frac{2}{\sqrt{\pi}} \lambda_r Z_1\right) \nonumber\\
    &= \frac{8\tau}{\pi}\lambda_r Z_1 + \mathcal{O}(\lambda_r^3), \nonumber\\
    \partial_1 U(\lambda_r Z) &= \frac{2}{\sqrt{\pi}}e^{-\lambda_r^2 Z_1^2} V(\lambda_r Z) + \partial_1 V(\lambda_r Z) \erf(\lambda_r Z_1) \nonumber\\
    &= \frac{8\tau}{\pi \sqrt{\pi}} \lambda_r^2 \sum_{i=1}^N Z_i^2 + \frac{16\tau}{\pi\sqrt{\pi}} \lambda_r^2 Z_1^2 + \mathcal{O}(\lambda_r^3) \nonumber\\
    &= \frac{8\tau}{\pi\sqrt{\pi}}\lambda_r^2 \left(\sum_{i=1}^N Z_i^2 + 2 Z_1^2\right) + \mathcal{O}(\lambda_r^3).
\end{align}
From these computations, it is also clear that every term $\mathcal{O}(\lambda_r^n) = C(Z_1, Z_2, \dots, Z_N) \lambda_r^n$, where $\mathbb{E}[|C(Z_1, \dots, Z_N)|] < \infty$.

Finally, \eqref{eq:app_r_lambda} becomes
\begin{equation*}
    \partial_\kappa \mathcal{R}^<(\kappa, -1) = \frac{16\tau \ell \gamma^2}{\pi\sqrt{\pi}}\lambda_r^3 \mathbb{E}\left[ \sum_{i=1}^N Z_i^2 + 2 Z_1^2\right] + \mathcal{O}(\lambda_r^4).
\end{equation*}
Since the coefficient of the leading term is positive, this concludes the proof that there exists $\lambda_r$ small enough such that $\partial_\kappa \mathcal{R}^<(\kappa, -1) > 0$ for all $\kappa \in [-1, 1]$.

\paragraph{Case 3. $\nu = 1$}
By leveraging the global symmetry of the theoretical risk function, specifically $\mathcal{R}^<(-\kappa, -\nu) = \mathcal{R}^<(\kappa, \nu)$, an analogous argument applies along the opposite boundary $\nu = 1$. The identical asymptotic smoothness guarantees $\partial_\kappa \mathcal{R}^<(\kappa, 1) < 0$ for all $\kappa \in [-1, 1]$

Finally combining the results of the three cases, one obtains the conclusion.

\subsection{Proof of Proposition~\ref{prop:diffeomorphism}}

We verify the claims of the proposition sequentially.

\textbf{Local diffeomorphism and eigenvalues at the origin.} Let $x = (\kappa, \nu)$. Because the gradient of the risk vanishes at the origin ($\partial_\kappa \mathcal{R}^<(0,0) = 0$) and $\mathcal{R}^<$ is smooth, the squared partial derivative satisfies $(\partial_\kappa \mathcal{R}^<(\kappa, \nu))^2 = \mathcal{O}(\|x\|_2^2)$. This allows us to write the Taylor expansion for the first component of $g$, denoted $g_1(\kappa, \nu)$:
\begin{align*}
    g_1(\kappa, \nu) &= \frac{\kappa - \alpha(\partial_\kappa \mathcal{R}^<(\kappa, \nu))(1 - \kappa^2)}{\sqrt{1 + \alpha^2(\partial_\kappa \mathcal{R}^<(\kappa, \nu))^2(1 - \kappa^2)}} \\
    &= \frac{\kappa - \alpha(\partial_\kappa \mathcal{R}^<(\kappa, \nu))(1 - \kappa^2)}{\sqrt{1 + \mathcal{O}(\|x\|_2^2)}} \\
    &= (\kappa - \alpha(\partial_\kappa \mathcal{R}^<(\kappa, \nu))(1 - \kappa^2))(1 + \mathcal{O}(\|x\|_2^2)) \\
    &= \kappa - \alpha\partial_\kappa \mathcal{R}^<(\kappa, \nu) + \mathcal{O}(\|x\|_2^2).
\end{align*}
Applying the same expansion to the second component $g_2(\kappa, \nu)$ yields the Jacobian matrix of $g$ evaluated at the origin, $J_g(0,0) = I_2 - \alpha H_{\mathcal{R}^<}(0,0)$. Because $\det(H_{\mathcal{R}^<}(0,0)) = -4\gamma^4\left(\mathbb{E}[\partial_\kappa U(0)]\right)^2$, from the proof of Proposition~\ref{prop:fixed_points}, the eigenvalues of $H_{\mathcal{R}^<}(0,0)$ have strictly opposite signs. Consequently, for a sufficiently small step size $\alpha > 0$, $J_g(0,0)$ possesses one eigenvalue strictly inside the interval $(0, 1)$ and another strictly greater than $1$. This confirms that $J_g(0,0)$ is invertible, which by the Inverse Function Theorem ensures $g$ is a local diffeomorphism around $(0,0)$.

\textbf{Differentiability and non-degenerate Jacobian on $[-1, 1]^2$.} The mapping $g$ is globally differentiable on $[-1,1]^2$ as it is a composition of smooth elementary functions. To prove its Jacobian is nowhere degenerate, we express the mapping as $g(x) = x + \alpha h(x)$. The first coordinate of $h$ takes the form:
\begin{align*}
    h_1(\kappa, \nu) &= \frac{1}{\alpha}(g_1(\kappa, \nu) - \kappa) \\
    &= \underbrace{\frac{\kappa}{\alpha}\left(\frac{1}{\sqrt{1 + \alpha^2(\partial_\kappa \mathcal{R}^<(\kappa, \nu))^2(1 - \kappa^2)}} - 1\right)}_{\text{Term 1}} - \underbrace{\frac{\partial_\kappa \mathcal{R}^<(\kappa, \nu)(1 - \kappa^2)}{\sqrt{1 + \alpha^2(\partial_\kappa \mathcal{R}^<(\kappa, \nu))^2(1 - \kappa^2)}}}_{\text{Term 2}}.
\end{align*}
We must prove that the gradient of $h_1(\kappa, \nu)$ remains uniformly bounded for all $\alpha \le 1$. Term 2 is a composition of functions with uniformly bounded gradients on the compact domain, so its gradient is trivially bounded. For Term 1, we introduce the auxiliary function $q(a, b) = \frac{a}{\alpha}\left(\frac{1}{\sqrt{1 + \alpha^2 b}} - 1\right)$ and define the supremum $B = \sup_{[-1,1]^2} (\partial_\kappa \mathcal{R}^<(\kappa, \nu))^2(1-\kappa^2)$. Term 1 can be expressed as $q(\kappa, (\partial_\kappa \mathcal{R}^<)^2(1-\kappa^2))$. Since the inner function is smooth and independent of $\alpha$, we only need to bound the partial derivatives of $q$. By Lemma~\ref{lemma:inv_sqrt}, for any $\alpha \le 1$:
\begin{align*}
    |\partial_a q(a, b)| &= \frac{1}{\alpha}\left|\frac{1}{\sqrt{1 + \alpha^2 b}} - 1\right| \le \alpha b \le B, \\
    |\partial_b q(a, b)| &= \left|-\frac{\alpha a}{2(1 + \alpha^2 b)^{3/2}}\right| \le \frac{\alpha}{2} \le \frac{1}{2}.
\end{align*}
Since both partial derivatives are uniformly bounded, the gradient of $h_1$ is uniformly bounded. A symmetric argument applies to $h_2$, confirming that the Jacobian matrix $J_h(\kappa, \nu)$ is uniformly bounded over $\alpha \le 1$. Because $J_g(\kappa, \nu) = I_2 + \alpha J_h(\kappa, \nu)$, selecting a sufficiently small $\alpha$ ensures the eigenvalues of $J_g$ remain strictly positive, guaranteeing the Jacobian is nowhere degenerate.

\textbf{Injectivity.} The uniform bound on $J_h$ implies that $h$ is $B_L$-Lipschitz continuous on $[-1,1]^2$ for some constant $B_L > 0$ independent of $\alpha$ (for small $\alpha$). Choose $\alpha < 1/B_L$. Suppose there exist two distinct points $x, y \in [-1, 1]^2$ such that $g(x) = g(y)$. This implies:
\begin{equation*}
    x + \alpha h(x) = y + \alpha h(y) \implies x - y = \alpha(h(y) - h(x)).
\end{equation*}
Taking the Euclidean norm of both sides yields:
\begin{equation*}
    \|x - y\|_2 = \alpha \|h(y) - h(x)\|_2 \le \alpha B_L \|x - y\|_2.
\end{equation*}
Because $\alpha B_L < 1$, this inequality forces $\|x - y\|_2 < \|x - y\|_2$, which is a strict contradiction. Therefore, $g$ must be injective on the domain $[-1, 1]^2$.

\subsection{Final details.} 
From Proposition~\ref{prop:risk_decrease}, the distance between successive iterates $\| (\kappa_{t+1}, \nu_{t+1}) - (\kappa_t, \nu_t) \|_2$ tends to zero. Because the sequence $(\kappa_t, \nu_t)_{t \ge 0}$ lies within the compact set $[-1, 1]^2$, the set of its accumulation points must be connected. Given that the set of available fixed points is finite and isolated $\{ (0,0), (\pm 1, \pm 1), (\pm 1, \mp 1) \}$, the sequence cannot oscillate and must uniquely converge to exactly one of them.

It remains to be shown that the probability of converging to the saddle point $(0,0)$ or the maxima $(\pm 1, \mp 1)$ is zero. Because the risk strictly decreases at every step (Proposition~\ref{prop:risk_decrease}), convergence to the global maxima is impossible unless initialized exactly upon them. 

To systematically rule out the saddle point $(0,0)$, we apply the Center-Stable Manifold Theorem \citep[Theorem III.7]{shub2013global}. Because $J_g(0,0)$ has exactly one eigenvalue with magnitude strictly greater than 1 (Proposition~\ref{prop:diffeomorphism}), the local center-stable manifold $W_{loc}^{cs}$ associated with the origin has codimension 1. In the 2-dimensional space $[-1,1]^2$, a codimension-1 manifold is a curve, and its Lebesgue measure is strictly zero.

The stable set $W^s$, consisting of all initializations that converge to $(0,0)$ under iterations of $g$, is contained within the countable union of pre-images of $W_{loc}^{cs}$:
\begin{equation}
    W^s \subset \bigcup_{T \ge 0} g^{-T}(W_{loc}^{cs}).
\end{equation}
Because $g$ is injective and differentiable, its inverse preserves sets of measure zero. The countable union of measure-zero sets is also of measure zero. 
    
Therefore, the set of all initializations on $\mathcal{M}$ that converge to the saddle point $(0,0)$ has Lebesgue measure zero. We conclude that for a generic random initialization $(k_0, v_0) \in \mathcal{M}$, the PGD sequence almost surely bypasses the saddle point and the maxima, converging strictly to the true oracle structure $(\pm 1, \pm 1)$.

\section{Technical lemmas}
In this section we list the technical lemmas used throughout the proofs in the Appendix.

\begin{lemma} \label{lemma:erf_ineq}
    For $u \in \mathbb{R}$,
    \begin{equation*}
        \erf^2(u) \leq \frac{4}{\pi} u^2.
    \end{equation*}
\end{lemma}
\textit{Proof.} The proof relies on the concavity of $\erf$ on $(0, \infty)$. Since $\erf$ is an odd function, it is enough to prove the statement for $u \geq 0$. Since $\erf$ is concave,
\begin{equation*}
    \erf(u) \leq \erf(0) + \erf'(0) u = \frac{2}{\sqrt{\pi}}u,
\end{equation*}
which implies the conclusion.

\begin{lemma}\label{lemma:expected_erf}
    For $G \sim \mathcal{N}(0, \sigma^2)$ and $t \in \mathbb{R}$,
    \begin{equation*}
        \mathbb{E}[\mathrm{erf}(t + G)] = \mathrm{erf}\left(\frac{t}{\sqrt{1+2\sigma^2}}\right).
    \end{equation*}
\end{lemma}
\textit{Proof.} One has
\begin{align*}
    \mathbb{E}[\erf(t + G)] &= \int_{0}^t \mathbb{E}[\erf'(s + G)] \text{d}s \\
    & = \int_{0}^t \mathbb{E}\left[\frac{2}{\sqrt{\pi}} e^{-(s+G)^2}\right] \text{d}s \\
    & = \int_{0}^t \int_{-\infty}^\infty \frac{\sqrt{2}}{\pi\sigma} e^{-(s+g)^2} e^{-\frac{g^2}{2\sigma^2}} \text{d}g \; \text{d}s \\
    & = \frac{\sqrt{2}}{\pi\sigma} \int_{0}^t \int_{-\infty}^\infty e^{-g^2 \frac{1 + 2\sigma^2}{2\sigma^2} - 2gs - s^2} \text{d}g \; \text{d}s \\
    & = \frac{\sqrt{2}}{\pi\sigma} \int_{0}^t \int_{-\infty}^\infty e^{-\left(g \sqrt{\frac{1+2\sigma^2}{2\sigma^2}} + s \sqrt{\frac{2\sigma^2}{1 + 2\sigma^2}}\right)^2} e^{-\frac{1}{1 + 2\sigma^2} s^2} \text{d}g \; \text{d}s \\
    & = \frac{\sqrt{2}}{\pi\sigma} \int_{0}^t \sqrt{\pi} \sqrt{\frac{2\sigma^2}{1 + 2\sigma^2}} e^{-\frac{1}{1 + 2\sigma^2} s^2}  \text{d}s \\
    & = \frac{2}{\sqrt{\pi}} \int_{0}^{\frac{t}{\sqrt{1 + 2\sigma^2}}} e^{-u^2} \text{d}u \\
    & = \erf\left(\frac{t}{\sqrt{1+2\sigma^2}}\right).
\end{align*}

\begin{lemma} \label{lemma:inv_sqrt}
    For $u \geq 0$,
    \begin{equation*}
        \left| \frac{1}{\sqrt{1+u}} - 1\right| \leq u.
    \end{equation*}
\end{lemma}

\textit{Proof.} For $u \geq 0$, the term inside the absolute value is at most 0, so we need to prove that
\begin{equation*}
    u + \frac{1}{\sqrt{1 + u}} \geq 1.
\end{equation*}
This statement is obvious for $u \geq 1$. For $u \in [0,1)$, it is equivalent to
\begin{equation*}
    \frac{1}{\sqrt{1 + u}} \geq 1 - u  \iff \frac{1}{1 + u} \geq (1 - u)^2 \iff 1 \geq (1 - u^2)(1 - u),
\end{equation*}
which is clearly true since $u \in [0,1)$. One has equality only when $u = 0$.

\section{Additional Experimental Details and Results}  \label{sec:app_numerical}

In this section, we provide supplementary empirical results that expand upon the optimization dynamics discussed in Section~\ref{sec:sec05}. Specifically, we analyze the training dynamics of the predictors without manifold penalties in Section~\ref{sec:app_no_penalty}, provide the full results for the experiments in Figure~\ref{fig:vary_lambda_gradients} in Section~\ref{sec:app_constant_lambda} and further investigate the optimization landscape for decoupled $\lambda_c$ and $\lambda_r$ in Section~\ref{sec:app_decoupled_lambda}, and for different values of $\ell$ and $N$ in Section~\ref{sec:app_grad_N_ell}.

\subsection{Implementation details and hyperparameters}

\paragraph{Code availability.}
The complete codebase used to run all numerical simulations and generate the figures presented in this work can be found at \url{https://github.com/anditp/Single-Location-Regression.git}. All the numerical experiments were run on a MacBook Pro M1 on CPU. Each individual experiment (20 runs) in the paper takes less than an hour to run.

\paragraph{Distance to the invariant manifold.}
Throughout our experiments, we track the optimization trajectories to determine if they remain on, or converge to, the invariant manifold $\mathcal{M}$. Recall that $\mathcal{M}$ is defined by the strict orthogonality constraints:
\begin{equation*}
    \mathcal{M} = \left\{(k, v) \in \left(\mathbb{S}^{d-1}\right)^2 \; \Big\vert \; k^T v^* = 0, \; v^T k^* = 0, \; k^T v = 0 \right\}.
\end{equation*}
Consequently, for any parameter pair $(k, v) \in \left(\mathbb{S}^{d-1}\right)^2$, we compute its distance to $\mathcal{M}$ continuously during training using the root sum of squares of the corresponding inner products:
\begin{equation}
    d_{\mathcal{M}}(k, v) \coloneqq \sqrt{(k^T v^*)^2 + (v^T k^*)^2 + (k^T v)^2}.
\end{equation}

\paragraph{Hyperparameters.}
Table~\ref{tab:hyperparameters} summarizes the specific values and schedules used across all numerical experiments presented in Section~\ref{sec:sec05} and Appendix~\ref{sec:app_numerical}. For the penalized empirical risk experiments, we introduced the weighting coefficients $\mu_1$ and $\mu_2$ to enforce the structural constraints.

\begin{table}[htbp]
    \centering
    \caption{Summary of hyperparameters used across the numerical experiments.}
    \label{tab:hyperparameters}
    \resizebox{\textwidth}{!}{%
    \begin{tabular}{l c c c c}
        \toprule
        \textbf{Parameter} & \textbf{Fig.~\ref{fig:on_manifold_full}} & \textbf{Fig.~\ref{fig:vary_lambda_gradients}}, \textbf{Fig.~\ref{fig:on_manifold_lc_0p01_lr0p1}} & \textbf{Fig.~\ref{fig:comparison_manifold}} & \textbf{Fig.~\ref{fig:app_no_penalties}} \\
        & ($\mathcal{M}$ initialization) & (Constant $\lambda$s) & (Penalized) & (Unpenalized) \\
        \midrule
        Grid height ($M$) & 5 & 5 & 5 & 5 \\
        Grid width ($N$) & 5 & 5 & 5 & 5 \\
        Embedding dim. ($d$) & 20 & 20 & 20 & 20 \\
        Shift magnitude ($\ell$) & 100 & 100 & 100 & 100 \\
        Special token variance ($\gamma^2$) & 1 & 1 & 1 & 1 \\
        Noise variance ($\epsilon^2$) & 0.1 & 0.1 & 0.1 & 0.1 \\
        Penalty weight ($\mu_1$) & 0 & 0 & 10 & 0 \\
        Penalty weight ($\mu_2$) & 0 & 0 & 1 & 0 \\
        Learning rate ($\alpha$) & 1e-4 & 1e-4 & 1e-4 (a) / 1e-3 (b) & 1e-4 (a) / 1e-3 (b) \\
        Number of iterations & 100,000 & 100,000 & 100,000 & 100,000 \\
        Batch size & 256 & 256 & 256 & 256 \\
        Independent runs & 20 & 20 / $\lambda$ value & 20 & 20 \\
        \bottomrule
    \end{tabular}
    }
\end{table}
The batches for Stochastic PGD are generated dynamically and resampled at each step, since we are aiming to mimic PGD on the theoretical risk.

\subsection{Optimization without Manifold Penalties} \label{sec:app_no_penalty}

In Section~\ref{sec:sec05}, we demonstrated that optimizing the penalized empirical risk allows both the $\erf$ and $\text{softmax}$ predictors to converge to the oracle parameters from generic random initializations. To isolate the effect of the penalty terms, we conduct an ablation study where both predictors are trained using the strictly unpenalized risk from generic initializations outside of $\mathcal{M}$. 

Figure~\ref{fig:app_no_penalties} presents the optimization trajectories for both $T^{(k,v)}_{\lambda_r, \lambda_c}$ ($\erf$ predictor) and $T^{(k, v)}_{\lambda_r, \lambda_c, \text{softmax}}$ ($\text{softmax}$ predictor) under these unconstrained conditions. Without the explicit geometric regularization provided by the penalties, the trajectories fail to always converge. For the $\erf$ predictor there are runs which get trapped in high loss regions. On the other hand, the optimization of the softmax predictor appears better behaved, with most of the runs converging in a similar manner as for the penalized version; however, there are some runs that get trapped in regions with high positive $\nu$ and high negative $\kappa$, far from $\mathcal{M}$. This empirical finding underscores the practical necessity of the regularization terms when operating outside the theoretical guarantees of Theorem~\ref{thm:pgd_convergence}.

\begin{figure}[htbp]
    \centering
    \begin{subfigure}{\textwidth}
        \centering
        \includegraphics[width=0.24\textwidth]{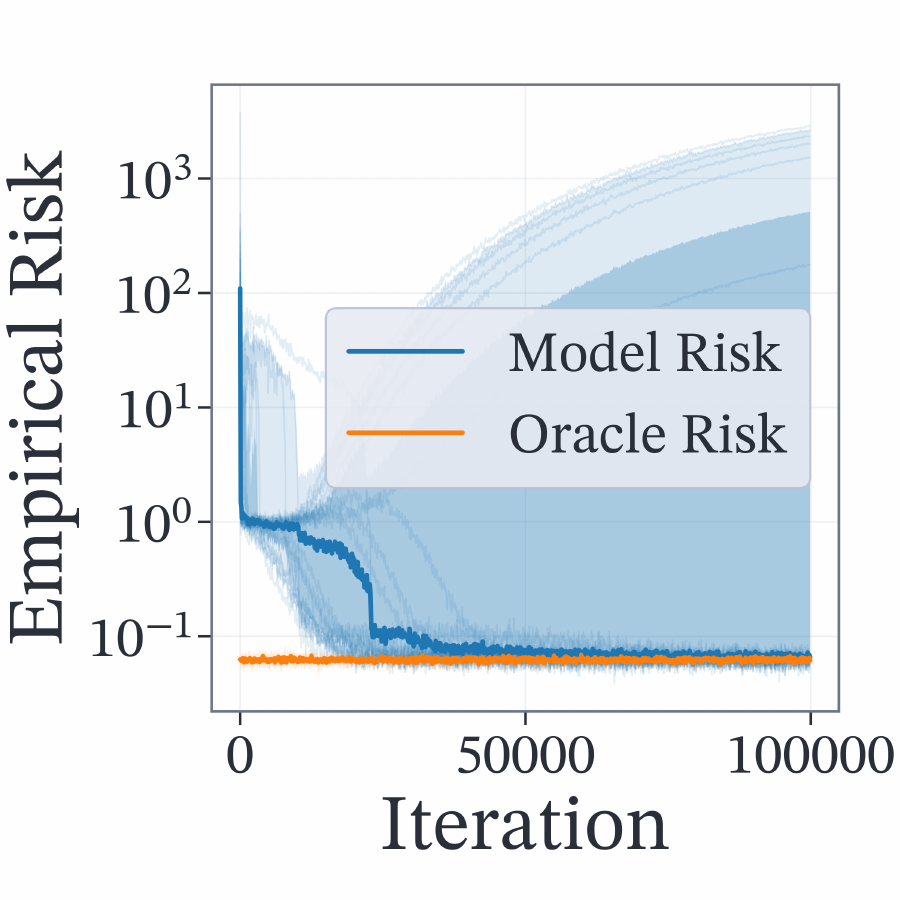}\hfill
        \includegraphics[width=0.24\textwidth]{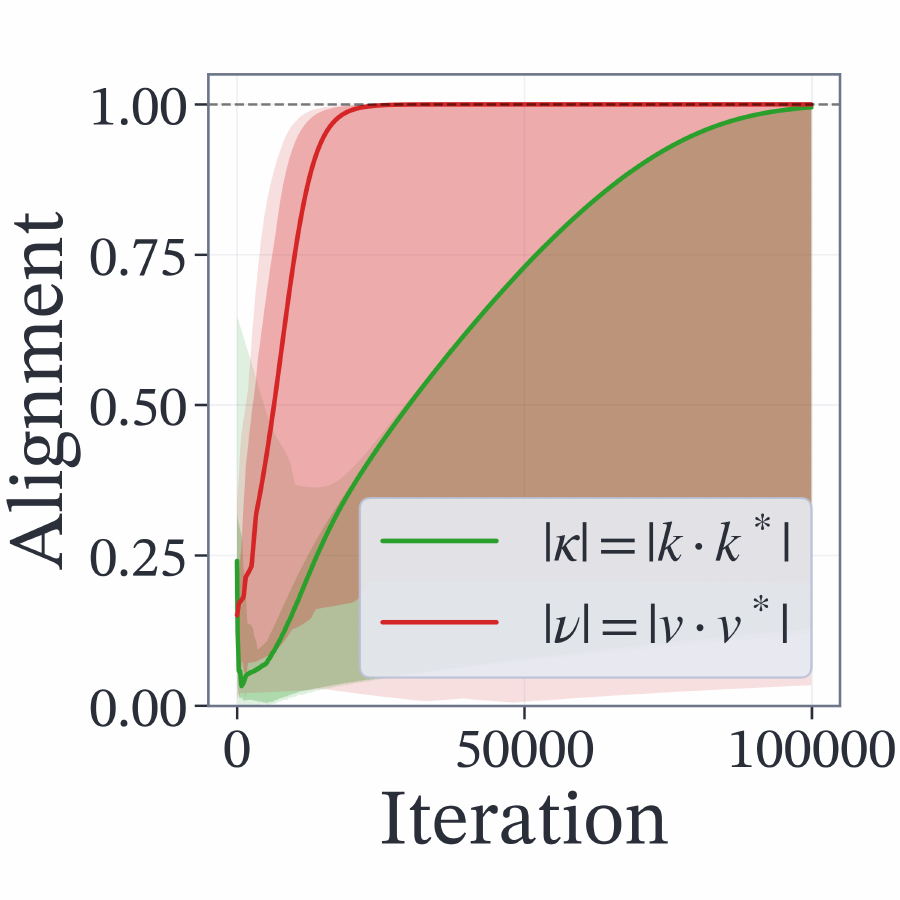}\hfill
        \includegraphics[width=0.24\textwidth]{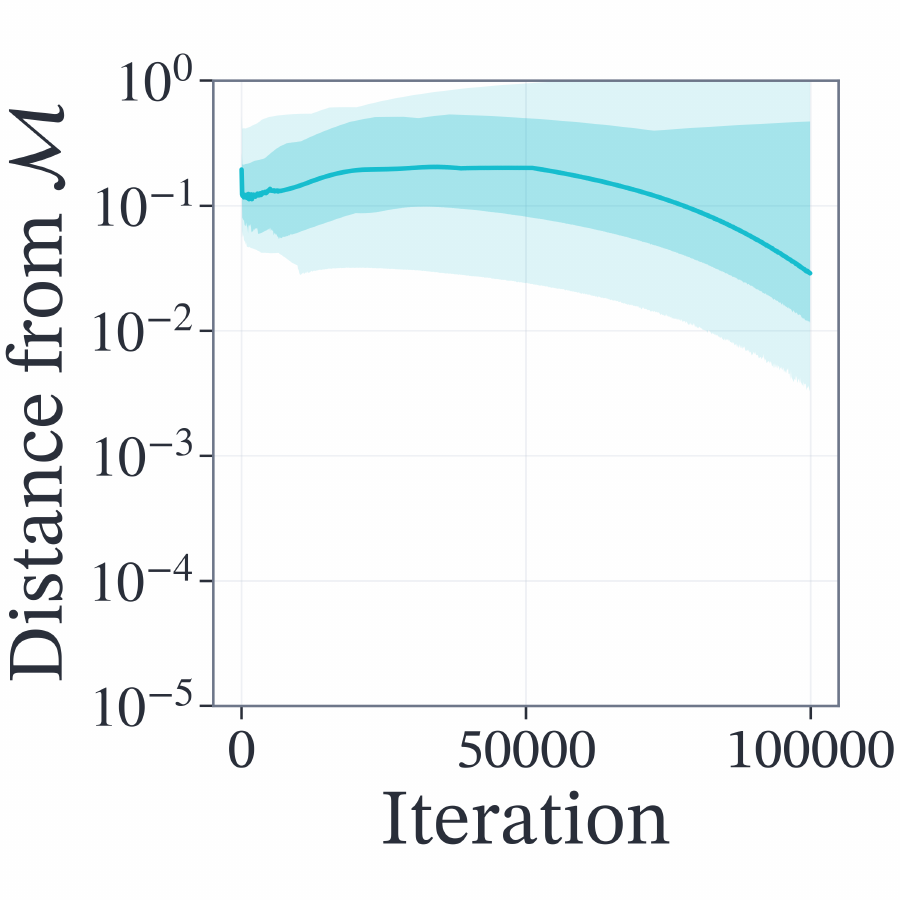}\hfill
        \includegraphics[width=0.24\textwidth]{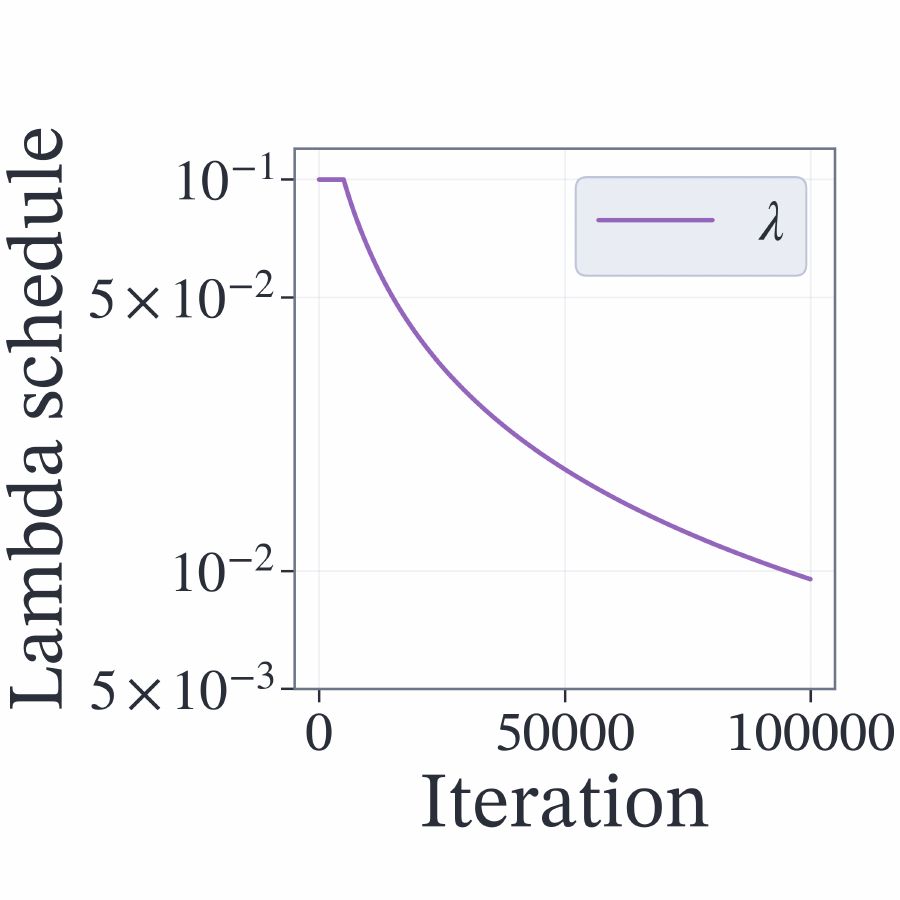}
        \caption{Dynamics for the $\erf$ predictor.}
    \end{subfigure}
    
    \vspace{1em}

    \begin{subfigure}{\textwidth}
        \centering
        \includegraphics[width=0.24\textwidth]{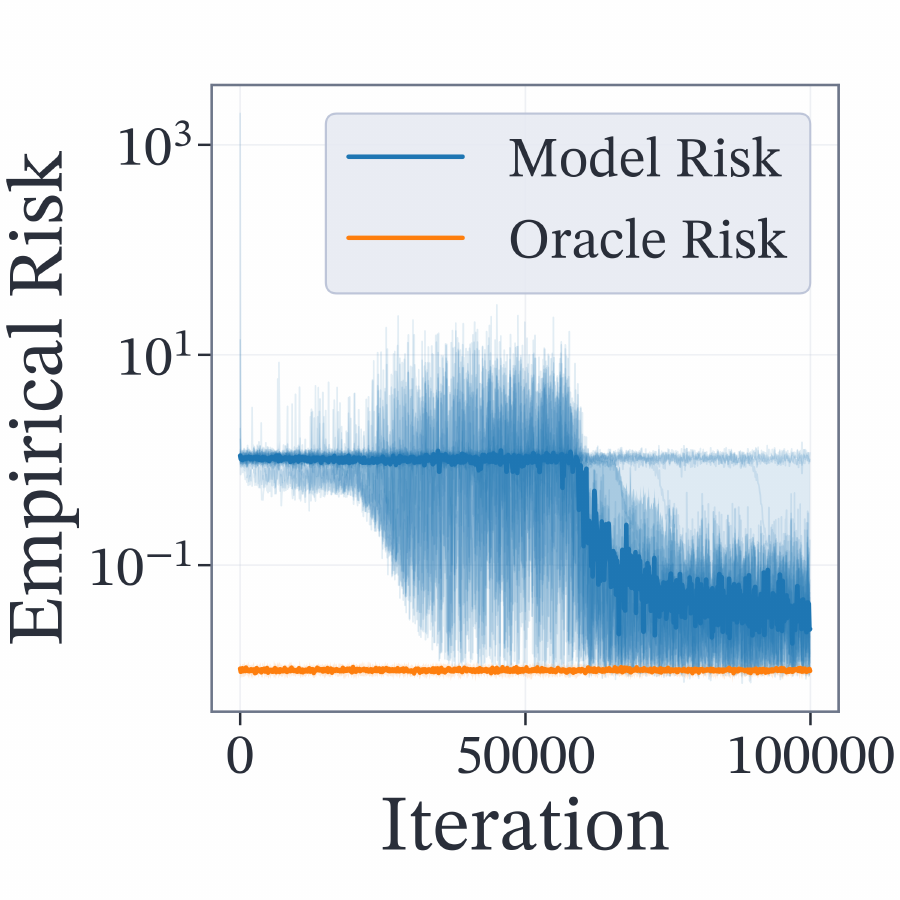}\hfill
        \includegraphics[width=0.24\textwidth]{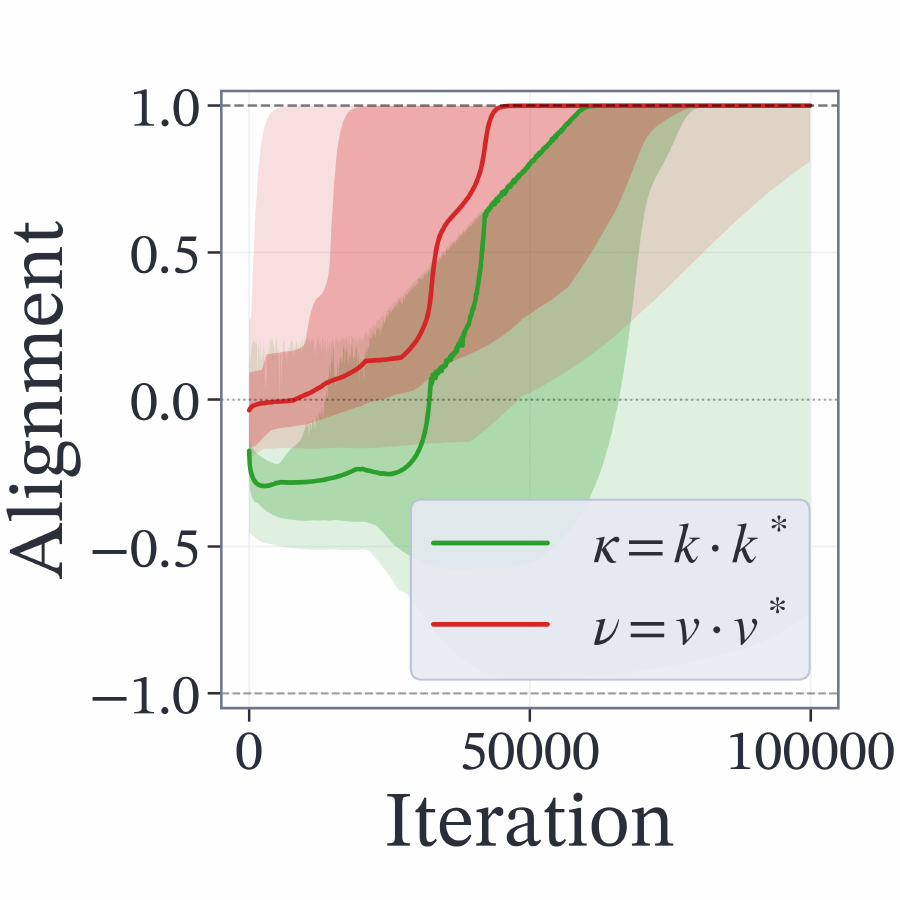}\hfill
        \includegraphics[width=0.24\textwidth]{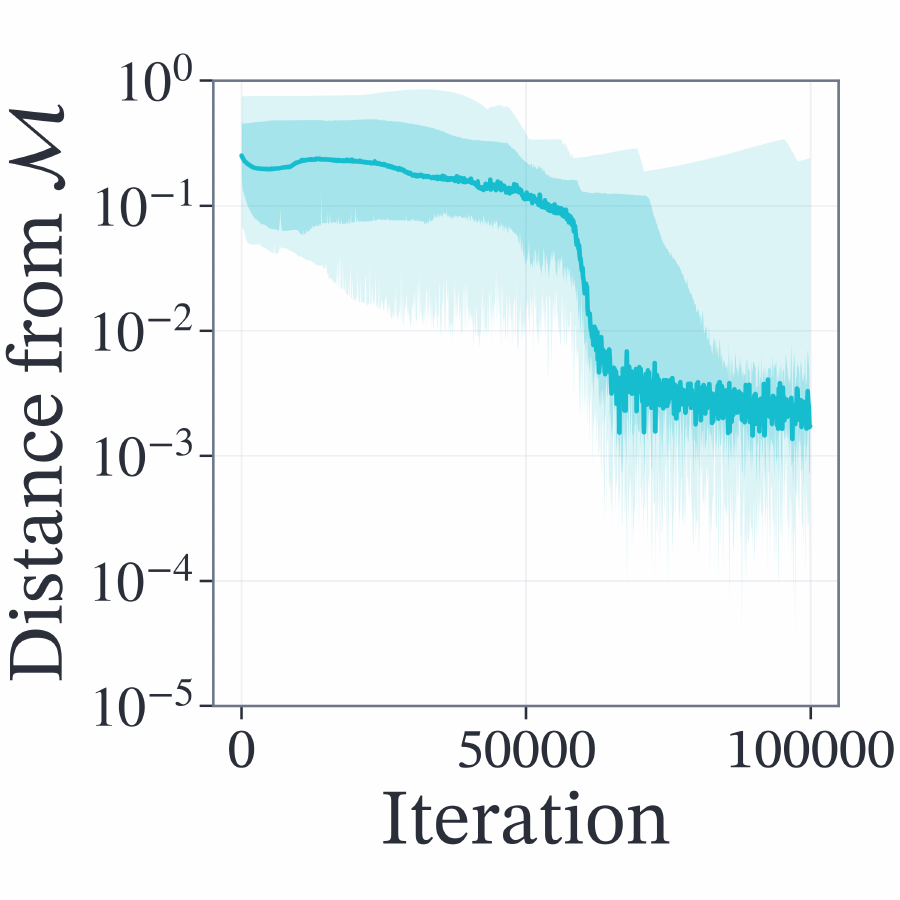}\hfill
        \includegraphics[width=0.24\textwidth]{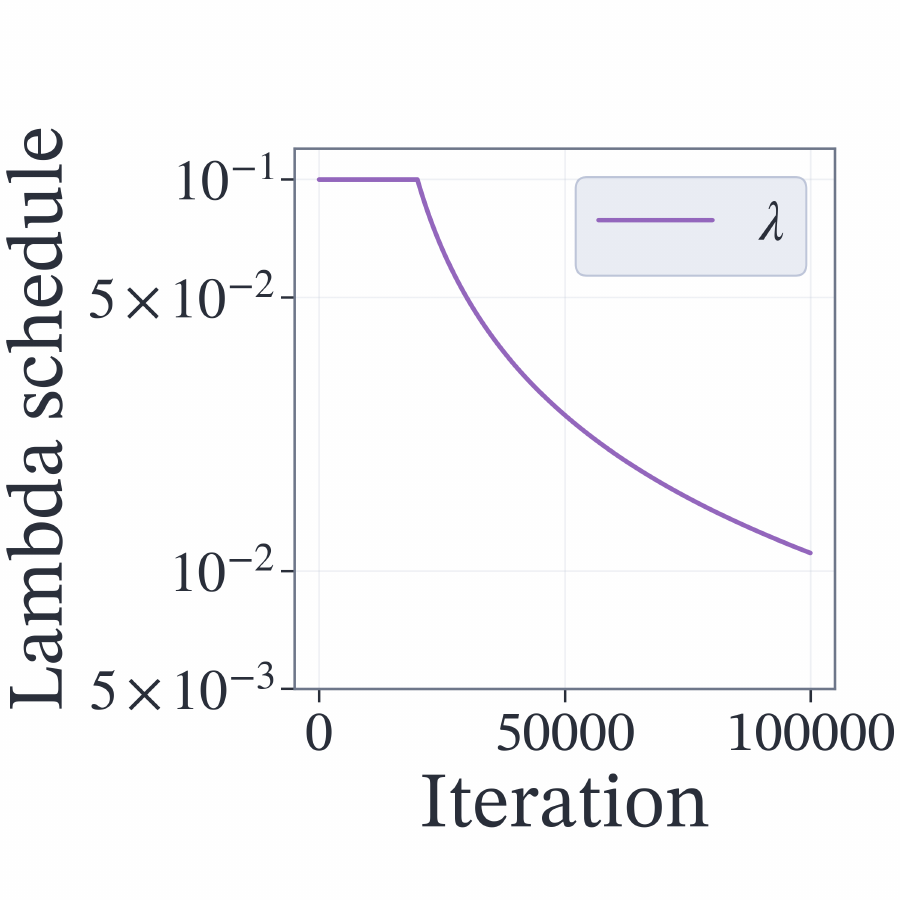}
        \caption{Dynamics for the $\text{softmax}$ predictor.}
    \end{subfigure}
    \caption{Optimization dynamics from 20 random initializations  without the alignment penalties ($\mu_1 = \mu_2 = 0$). The solid curve represents the empirical median across the runs. The bands represent the 0.05-0.95 and, respectively, the 0.25-0.75 quantile intervals. The parameters used are $M = N = 5$, $d = 20$, $\gamma = 1$, $\ell = 100$ and $\lambda_t =  10^{-1} / (1 + \max(t - 5\cdot 10^4, 0) \cdot 10^{-4})$, for the $\erf$ predictor and $\lambda_t =  10^{-1} / (1 + \max(t - 2\cdot 10^4, 0) \cdot 10^{-4})$ for the softmax predictor. For computing the oracle loss we use the initial inverse temperatures.}
    \label{fig:app_no_penalties}
\end{figure}

\subsection{Full results for the experiment in Figure~\ref{fig:vary_lambda_gradients}} \label{sec:app_constant_lambda}

In Section~\ref{sec:sec05}, we discussed the optimization dynamics under various constant inverse temperatures following an initial warm-up phase. Here, we present the full empirical results for this experiment.

As illustrated in Figures~\ref{fig:lambda_0.01}, \ref{fig:lambda_0.03}, \ref{fig:lambda_0.05}, \ref{fig:lambda_0.07}, the training procedure begins with a 20,000 step warm up using $\lambda_c = \lambda_r = 0.1$. After this period, the inverse temperatures are held constant at values $\lambda \in \{0.01, 0.03, 0.05, 0.07\}$. We observe that without a meticulously chosen lambda schedule, not all the runs converged in the simulated 100,000 steps. For lower values of $\lambda$, more runs got stuck around the origin compared to the more consistent convergence obtained for larger $\lambda$s.

These full results empirically validate the need to employ a decreasing schedule for the inverse temperatures. A static $\lambda$ forces a compromise between quickly escaping the saddle point at the origin and converging precisely at the global minima. A schedule dynamically adapts the landscape to first push the parameters away from the saddle point, and subsequently allows for fine-tuned convergence to the exact oracle parameters.

\begin{figure}[htbp]
    \centering
    \begin{subfigure}[b]{0.24\textwidth}
        \centering
        \includegraphics[width=\textwidth]{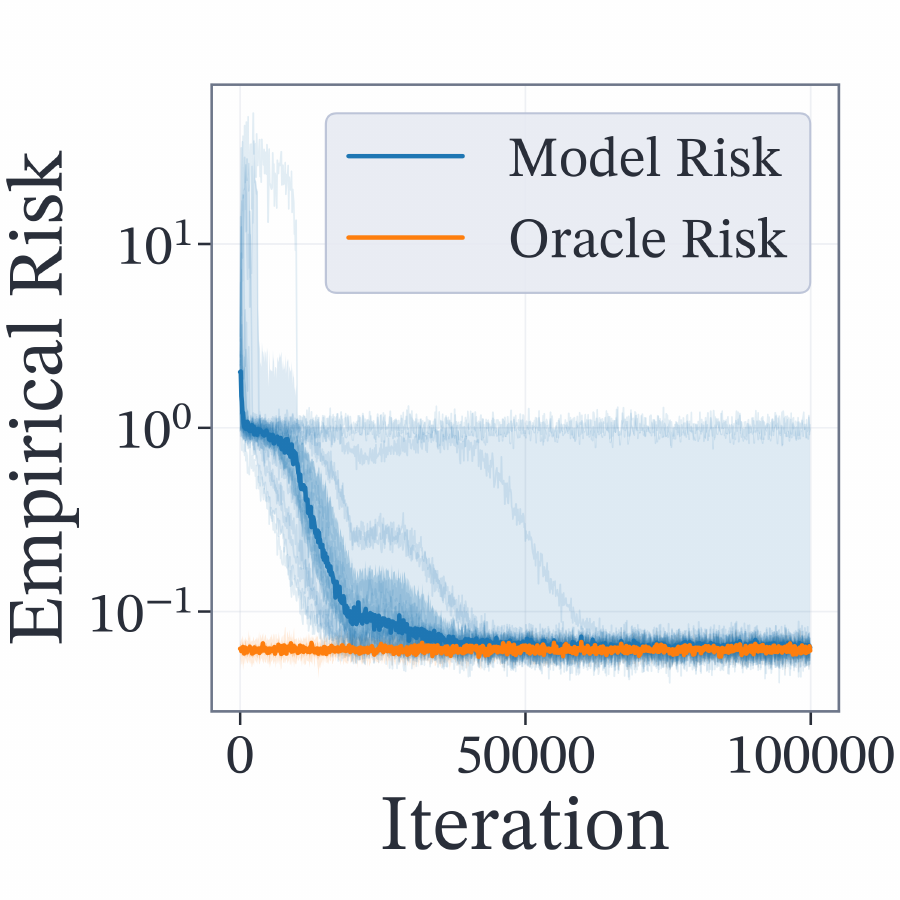}
        \label{fig:lambda_0.01_loss}
    \end{subfigure}
    \hfill
    \begin{subfigure}[b]{0.24\textwidth}
        \centering
        \includegraphics[width=\textwidth]{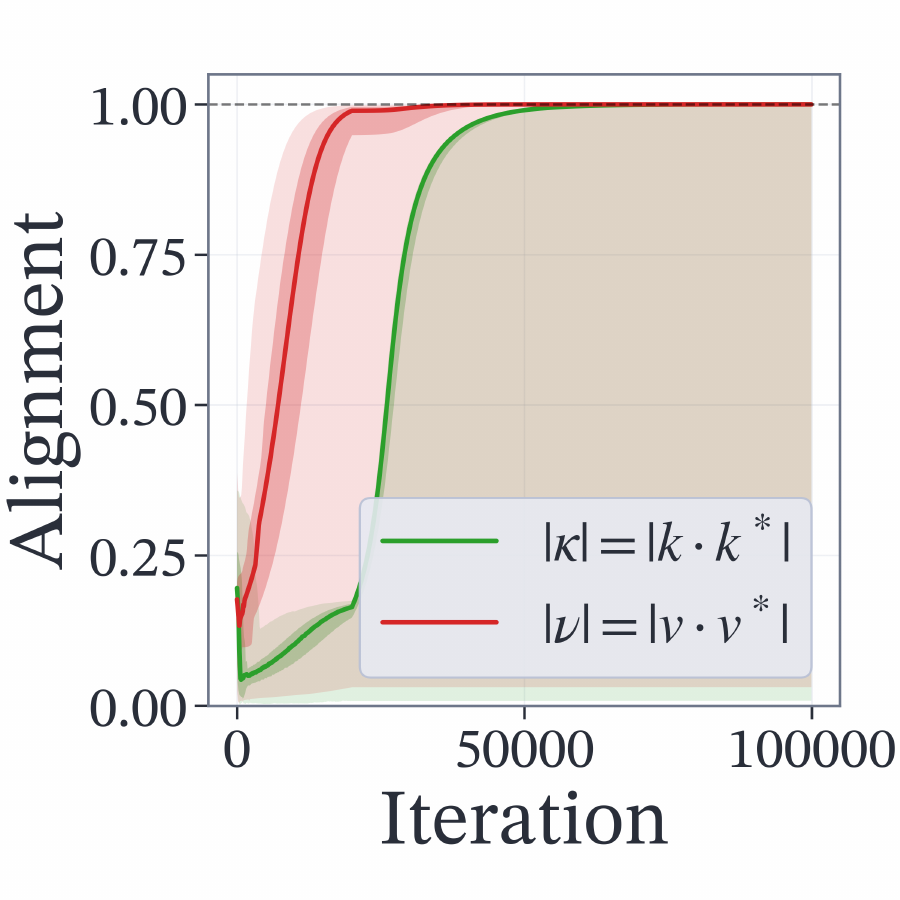}
        \label{fig:lambda_0.01_kappa}
    \end{subfigure}
    \hfill
    \begin{subfigure}[b]{0.24\textwidth}
        \centering
        \includegraphics[width=\textwidth]{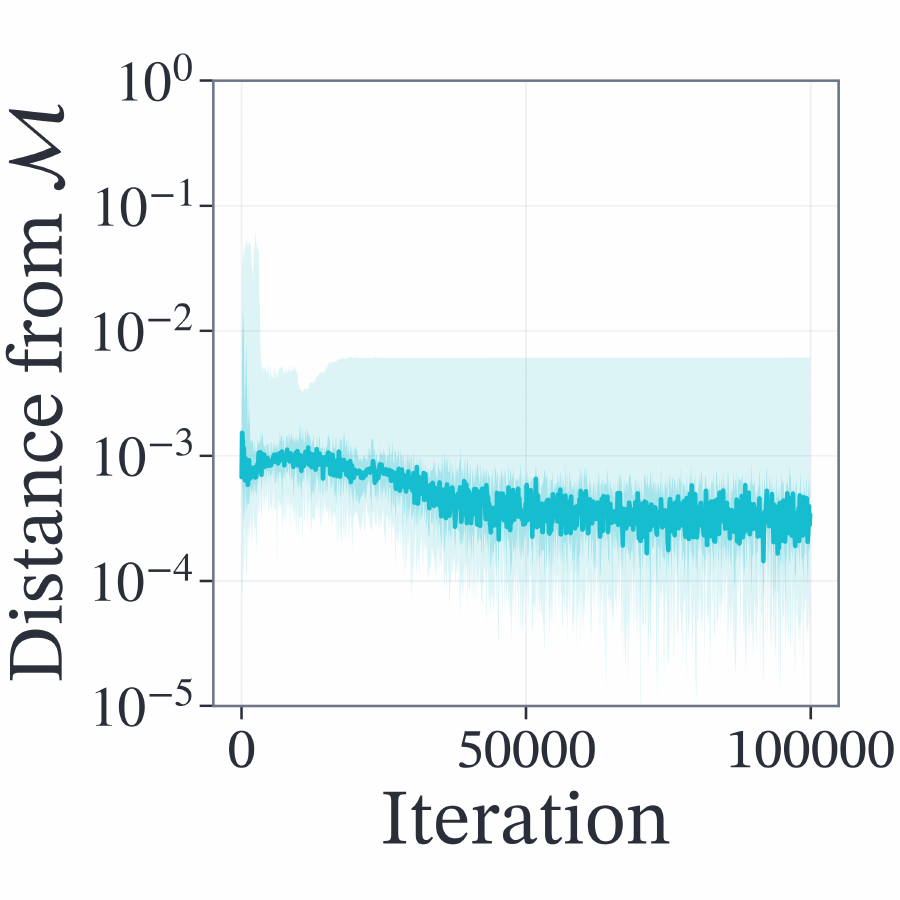}
        \label{fig:lambda_0.01_nu}
    \end{subfigure}
    \hfill
    \begin{subfigure}[b]{0.24\textwidth}
        \centering
        \includegraphics[width=\textwidth]{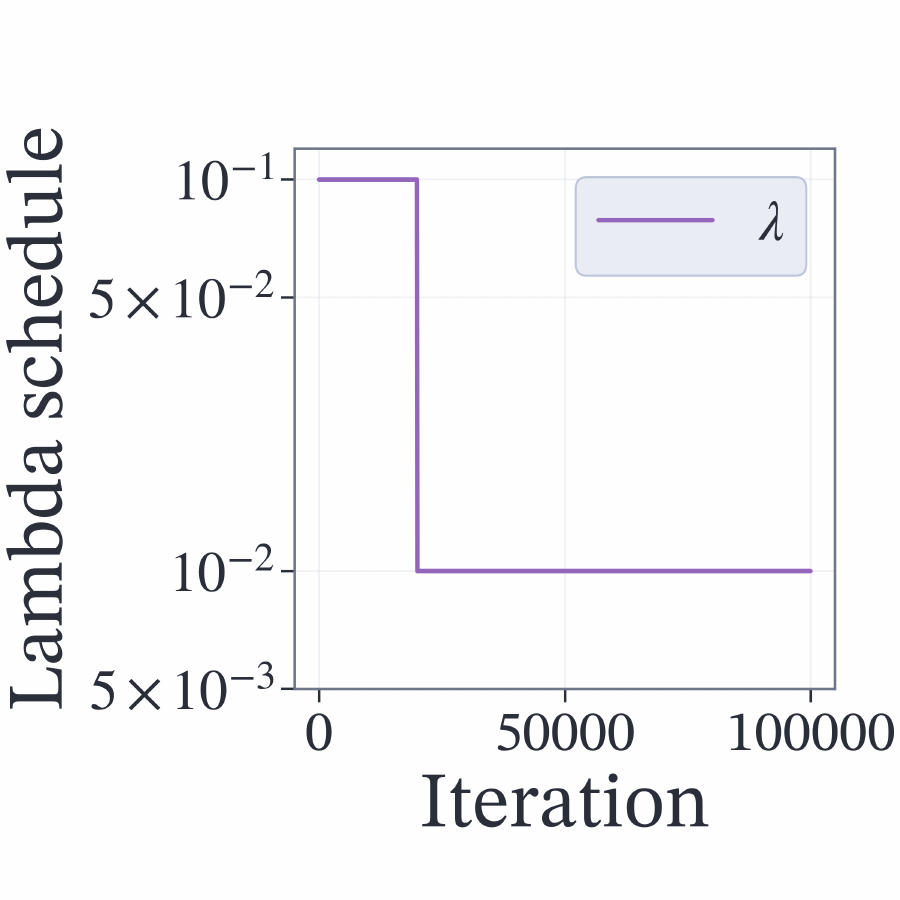}
        \label{fig:lambda_0.01_lambda}
    \end{subfigure}
    
    \caption{Training dynamics of the predictor with a 20k warm-up phase $\lambda_c = \lambda_r = \lambda = 0.1$ followed by constant inverse temperatures $\lambda_c = \lambda_r = 0.01$. The solid curve represents the empirical median across 20 runs and  the bands represent the 0.05-0.95 and, respectively, the 0.25-0.75 quantile intervals (although not always visible). The parameters used are $M = N = 5$, $d = 20$, $\gamma = 1$, $\ell = 100$.}
    \label{fig:lambda_0.01}
\end{figure}

\begin{figure}[htbp]
    \centering
    \begin{subfigure}[b]{0.24\textwidth}
        \centering
        \includegraphics[width=\textwidth]{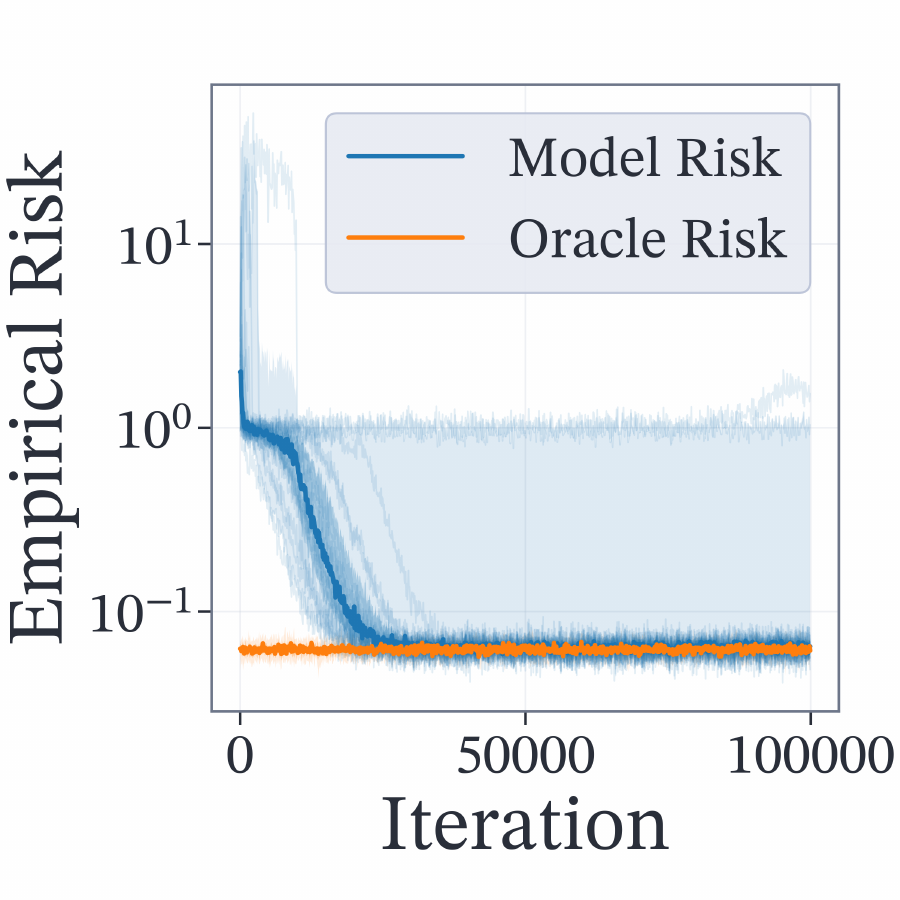}
        \label{fig:lambda_0.03_loss}
    \end{subfigure}
    \hfill
    \begin{subfigure}[b]{0.24\textwidth}
        \centering
        \includegraphics[width=\textwidth]{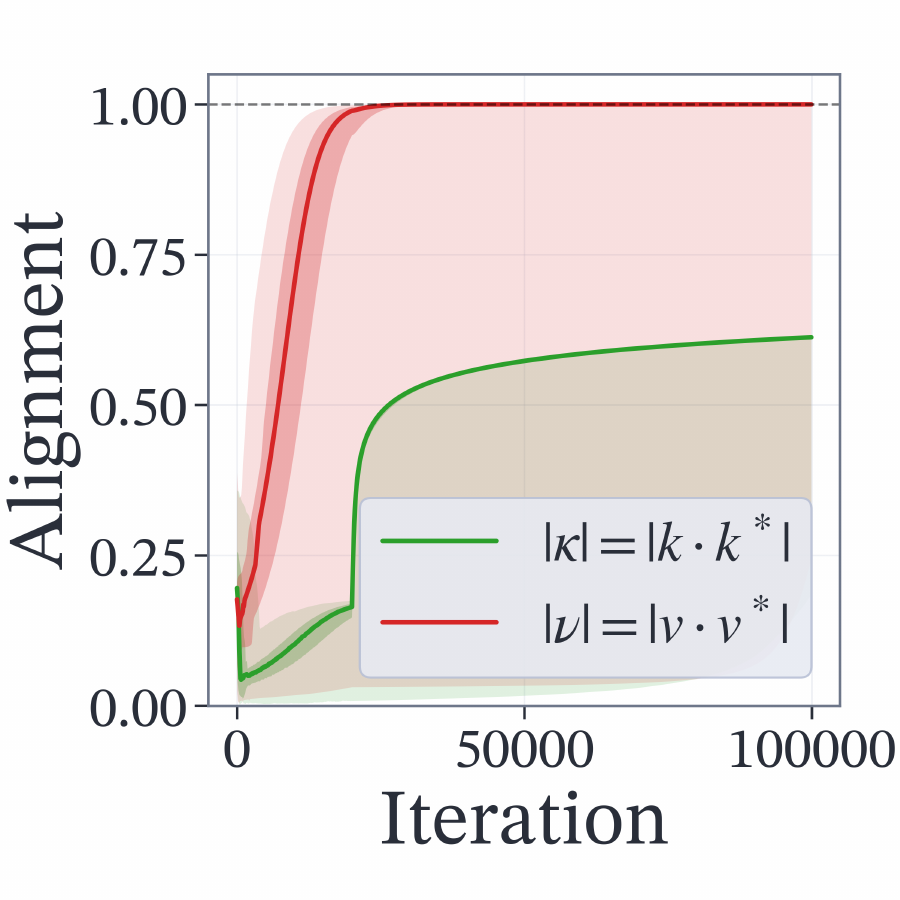}
        \label{fig:lambda_0.03_kappa}
    \end{subfigure}
    \hfill
    \begin{subfigure}[b]{0.24\textwidth}
        \centering
        \includegraphics[width=\textwidth]{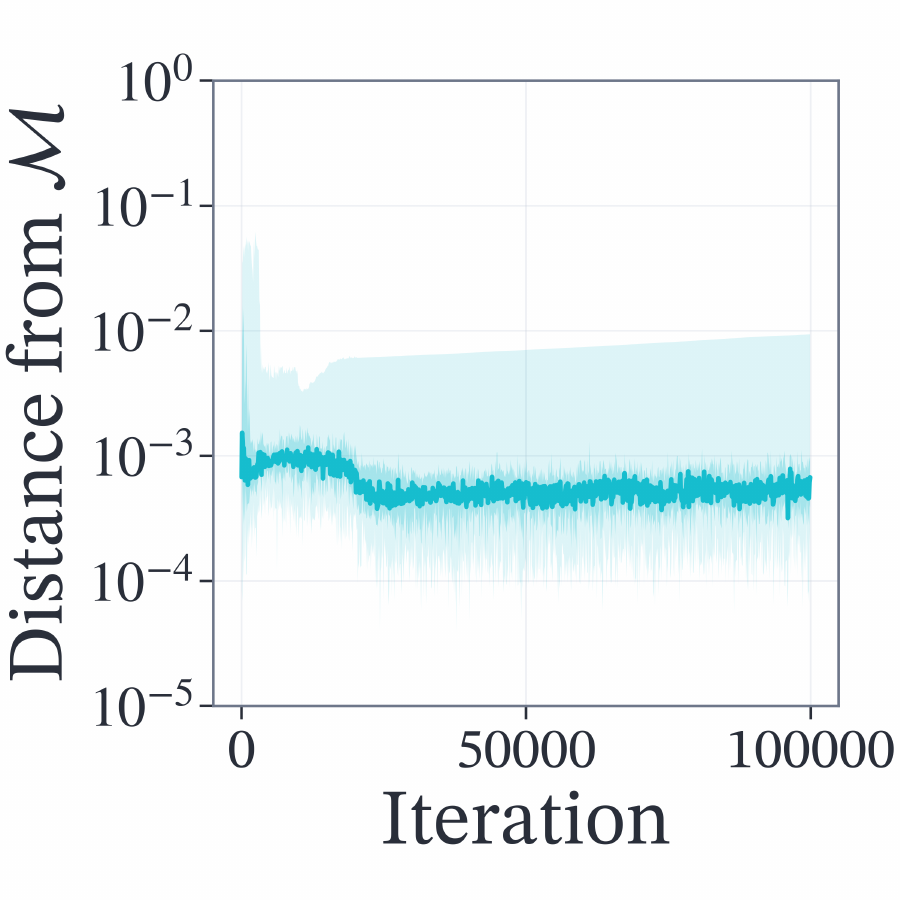}
        \label{fig:lambda_0.03_nu}
    \end{subfigure}
    \hfill
    \begin{subfigure}[b]{0.24\textwidth}
        \centering
        \includegraphics[width=\textwidth]{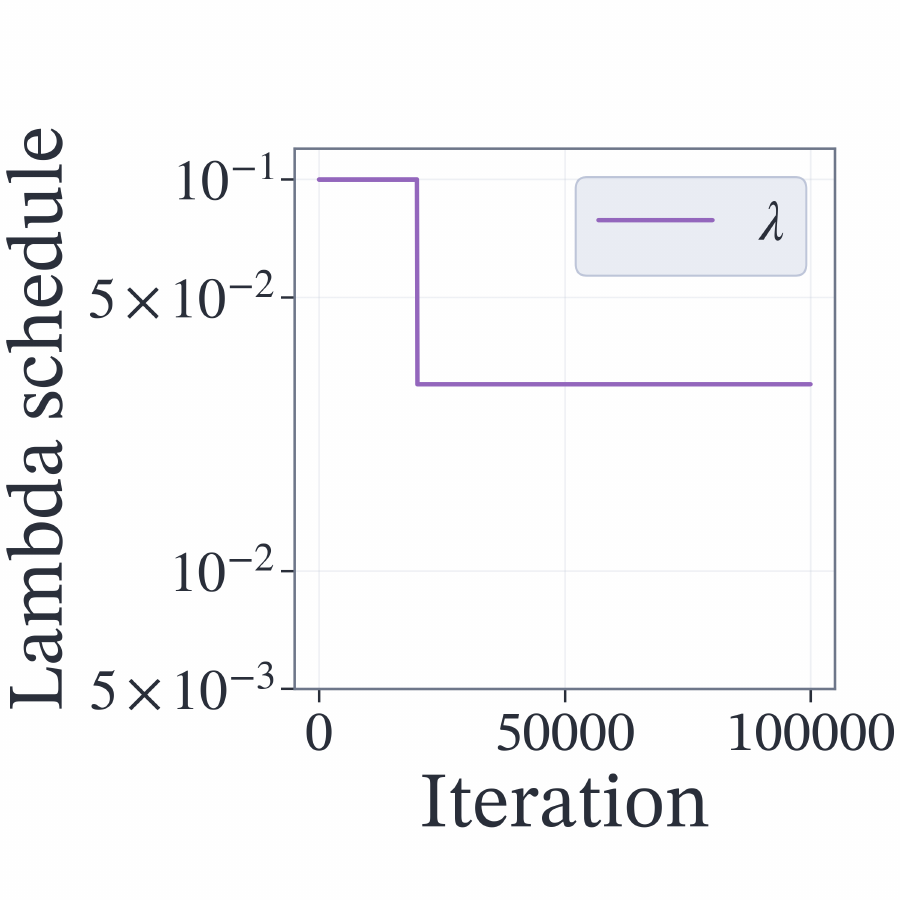}
        \label{fig:lambda_0.03_lambda}
    \end{subfigure}
    
    \caption{Training dynamics of the predictor with a 20k warm-up phase $\lambda_c = \lambda_r = \lambda = 0.1$ followed by constant inverse temperatures $\lambda_c = \lambda_r = 0.03$. The solid curve represents the empirical median across 20 runs and  the bands represent the 0.05-0.95 and, respectively, the 0.25-0.75 quantile intervals (although not always visible). The parameters used are $M = N = 5$, $d = 20$, $\gamma = 1$, $\ell = 100$.}
    \label{fig:lambda_0.03}
\end{figure}

\begin{figure}[htbp]
    \centering
    \begin{subfigure}[b]{0.24\textwidth}
        \centering
        \includegraphics[width=\textwidth]{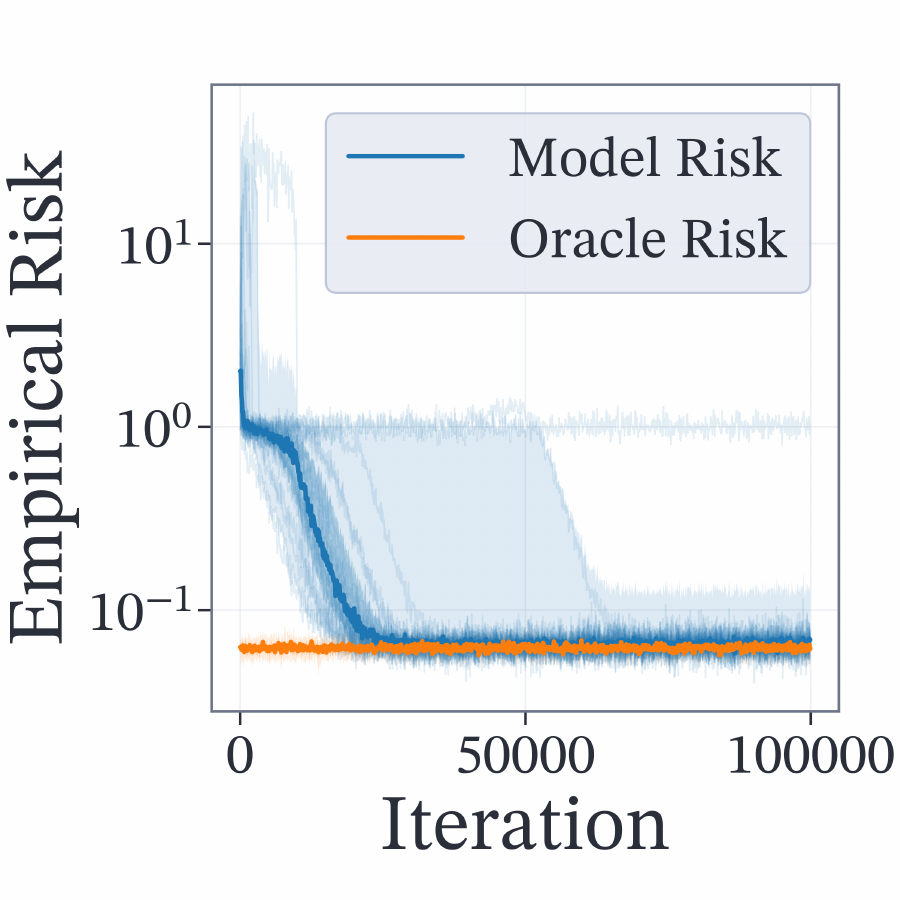}
        \label{fig:lambda_0.05_loss}
    \end{subfigure}
    \hfill
    \begin{subfigure}[b]{0.24\textwidth}
        \centering
        \includegraphics[width=\textwidth]{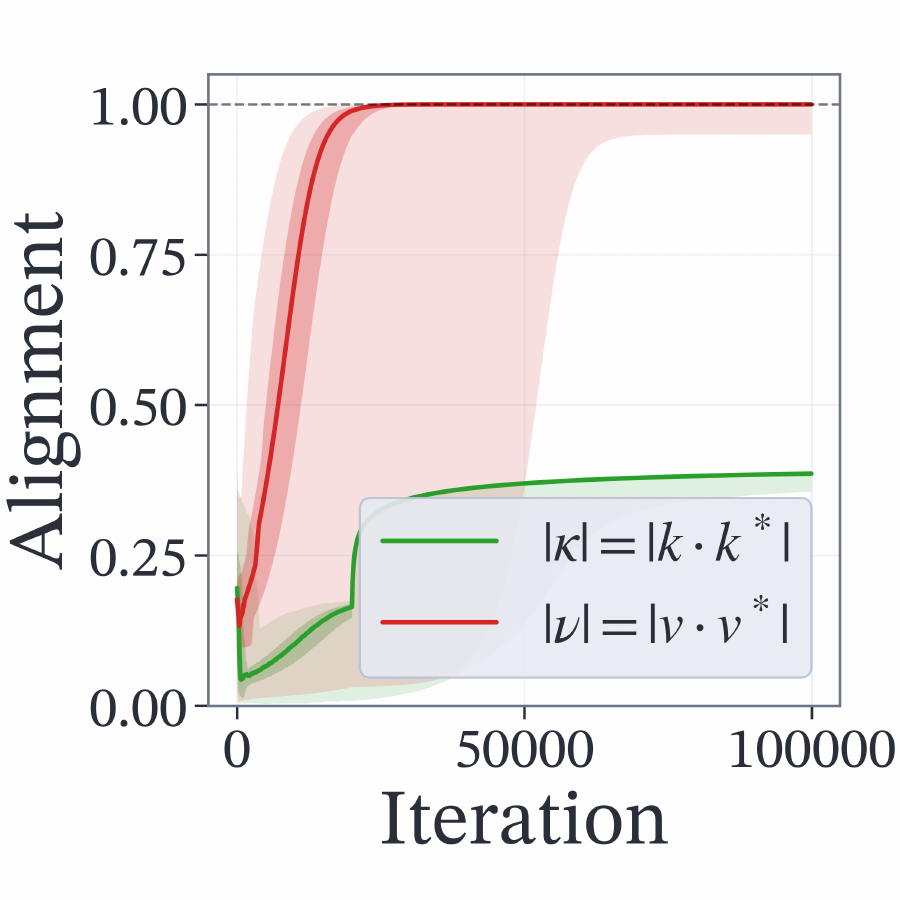}
        \label{fig:lambda_0.05_kappa}
    \end{subfigure}
    \hfill
    \begin{subfigure}[b]{0.24\textwidth}
        \centering
        \includegraphics[width=\textwidth]{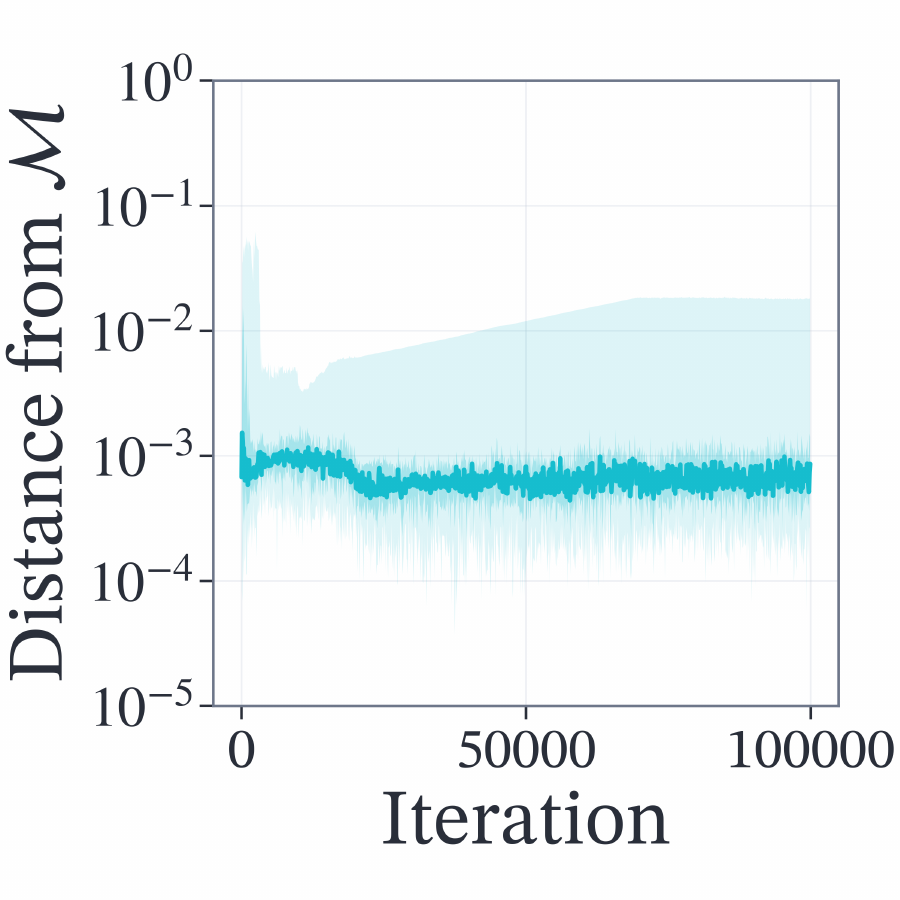}
        \label{fig:lambda_0.05_nu}
    \end{subfigure}
    \hfill
    \begin{subfigure}[b]{0.24\textwidth}
        \centering
        \includegraphics[width=\textwidth]{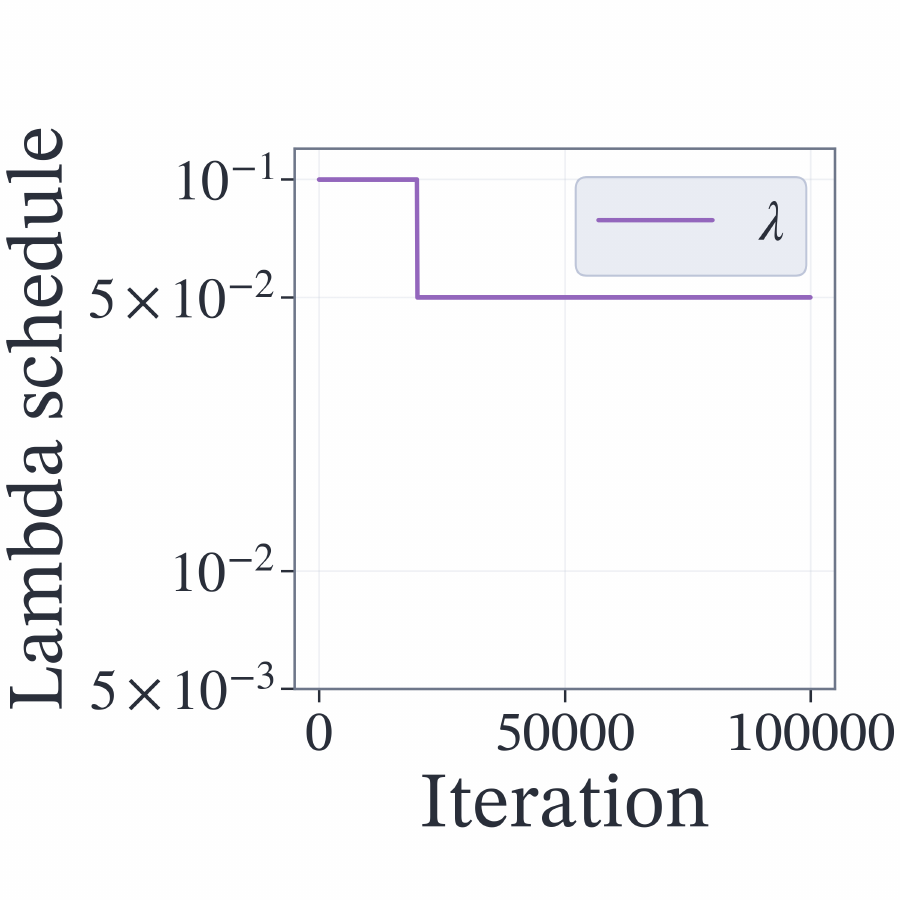}
        \label{fig:lambda_0.05_lambda}
    \end{subfigure}
    
    \caption{Training dynamics of the predictor with a 20k warm-up phase $\lambda_c = \lambda_r = \lambda = 0.1$ followed by constant inverse temperatures $\lambda_c = \lambda_r = 0.05$. The solid curve represents the empirical median across 20 runs and  the bands represent the 0.05-0.95 and, respectively, the 0.25-0.75 quantile intervals (although not always visible). The parameters used are $M = N = 5$, $d = 20$, $\gamma = 1$, $\ell = 100$.}
    \label{fig:lambda_0.05}
\end{figure}

\begin{figure}[htbp]
    \centering
    \begin{subfigure}[b]{0.24\textwidth}
        \centering
        \includegraphics[width=\textwidth]{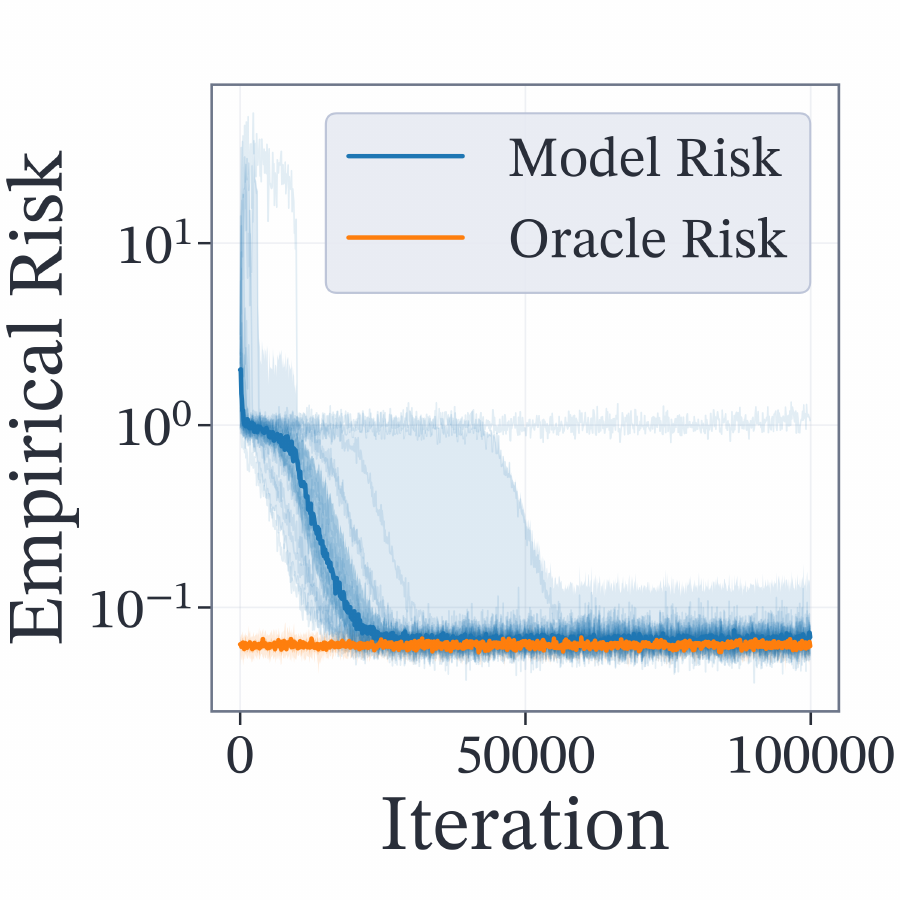}
        \label{fig:lambda_0.07_loss}
    \end{subfigure}
    \hfill
    \begin{subfigure}[b]{0.24\textwidth}
        \centering
        \includegraphics[width=\textwidth]{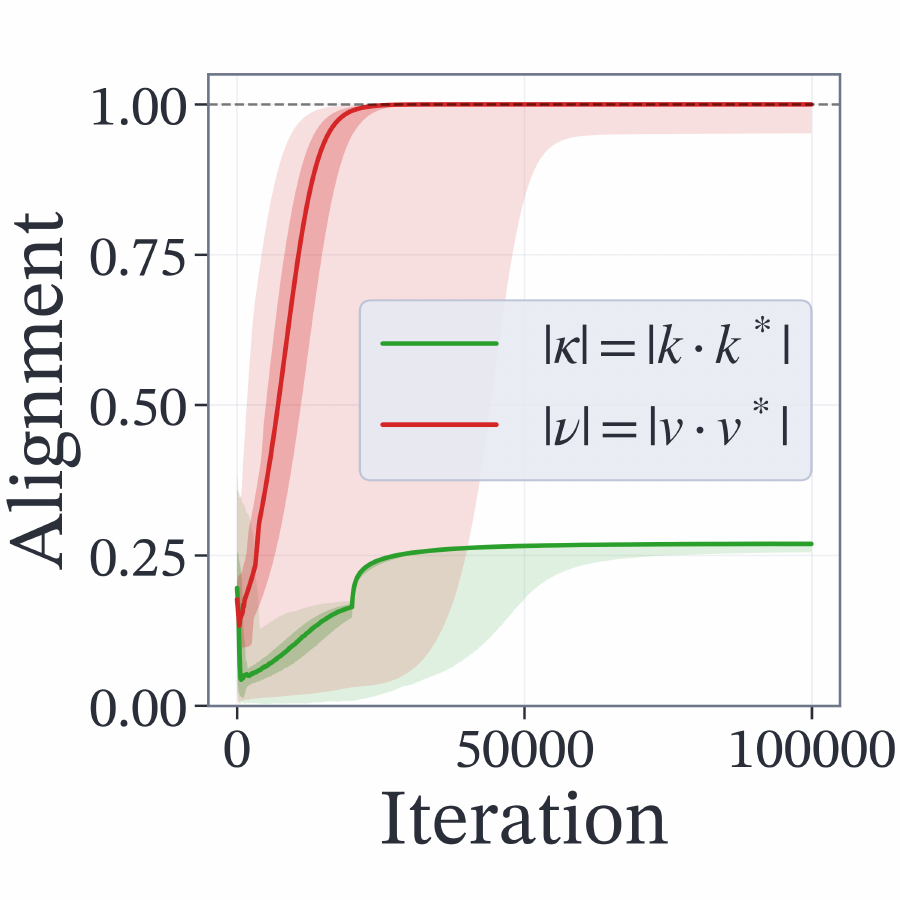}
        \label{fig:lambda_0.07_kappa}
    \end{subfigure}
    \hfill
    \begin{subfigure}[b]{0.24\textwidth}
        \centering
        \includegraphics[width=\textwidth]{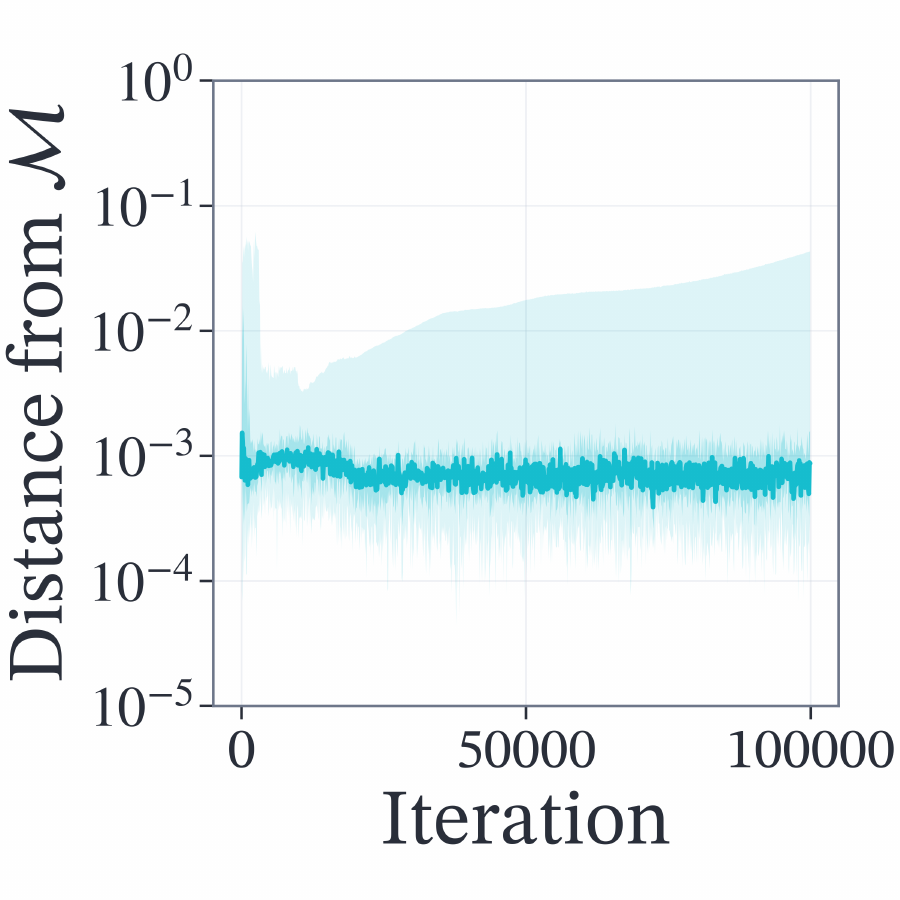}
        \label{fig:lambda_0.07_nu}
    \end{subfigure}
    \hfill
    \begin{subfigure}[b]{0.24\textwidth}
        \centering
        \includegraphics[width=\textwidth]{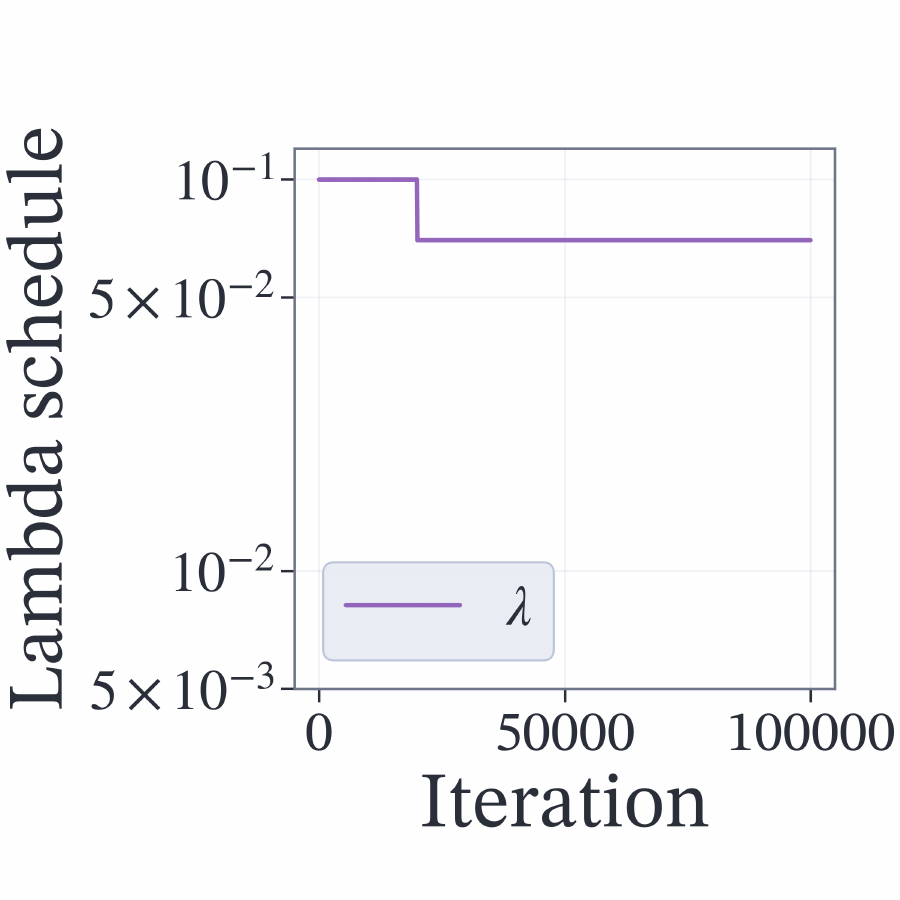}
        \label{fig:lambda_0.07_lambda}
    \end{subfigure}
    
    \caption{Training dynamics of the predictor with a 20k warm-up phase $\lambda_c = \lambda_r = \lambda = 0.1$ followed by constant inverse temperatures $\lambda_c = \lambda_r = 0.07$. The solid curve represents the empirical median across 20 runs and  the bands represent the 0.05-0.95 and, respectively, the 0.25-0.75 quantile intervals (although not always visible). The parameters used are $M = N = 5$, $d = 20$, $\gamma = 1$, $\ell = 100$.}
    \label{fig:lambda_0.07}
\end{figure}

\newpage

\subsection{Gradient analysis for decoupled inverse temperatures} \label{sec:app_decoupled_lambda}

This subsection isolates the effects of decoupling the two inverse temperatures. We visualize the gradient vector field of the reduced risk $\mathcal{R}^<_{\lambda_r, \lambda_c}(\kappa, \nu)$ evaluated on the invariant manifold $\mathcal{M}$ for multiple configurations of $(\lambda_r,  \lambda_c)$.

\begin{figure}[htbp]
    \centering
    \begin{subfigure}[b]{1\textwidth}
        \centering
        \includegraphics[width=\textwidth]{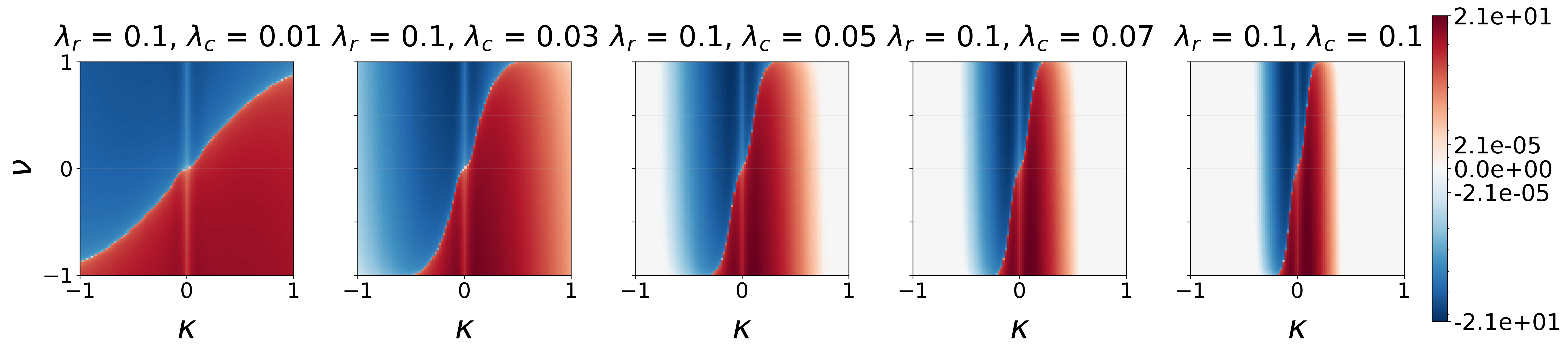}
        \caption{Partial derivative $\partial_\kappa \mathcal{R}^<_{\lambda_r, \lambda_c}$ for $\lambda_r = 0.1$ and $\lambda_c \in \{0.01, 0.03, 0.05, 0.07, 0.1\}.$}
        \label{fig:grad_lr_0p1_kappa}
    \end{subfigure}
    \begin{subfigure}[b]{1\textwidth}
        \centering
        \includegraphics[width=\textwidth]{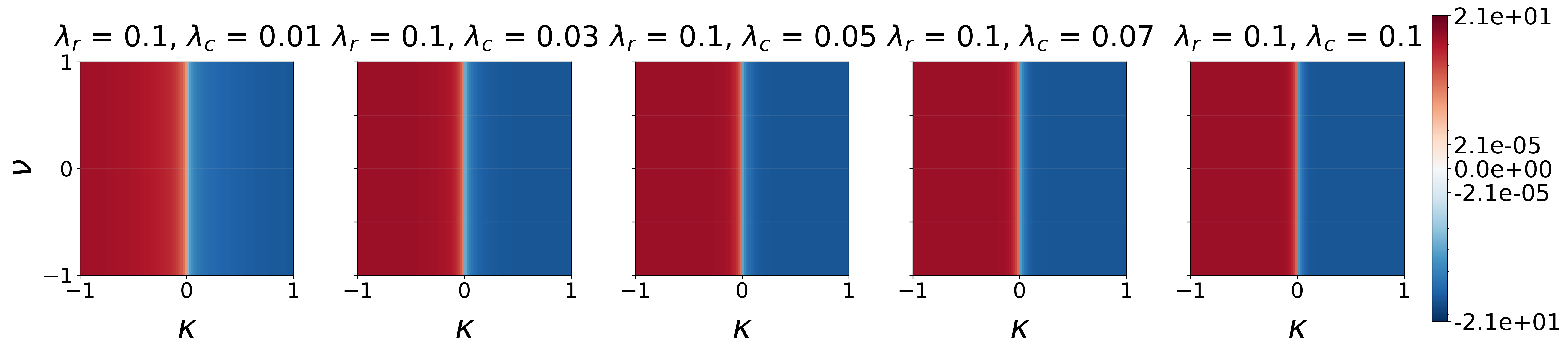}
        \caption{Partial derivative $\partial_\nu \mathcal{R}^<_{\lambda_r, \lambda_c}$ for $\lambda_r = 0.1$ and $\lambda_c \in \{0.01, 0.03, 0.05, 0.07, 0.1\}.$}
        \label{fig:grad_lr_0p1_nu}
    \end{subfigure}
    
    \vspace{1em}
    
    \begin{subfigure}[b]{1\textwidth}
        \centering
        \includegraphics[width=\textwidth]{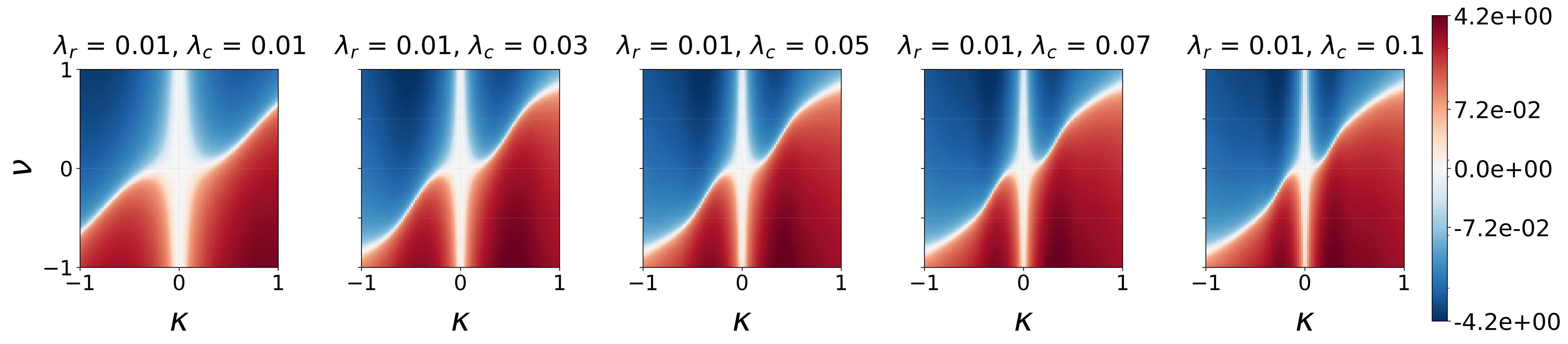}
        \caption{Partial derivative $\partial_\kappa \mathcal{R}^<_{\lambda_r, \lambda_c}$ for $\lambda_r = 0.01$ and $\lambda_c \in \{0.01, 0.03, 0.05, 0.07, 0.1\}.$}
        \label{fig:grad_lr_0p01_kappa}
    \end{subfigure}
    \begin{subfigure}[b]{1\textwidth}
        \centering
        \includegraphics[width=\textwidth]{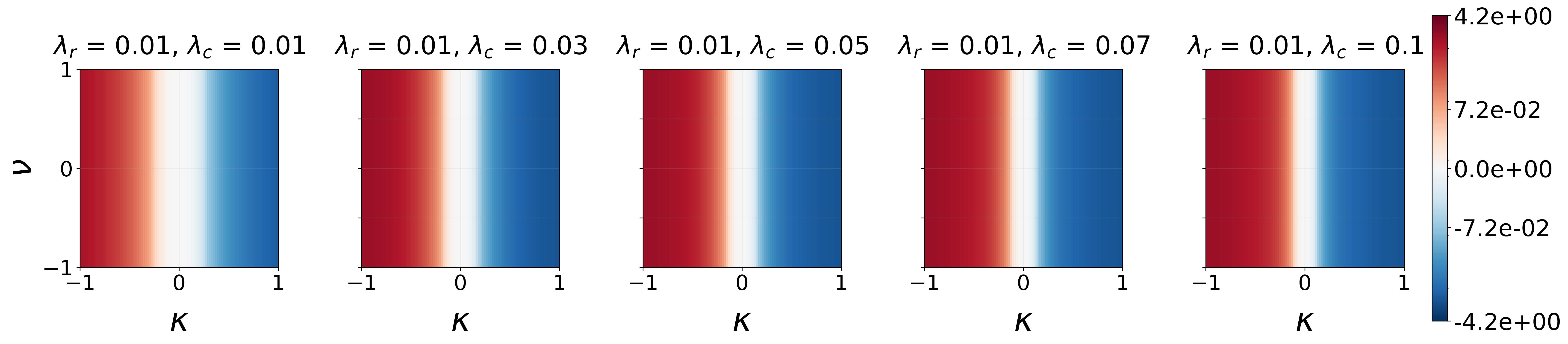}
        \caption{Partial derivative $\partial_\nu \mathcal{R}^<_{\lambda_r, \lambda_c}$ for $\lambda_r = 0.01$ and $\lambda_c \in \{0.01, 0.03, 0.05, 0.07, 0.1\}.$}
        \label{fig:grad_lr_0p01_nu}
    \end{subfigure}
    \caption{Numerically computed gradients of the theoretical risk $\mathcal{R}^<_{\lambda_r, \lambda_c}(\kappa, \nu)$ for $\lambda_r \in \{0.1, 0.01\}$ and varying values of $\lambda_c$. The parameters used are $M = N = 5$ and $\ell = 100$. The gradients, given by Equation~\eqref{eq:app_partials}, are approximated using the Monte Carlo method with a sample size of 10,000 on a $100 \times 100$ discretization grid.}
    \label{fig:app_gradients_lr}
\end{figure}

\begin{figure}[htbp]
    \begin{subfigure}[b]{1\textwidth}
        \centering
        \includegraphics[width=\textwidth]{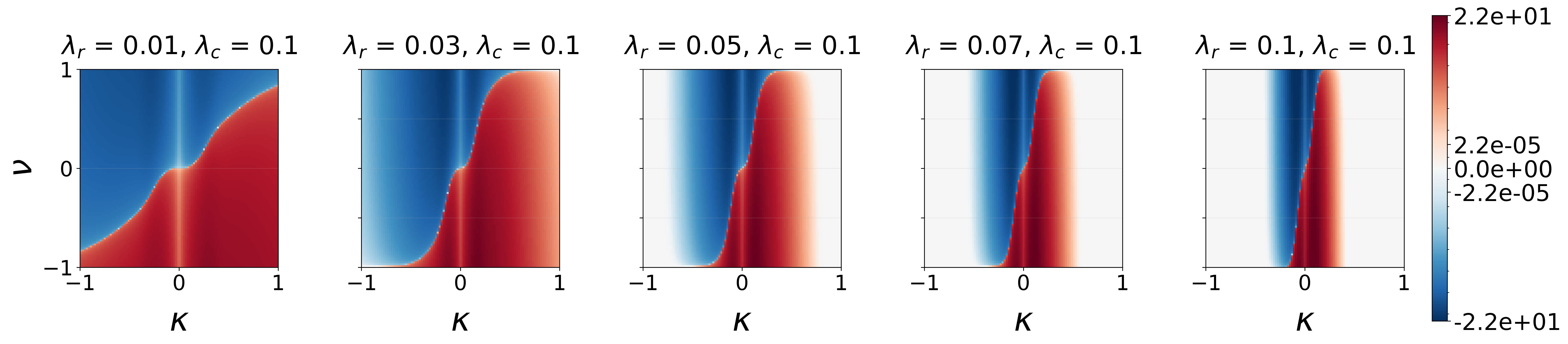}
        \caption{Partial derivative $\partial_\kappa \mathcal{R}^<_{\lambda_r, \lambda_c}$ for $\lambda_c = 0.1$ and $\lambda_r \in \{0.01, 0.03, 0.05, 0.07, 0.1\}.$}
        \label{fig:grad_lc_0p1_kappa}
    \end{subfigure}
    \begin{subfigure}[b]{1\textwidth}
        \centering
        \includegraphics[width=\textwidth]{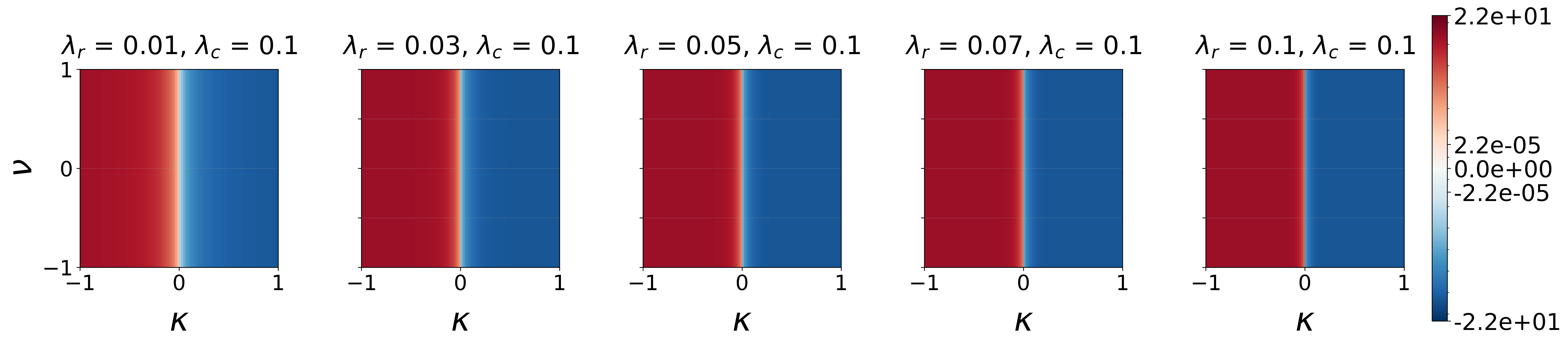}
        \caption{Partial derivative $\partial_\nu \mathcal{R}^<_{\lambda_r, \lambda_c}$ for $\lambda_c = 0.1$ and $\lambda_r \in \{0.01, 0.03, 0.05, 0.07, 0.1\}.$}
        \label{fig:grad_lc_0p1_nu}
    \end{subfigure}
    
    \vspace{1em}
    
    \begin{subfigure}[b]{1\textwidth}
        \centering
        \includegraphics[width=\textwidth]{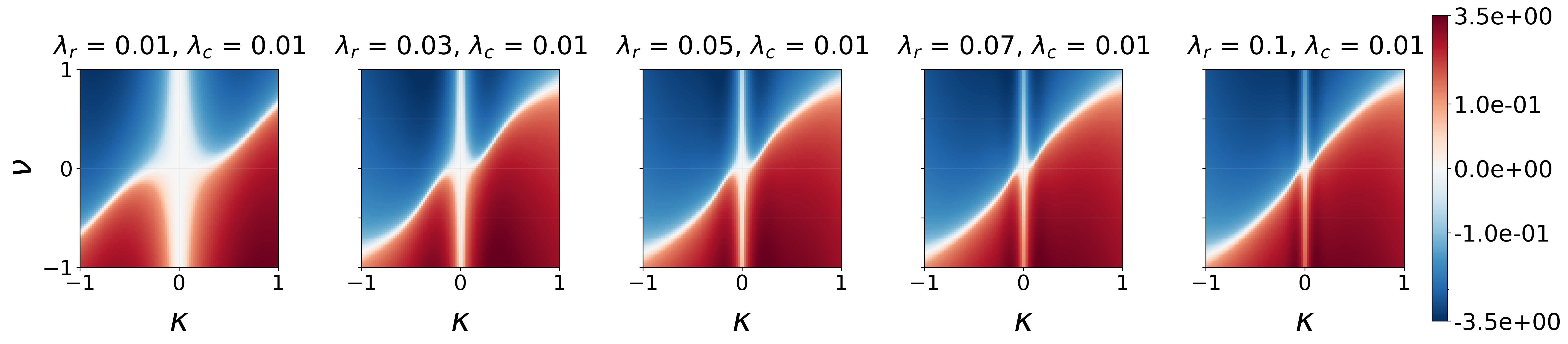}
        \caption{Partial derivative $\partial_\kappa \mathcal{R}^<_{\lambda_r, \lambda_c}$ for $\lambda_c = 0.01$ and $\lambda_r \in \{0.01, 0.03, 0.05, 0.07, 0.1\}.$}
        \label{fig:grad_lc_0p01_kappa}
    \end{subfigure}
    \begin{subfigure}[b]{1\textwidth}
        \centering
        \includegraphics[width=\textwidth]{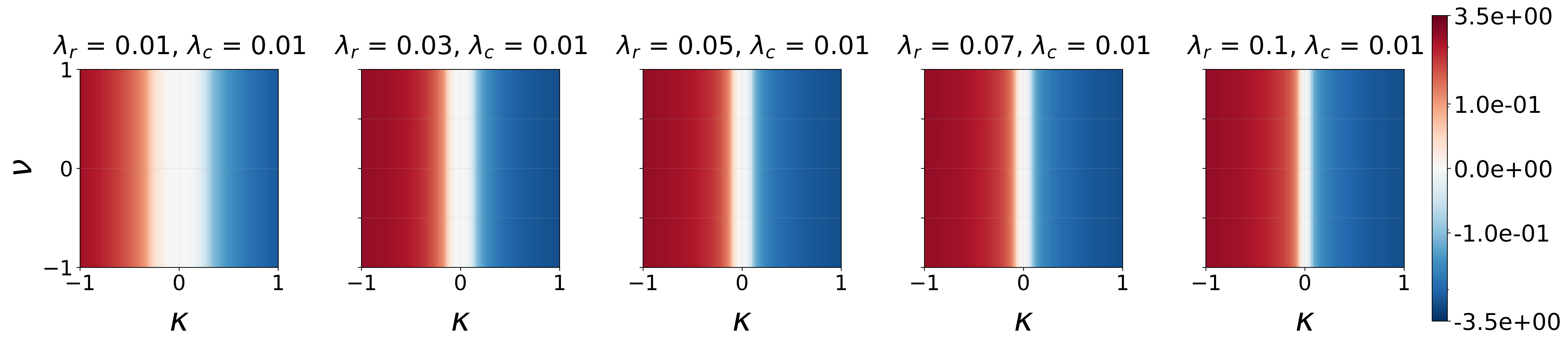}
        \caption{Partial derivative $\partial_\nu \mathcal{R}^<_{\lambda_r, \lambda_c}$ for $\lambda_c = 0.01$ and $\lambda_r \in \{0.01, 0.03, 0.05, 0.07, 0.1\}.$}
        \label{fig:grad_lc_0p01_nu}
    \end{subfigure}

    \caption{Numerically computed gradients of the theoretical risk $\mathcal{R}^<_{\lambda_r, \lambda_c}(\kappa, \nu)$ for $\lambda_c \in \{0.1, 0.01\}$ and varying values of $\lambda_r$. The parameters used are $M = N = 5$ and $\ell = 100$. The gradients, given by Equation~\eqref{eq:app_partials}, are approximated using the Monte Carlo method with a sample size of 10,000 on a $100 \times 100$ discretization grid.}
    \label{fig:app_gradients_lc}
\end{figure}

Figures~\ref{fig:app_gradients_lr}, \ref{fig:app_gradients_lc} present the numerically computed gradients for various mixed combinations of $\lambda_c$ and $\lambda_r$. By breaking the symmetry, we observe how the magnitude and direction of the gradients shift across the landscape. Specifically, two interesting phenomena appear:
\begin{itemize}
    \item $\partial_\nu \mathcal{R}^<_{\lambda_r, \lambda_c}$ behaves symmetrically with respect to both $\lambda_r$ and $\lambda_c$;
    \item The same is not true for $\partial_\kappa \mathcal{R}^<_{\lambda_r, \lambda_c}$. The shape of the gradient and most importantly of the 0 level curve of $\partial_\kappa \mathcal{R}^<$ varies significantly with the pairs $(\lambda_r, \lambda_c)$.
\end{itemize}
This analysis leads to another method of controlling the training dynamics, namely through the ratio $\tau = \lambda_c/\lambda_r$. Indeed, Figure~\ref{fig:on_manifold_lc_0p01_lr0p1} presents the optimization dynamics obtained for constant $\lambda_c = 0.01$ and $\lambda_r = 0.1$. While 100,000 steps are not enough for convergence at the oracle parameters, the runs consistently move towards convergence. Moreover, the training dynamics are different from the ones studied in the main paper, as can be seen from the much faster increase of $|\kappa|$, closely mirroring the behavior of $|\nu|$.

\begin{figure}[htbp]
    \centering
    \begin{subfigure}[b]{0.32\textwidth}
        \centering
        \includegraphics[width=\textwidth]{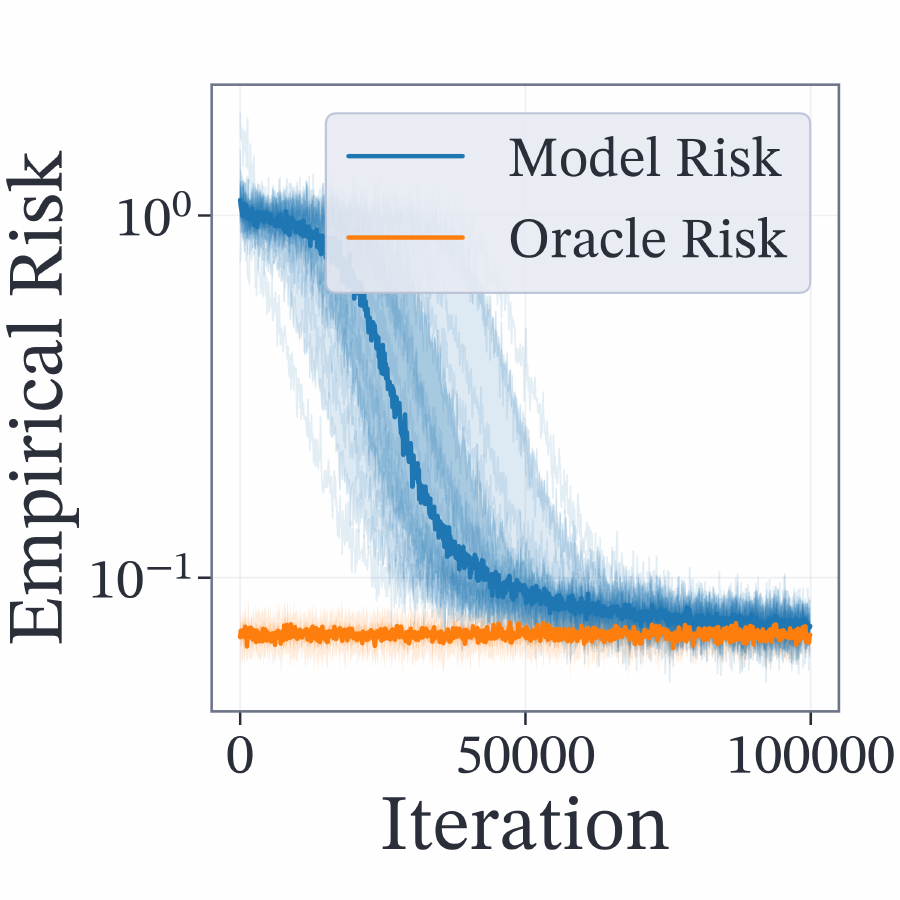}
        \label{fig:on_manifold_lc_0p01_lr0p1_loss}
    \end{subfigure}
    \hfill
    \begin{subfigure}[b]{0.32\textwidth}
        \centering
        \includegraphics[width=\textwidth]{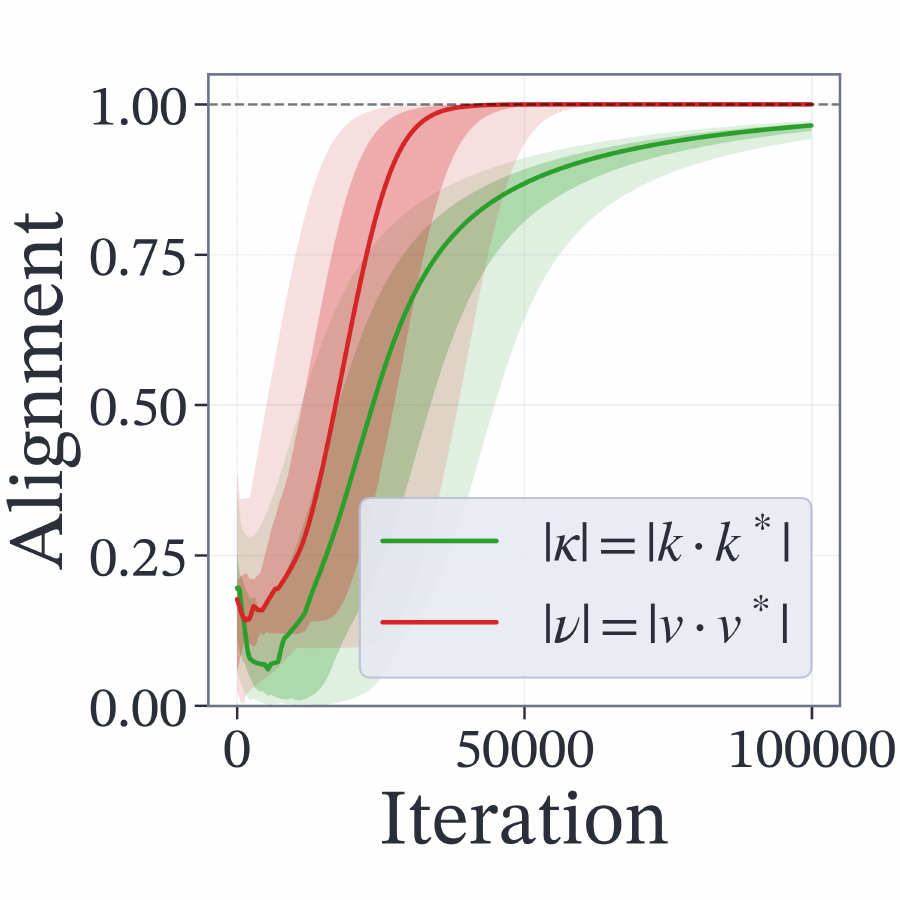}
        \label{fig:on_manifold_lc_0p01_lr0p1_alignment}
    \end{subfigure}
    \hfill
    \begin{subfigure}[b]{0.32\textwidth}
        \centering
        \includegraphics[width=\textwidth]{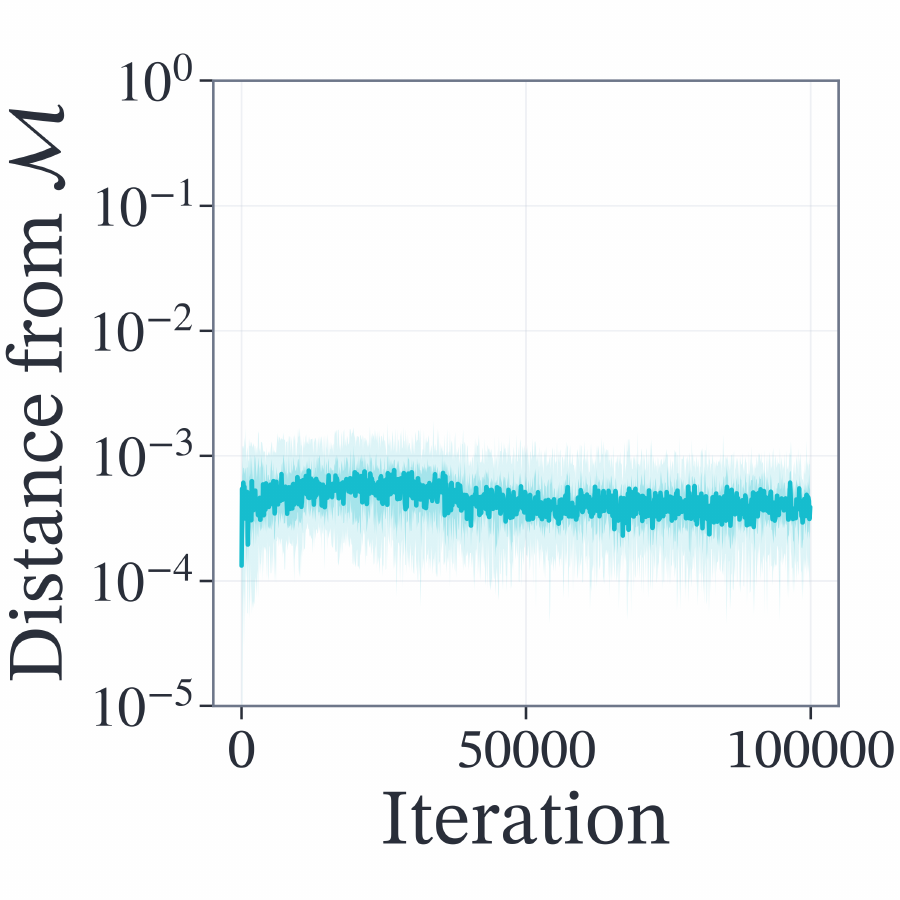}
        \label{fig:on_manifold_lc_0p01_lr0p1_distance}
    \end{subfigure}
    
    \caption{Training dynamics of the predictor with constant $\lambda_c = 0.01$ and $\lambda_r= 0.1$. The solid curve represents the empirical median across the runs and the bands represent the 0.05-0.95 and, respectively, the 0.25-0.75 quantile intervals (although not always visible). The parameters used are $M = N = 5$, $d = 20$, $\gamma = 1$, $\ell = 100$.}
    \label{fig:on_manifold_lc_0p01_lr0p1}
\end{figure}

\newpage

\subsection{Gradient analysis for other hyperparameters $N$ and $\ell$} \label{sec:app_grad_N_ell}
In this section, we investigate the behavior of the partial derivatives of $\mathcal{R}^<$ when varying $N$ and $\ell$, for fixed $\lambda_r = \lambda_c = 0.1$. 

\begin{figure}[htbp]
    \centering
    \begin{subfigure}[b]{1\textwidth}
        \centering
        \includegraphics[width=\textwidth]{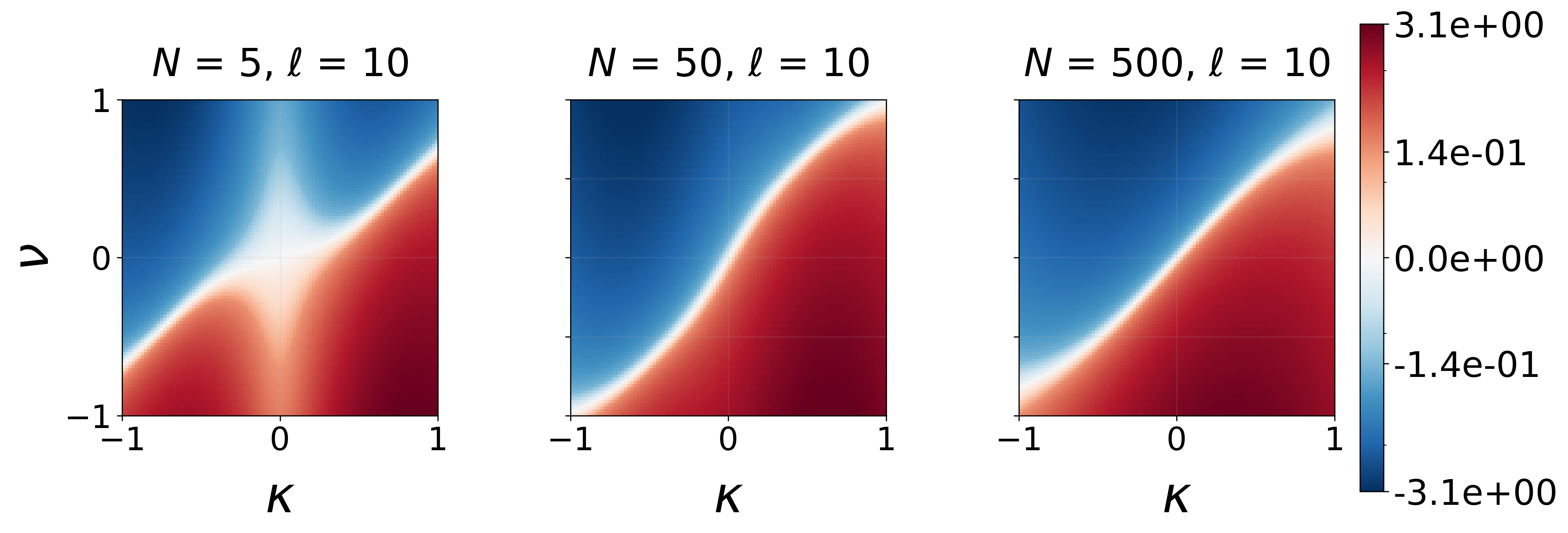}
        \caption{Partial derivative $\partial_\kappa \mathcal{R}^<_{\lambda_r, \lambda_c}$ for $\lambda_r = \lambda_c = 0.1$, $\ell = 10$ and $N \in \{5, 50, 500\}$.}
        \label{fig:grad_ell_10_kappa}
    \end{subfigure}
    \begin{subfigure}[b]{1\textwidth}
        \centering
        \includegraphics[width=\textwidth]{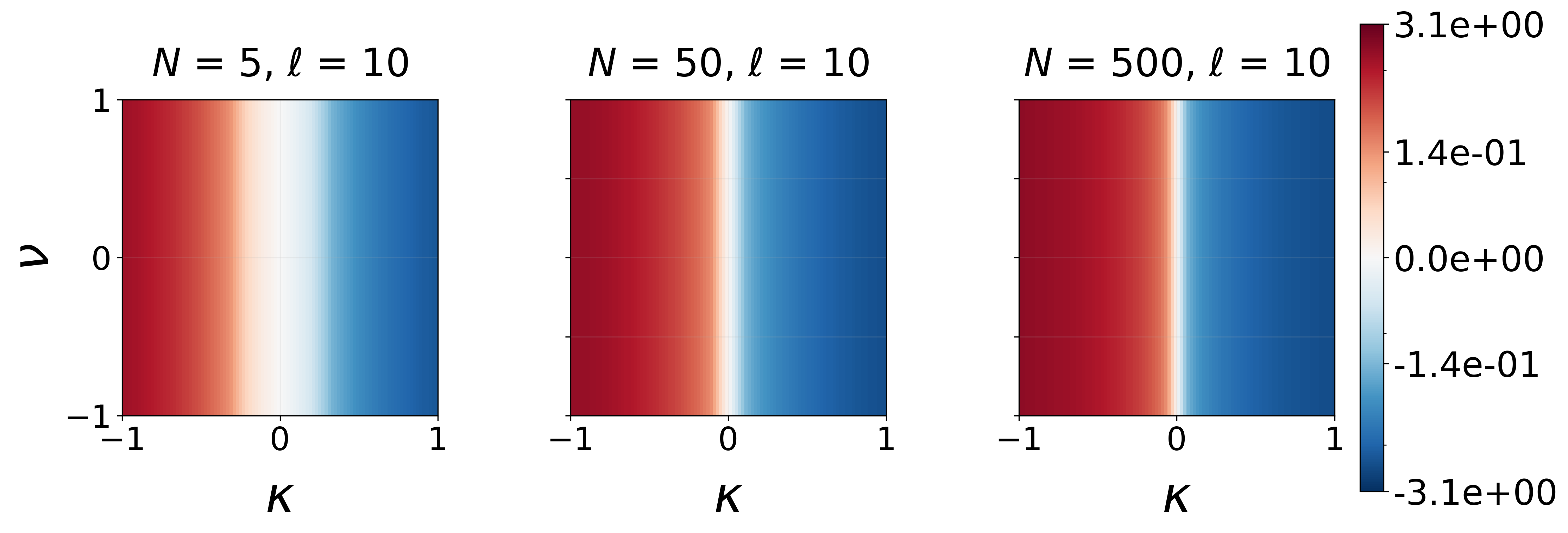}
        \caption{Partial derivative $\partial_\nu \mathcal{R}^<_{\lambda_r, \lambda_c}$ for $\lambda_r = \lambda_c = 0.1$, $\ell = 10$ and $N \in \{5, 50, 500\}$.}
        \label{fig:grad_ell_10_nu}
    \end{subfigure}
    \caption{Numerically computed gradients of the theoretical risk $\mathcal{R}^<_{\lambda_r, \lambda_c}(\kappa, \nu)$ for $\lambda_r = \lambda_c = 0.1$, $\ell = 10$ and varying values of $N$. The gradients, given by Equation~\eqref{eq:app_partials}, are approximated using the Monte Carlo method with a sample size of 10,000 on a $100 \times 100$ discretization grid.}
    \label{fig:grad_ell_10}
\end{figure}

The values considered are not following the asymptotics required for Bayes optimality of the predictor, but better reflect real-world scenarios in which TSFMs are trained to take tens of channels and hundreds of time patches as input, while an individual token's shift in the mean cannot be too large due to normalization layers. We restrict our analysis to exploring the gradient across multiple pairs of values $(N, \ell)$, as conducting full trainings for each pair would be computationally expensive.

\begin{figure}[htbp]
    \centering
    \begin{subfigure}[b]{1\textwidth}
        \centering
        \includegraphics[width=\textwidth]{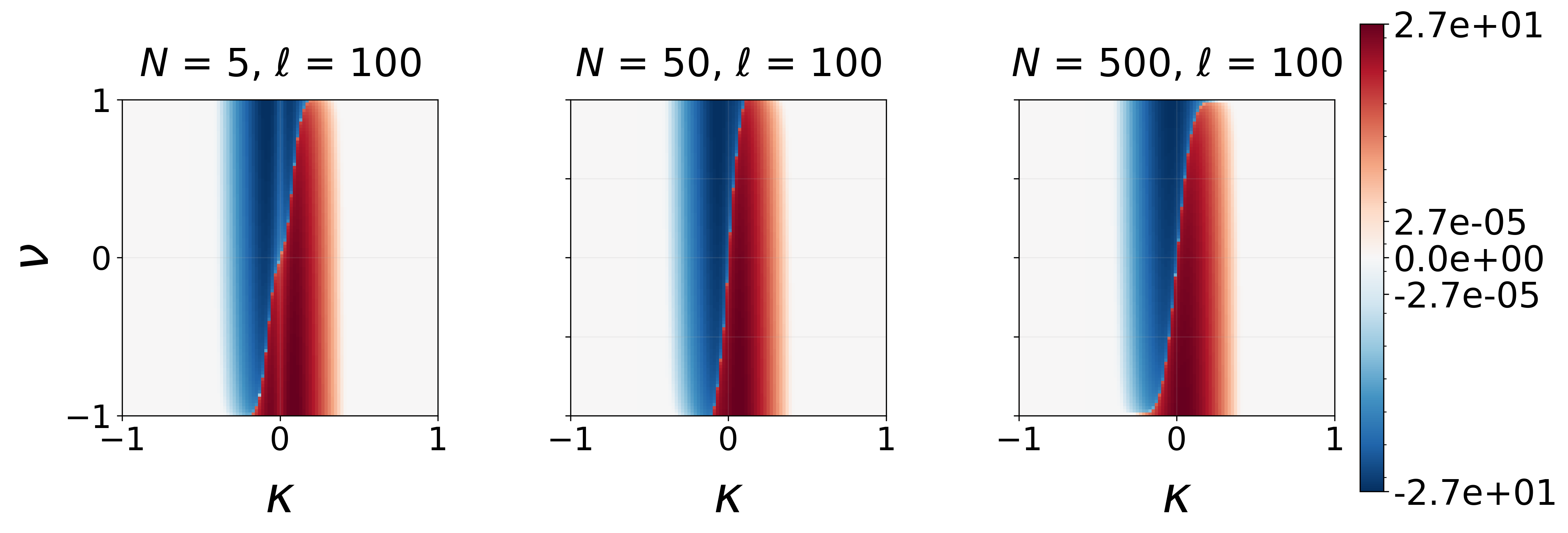}
        \caption{Partial derivative $\partial_\kappa \mathcal{R}^<_{\lambda_r, \lambda_c}$ for $\lambda_r = \lambda_c = 0.1$, $\ell = 100$ and $N \in \{5, 50, 500\}$.}
        \label{fig:grad_ell_100_kappa}
    \end{subfigure}
    \begin{subfigure}[b]{1\textwidth}
        \centering
        \includegraphics[width=\textwidth]{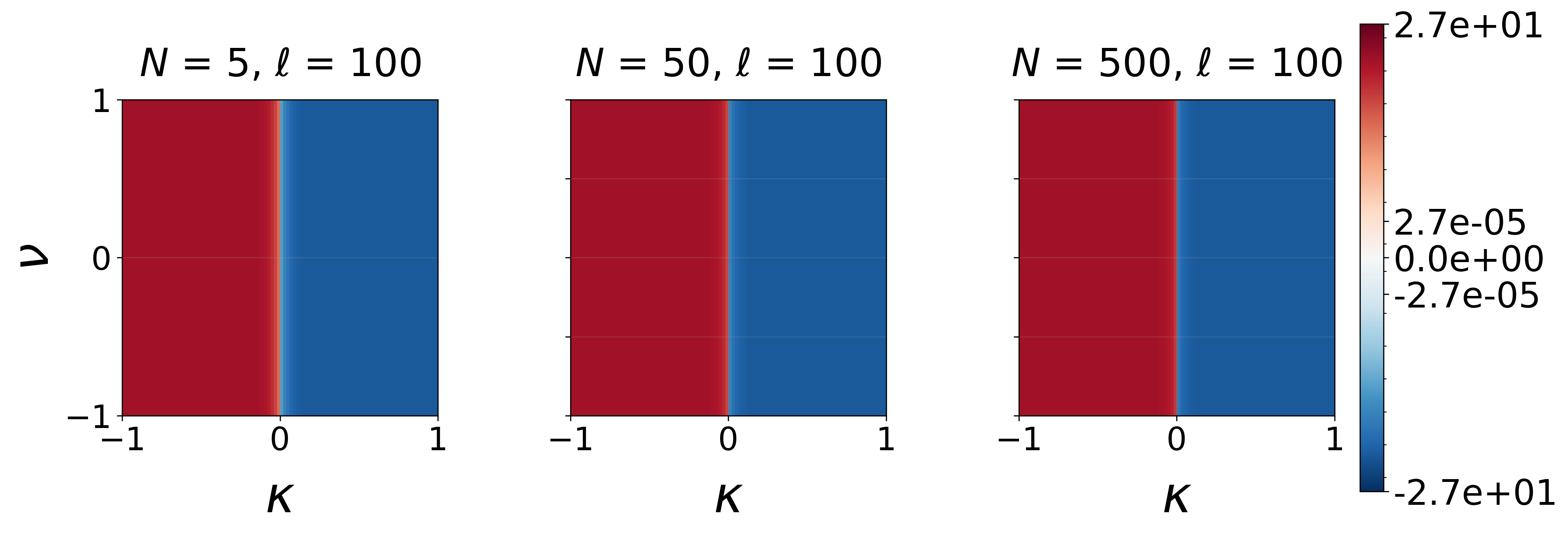}
        \caption{Partial derivative $\partial_\nu \mathcal{R}^<_{\lambda_r, \lambda_c}$ for $\lambda_r = \lambda_c = 0.1$, $\ell = 100$ and $N \in \{5, 50, 500\}$.}
        \label{fig:grad_ell_100_nu}
    \end{subfigure}
    \caption{Numerically computed gradients of the theoretical risk $\mathcal{R}^<_{\lambda_r, \lambda_c}(\kappa, \nu)$ for $\lambda_r = \lambda_c = 0.1$, $\ell = 100$ and varying values of $N$. The gradients, given by Equation~\eqref{eq:app_partials}, are approximated using the Monte Carlo method with a sample size of 10,000 on a $100 \times 100$ discretization grid.}
    \label{fig:grad_ell_100}
\end{figure}

\begin{figure}[htbp]
    \centering
    \begin{subfigure}[b]{1\textwidth}
        \centering
        \includegraphics[width=\textwidth]{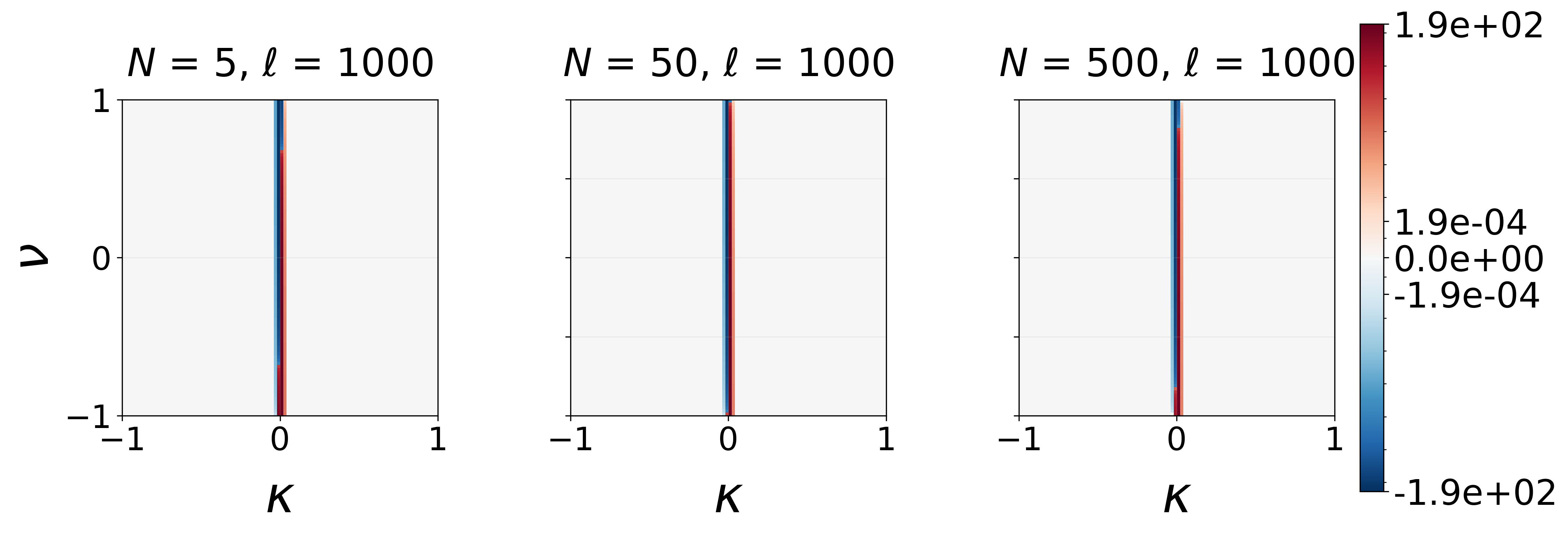}
        \caption{Partial derivative $\partial_\kappa \mathcal{R}^<_{\lambda_r, \lambda_c}$ for $\lambda_r = \lambda_c = 0.1$, $\ell = 1000$ and $N \in \{5, 50, 500\}$.}
        \label{fig:grad_ell_1000_kappa}
    \end{subfigure}
    \begin{subfigure}[b]{1\textwidth}
        \centering
        \includegraphics[width=\textwidth]{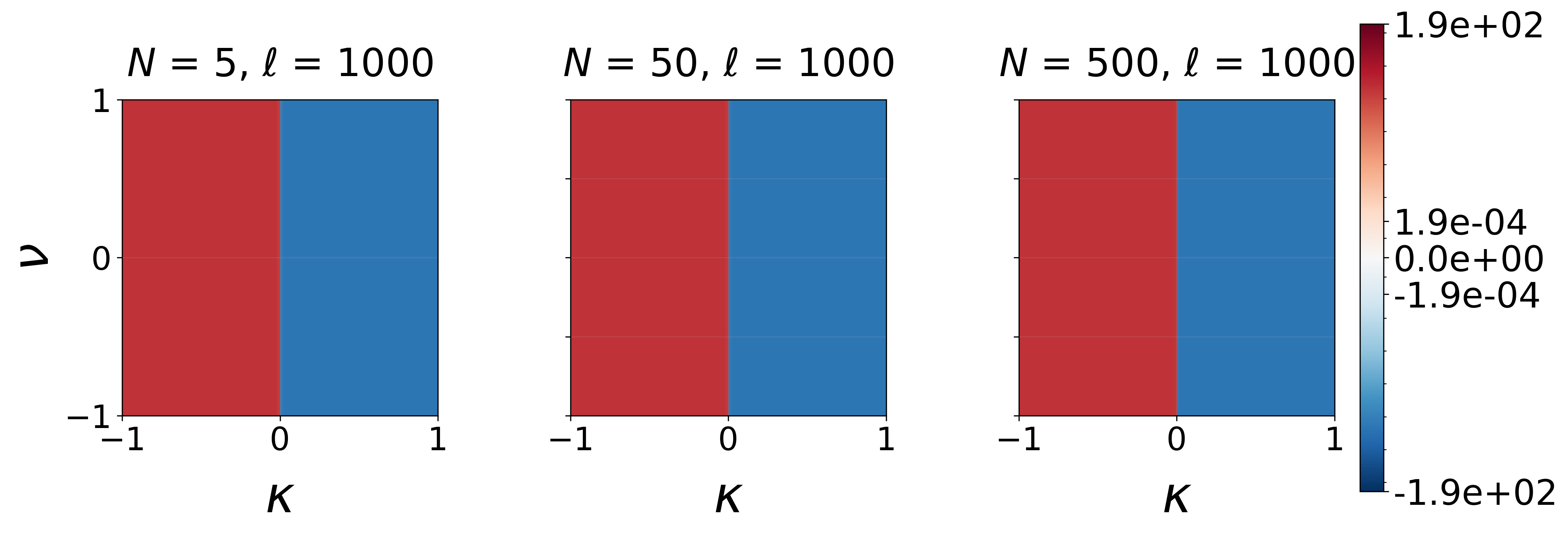}
        \caption{Partial derivative $\partial_\nu \mathcal{R}^<_{\lambda_r, \lambda_c}$ for $\lambda_r = \lambda_c = 0.1$, $\ell = 1000$ and $N \in \{5, 50, 500\}$.}
        \label{fig:grad_ell_1000_nu}
    \end{subfigure}
    \caption{Numerically computed gradients of the theoretical risk $\mathcal{R}^<_{\lambda_r, \lambda_c}(\kappa, \nu)$ for $\lambda_r = \lambda_c = 0.1$, $\ell = 1000$ and varying values of $N$. The gradients, given by Equation~\eqref{eq:app_partials}, are approximated using the Monte Carlo method with a sample size of 10,000 on a $100 \times 100$ discretization grid.}
    \label{fig:grad_ell_1000}
\end{figure}

As can be seen in Figures~\ref{fig:grad_ell_10}, \ref{fig:grad_ell_100}, \ref{fig:grad_ell_1000}, the effect of $\ell$ overshadows that of $N$. Although for small $\ell$, a higher $N$ appears to eliminate the problem of escaping the origin for $\kappa$, this disappears for larger $\ell$ values.

Furthermore, the higher $\ell$ is, the higher the values of $\partial_\nu\mathcal{R}^<$ and the steeper its sign change around $\kappa = 0$.

\newpage


\newpage
\section*{NeurIPS Paper Checklist}

\begin{enumerate}

\item {\bf Claims}
    \item[] Question: Do the main claims made in the abstract and introduction accurately reflect the paper's contributions and scope?
    \item[] Answer: \answerYes{} 
    \item[] Justification: We provide proofs for all the theoretical claims made and extensive numerical experiments both in the main text and in the supplemental material for all the empirically observed claims.

\item {\bf Limitations}
    \item[] Question: Does the paper discuss the limitations of the work performed by the authors?
    \item[] Answer: \answerYes{} 
    \item[] Justification: While not having a separate "Limitations" section in the paper, we discuss limitations of our work in the "Conclusions and limitations" section.

\item {\bf Theory assumptions and proofs}
    \item[] Question: For each theoretical result, does the paper provide the full set of assumptions and a complete (and correct) proof?
    \item[] Answer: \answerYes{} 
    \item[] Justification: All the theoretical results and claims are accompanied by full proofs in the supplemental material.

    \item {\bf Experimental result reproducibility}
    \item[] Question: Does the paper fully disclose all the information needed to reproduce the main experimental results of the paper to the extent that it affects the main claims and/or conclusions of the paper (regardless of whether the code and data are provided or not)?
    \item[] Answer: \answerYes{} 
    \item[] Justification: We provide the full code used for the numerical experiments in the paper. Furthermore, we provide the complete set of parameters for each experiment in the paper.

\item {\bf Open access to data and code}
    \item[] Question: Does the paper provide open access to the data and code, with sufficient instructions to faithfully reproduce the main experimental results, as described in supplemental material?
    \item[] Answer: \answerYes{} 
    \item[] Justification: The code is given as a resource in the supplemental materials.

\item {\bf Experimental setting/details}
    \item[] Question: Does the paper specify all the training and test details (e.g., data splits, hyperparameters, how they were chosen, type of optimizer) necessary to understand the results?
    \item[] Answer: \answerYes{} 
    \item[] Justification: The hyperparameters used in all the experiments are listed in the supplemental material.

\item {\bf Experiment statistical significance}
    \item[] Question: Does the paper report error bars suitably and correctly defined or other appropriate information about the statistical significance of the experiments?
    \item[] Answer: \answerYes{} 
    \item[] Justification: Empirical distributions are always plotted alongside the main curves.

\item {\bf Experiments compute resources}
    \item[] Question: For each experiment, does the paper provide sufficient information on the computer resources (type of compute workers, memory, time of execution) needed to reproduce the experiments?
    \item[] Answer: \answerYes{} 
    \item[] Justification: Details regarding the compute resource are given in Appendix~\ref{sec:app_numerical}.
    
\item {\bf Code of ethics}
    \item[] Question: Does the research conducted in the paper conform, in every respect, with the NeurIPS Code of Ethics \url{https://neurips.cc/public/EthicsGuidelines}?
    \item[] Answer: \answerYes{} 
    \item[] Justification: All the experiments are conducted using synthetically generated data. The authors don't see any possible harm arising from the present research.

\item {\bf Broader impacts}
    \item[] Question: Does the paper discuss both potential positive societal impacts and negative societal impacts of the work performed?
    \item[] Answer: \answerNA{} 
    \item[] Justification: The work in this paper is purely theory-oriented, aiming to better understand existing architectures. The authors see no societal impact related to the present paper.
    
\item {\bf Safeguards}
    \item[] Question: Does the paper describe safeguards that have been put in place for responsible release of data or models that have a high risk for misuse (e.g., pre-trained language models, image generators, or scraped datasets)?
    \item[] Answer: \answerNA{} 
    \item[] Justification: The authors don't see any risk in the models used for this research.

\item {\bf Licenses for existing assets}
    \item[] Question: Are the creators or original owners of assets (e.g., code, data, models), used in the paper, properly credited and are the license and terms of use explicitly mentioned and properly respected?
    \item[] Answer: \answerNA{} 
    \item[] Justification: No existing asset is used.

\item {\bf New assets}
    \item[] Question: Are new assets introduced in the paper well documented and is the documentation provided alongside the assets?
    \item[] Answer: \answerYes{} 
    \item[] Justification: The only asset released is code for the experiments in the main paper. The code is well documented in the opinion of the authors.

\item {\bf Crowdsourcing and research with human subjects}
    \item[] Question: For crowdsourcing experiments and research with human subjects, does the paper include the full text of instructions given to participants and screenshots, if applicable, as well as details about compensation (if any)? 
    \item[] Answer: \answerNA{} 
    \item[] Justification: The paper does not involve crowdsourcing nor research with human subjects.

\item {\bf Institutional review board (IRB) approvals or equivalent for research with human subjects}
    \item[] Question: Does the paper describe potential risks incurred by study participants, whether such risks were disclosed to the subjects, and whether Institutional Review Board (IRB) approvals (or an equivalent approval/review based on the requirements of your country or institution) were obtained?
    \item[] Answer: \answerNA{} 
    \item[] Justification: The paper does not involve crowdsourcing nor research with human subjects.

\item {\bf Declaration of LLM usage}
    \item[] Question: Does the paper describe the usage of LLMs if it is an important, original, or non-standard component of the core methods in this research? Note that if the LLM is used only for writing, editing, or formatting purposes and does \emph{not} impact the core methodology, scientific rigor, or originality of the research, declaration is not required.
    \item[] Answer: \answerNA{} 
    \item[] Justification: LLMs were used only for writing, editing, or formatting purposes.

\end{enumerate}

\end{document}